\documentclass[12pt,journal,draftcls,a4paper,twoside,onecolumn]{IEEEtran} 
\usepackage{setspace}
\usepackage{setspace}
\usepackage{amsmath}                        
\usepackage{graphicx}
\usepackage{amssymb}
\usepackage{cite}
\usepackage{algorithmic}
\usepackage[linesnumbered,lined,ruled]{algorithm2e}
\usepackage{array}
\usepackage{stfloats}
\usepackage{subfigure}
\usepackage{multicol}
\usepackage[all,poly]{xy}
\usepackage{multirow}
\usepackage{color}
\usepackage{graphicx,psfrag,color}
\usepackage{amsthm}
\usepackage{url}

\def\bfd{{\mathbf{d}}}

\def\bfh{{\mathbf{h}}}

\def\bfn{{\mathbf{n}}}

\def\bfr{{\mathbf{r}}}

\def\bfu{{\mathbf{u}}}
\def\bfv{{\mathbf{v}}}

\def\bfx{{\mathbf{x}}}
\def\bfy{{\mathbf{y}}}
\def\bfz{{\mathbf{z}}}

\def\bfA{{\mathbf{A}}}

\def\bfD{{\mathbf{D}}}

\def\bfF{{\mathbf{F}}}

\def\bfH{{\mathbf{H}}}
\def\bfI{{\mathbf{I}}}
\def\bfJ{{\mathbf{J}}}

\def\bfS{{\mathbf{S}}}

\def\bfW{{\mathbf{W}}}

\def\bsx{{\boldsymbol{x}}}

\newcommand{\diag}[1]{\mathrm{diag}{\left\{#1\right\}}}

\newcommand{\Id}[1]{\textbf{I}_{#1}}

\newcounter{algo}
\renewcommand{\thealgo}{\arabic{algo}}

\def\bstheta{{\boldsymbol{\theta}}}

\def\bsrho{{\boldsymbol{\rho}}}
\def\bsnu{{\boldsymbol{\nu}}}

\def\bsLambda{{\boldsymbol{\Lambda}}}
\def\bslambda{{\boldsymbol{\lambda}}}
\def\bsSigma{{\boldsymbol{\Sigma}}}
\def\bsPsi{{\boldsymbol{\Psi}}}
\newcommand{\argmin}{\arg\!\min}

\def\prox{\textrm{prox}}

\newcommand{\EigSq}{\underline{\bsLambda}}
\newcommand{\tcpp}[1]{\tcp{\emph{\small{#1}}}}

\newlength{\tempdima}

\newcommand{\rowname}[1]
{\rotatebox{90}{\makebox[\tempdima][c]{\textbf{#1}}}}

\newcommand{\zhaon}[1]{\textcolor[rgb]{0.00,0.00,0.00}{#1}}

\newcommand{\AB}[1]{\textcolor[rgb]{0.00,0.00,0.00}{#1}}
\newcommand{\zhaov}[1]{\textcolor[rgb]{0.00,0.00,0.00}{#1}}
\newcommand{\nd}[1]{\textcolor[rgb]{0.00,0.00,0.00}{#1}}
\newcommand{\Qi}[1]{\textcolor[rgb]{0.00,0.00,0.00}{#1}}
\newcommand{\revAQ}[1]{\textcolor[rgb]{0.00,0.00,0.00}{#1}}

  {
      \theoremstyle{plain}
      \newtheorem*{assumption1f}{Assumption 1}
      \newtheorem*{assumption2f}{Assumption 2}
      \newtheorem{lemma}{Lemma}
      \newtheorem{theorem}{Theorem}
  }

\begin{document}
\title{Fast Single Image Super-resolution using a New Analytical Solution for $\ell_2-\ell_2$ Problems}	
\author{Ningning Zhao, Qi Wei, Adrian Basarab, Nicolas Dobigeon, Denis Kouam{\'e} and Jean-Yves Tourneret
\thanks{Part of this work has been supported by the Chinese Scholarship Council and the thematic trimester on image processing of the CIMI Labex, Toulouse, France, under grant ANR-11-LABX-0040-CIMI within the program ANR-11-IDEX-0002-02.}
\thanks{Ningning Zhao, Nicolas Dobigeon and Jean-Yves Tourneret are with University of Toulouse, IRIT/INP-ENSEEIHT, 31071 Toulouse Cedex 7, France (e-mail: \{nzhao, nicolas.dobigeon, jean-yves.tourneret\}@enseeiht.fr).}
\thanks{Qi Wei is with Department of Engineering, University of Cambridge, CB21PZ, U.K. (e-mail: \{qw245\}@cam.ac.uk).}
\thanks{
Adrian Basarab and Denis Kouam{\'e} are with University of Toulouse, IRIT, CNRS UMR 5505, 118 Route de Narbonne, F-31062, Toulouse Cedex 9, France (e-mail: \{adrian.basarab, denis.kouame\}@irit.fr).} }

\maketitle
\begin{abstract}
This paper addresses the problem of single image super-resolution (SR), which consists of recovering a high resolution image from its blurred, decimated and noisy version.
The existing algorithms for single image SR \AB{use different strategies to handle the decimation and blurring operators. In addition to the traditional first-order gradient methods, recent techniques investigate splitting-based methods dividing the SR problem into up-sampling and deconvolution steps that can be easily solved.}
Instead of following this splitting strategy, we propose to deal with the decimation and blurring operators simultaneously by taking advantage of their particular properties in the frequency domain, leading to a new fast SR approach.
\AB{Specifically, an analytical solution can be obtained and implemented efficiently for the Gaussian prior or any other regularization that can be formulated into an $\ell_2$-regularized quadratic model, \nd{i.e., an $\ell_2$-$\ell_2$ optimization problem}. Furthermore, the flexibility of the proposed SR scheme is shown through the use of various priors/regularizations, ranging from generic image priors to learning-based approaches.} In the case of non-Gaussian priors, we show \AB{how} the analytical solution derived from the Gaussian case can be embedded into
\AB{traditional splitting frameworks}, \AB{allowing the computation cost of existing algorithms to be decreased significantly}. Simulation results conducted on several images with different priors illustrate the effectiveness of our fast SR approach compared with the existing techniques.
\end{abstract}
\section{Introduction}
\IEEEPARstart{S}{ingle} image super-resolution (SR), also known as image scaling up or image enhancement, aims at estimating a high-resolution (HR) image from a low-resolution (LR) observed image \cite{Park2003_SR_Overview}. This resolution enhancement problem is still an ongoing research problem with applications in various fields, such as remote sensing \cite{Martin_Hyper2015}, video surveillance \cite{yang2010image}, hyperspectral \cite{AkgunTIP2005}, microwave\cite{Yanovsky2015} or medical imaging \cite{Morin2012_SR}.

The methods dedicated to single image SR can be classified into three categories \cite{Yang2010TIP,SunJ_CVPR_2008,YWTai_CVPR_2010}.
The first category includes the interpolation based algorithms such as nearest neighbor interpolation, bicubic interpolation \cite{Thevenaz_2000} or adaptive interpolation techniques \cite{Zhang2008TIP,Mallat2010}.
Despite their simplicity and easy implementation, it is well-known that these algorithms generally over-smooth the high frequency details. The second type of methods consider learning-based (or example-based) algorithms that learn the relations between LR and HR image patches from a given database \cite{Freeman2000_SR,Glasner2009_SR, Huang_CVPR_2015, Zeyde2012,Yang2010TIP}. Note that the effectiveness of the learning-based algorithms highly depends on the training image database and these algorithms have generally a high computational complexity.
Reconstruction-based approaches that are considered in this paper belong to the third category of SR approaches \cite{SunJ_CVPR_2008,SunJ_TIP_2011,YWTai_CVPR_2010,MNg2010SR_TV}. These approaches formulate the image SR as an reconstruction problem, either by incorporating priors in a Bayesian framework or by introducing regularizations into the ill-posed inverse problem.

Existing reconstruction-based techniques used to solve the single image SR include the first order gradient-based methods \cite{SunJ_CVPR_2008,SunJ_TIP_2011,YWTai_CVPR_2010,Yang2010TIP}, the iterative shrinkage thresholding-based algorithms \cite{BeckTeboulle2009} (also called forward-backward algorithms), proximal gradient algorithms and other variable splitting algorithms that rely on the augmented Lagrangian (AL) scheme. The AL based algorithms include the alternating direction method of multipliers (ADMM) \cite{MNg2010SR_TV,Martin_Hyper2015,Morin2012_SR,Marquina2008}, split Bregman (SB) methods \cite{Yanovsky2015} (known to be equivalent to ADMM in certain conditions \cite{YinWotao2008}) and their variants.

Particularly, Ng \textit{et. al.} \cite{MNg2010SR_TV} proposed an ADMM-based algorithm to solve a TV-regularized single image SR problem, where the decimation and blurring operators are split and solved iteratively. Due to this splitting, the cumbersome SR problem can be decomposed into an up-sampling problem and a deconvolution problem, that can be both solved efficiently. Yanovsky \textit{et. al.} \cite{Yanovsky2015} proposed to solve the same problem with an SB algorithm. However, the \nd{decimation operator} was handled through a gradient descent method integrated in the SB framework. Sun \textit{et. al.} \cite{SunJ_CVPR_2008,SunJ_TIP_2011} proposed a gradient profile prior and formulated the single image SR problem as an $\ell_2$-regularized optimization problem, further solved with the gradient descent method.  Yang \textit{et. al.} \cite{Yang2010TIP} proposed a learning-based algorithm for the single image SR by seeking a sparse representation using the patches of LR and HR images, followed by back projecting through a gradient descent method.
Despite the efficiency of these methods, it is still appealing to deal with the single image SR problem in a non-iterative or more efficient way.

\revAQ{This paper aims at reducing the computational cost of these methods by proposing a new approach handling the decimation and blurring operators simultaneously by exploring their intrinsic properties in the frequency domain.} \nd{It is interesting to note that similar properties were explored in \cite{Robinson2010,Sroubek2011} for multi-frame SR.} \revAQ{However, the implementation of the matrix inversions proposed in \cite{Robinson2010,Sroubek2011} are less efficient than those proposed in this work,} \nd{as it will be demonstrated in the complexity analysis conducted in Section \ref{sec:TikAS}. More precisely, this paper  derives a closed-form expression of the solution associated with the $\ell_2$-penalized least-squares SR problem, when the observed LR image is assumed to be a noisy, subsampled and blurred version of the HR image with a spatially invariant blur. This model, referred to as $\ell_2-\ell_2$ in what follows, underlies the restoration of an image contaminated by additive Gaussian noise and has been used intensively for the single image SR problem, see, e.g., \cite{Yang2010TIP,SunJ_CVPR_2008,Ebrahimi2008} and the references mentioned above. The proposed solution is shown to be easily embeddable into an AL framework to handle non-Gaussian priors (i.e., non-$\ell_2$ regularizations), which significantly lightens the computational burdens of several existing SR algorithms.}

The remainder of the paper is organized as follows. Section II formulates the single image SR problem as an optimization problem. In Section III, we study the properties of the down-sampling and blurring operators in the frequency domain and introduce a fast SR scheme based on an analytical solution of the $\ell_2-\ell_2$ model, that can be formulated in the image or gradient domains. Section IV generalizes the proposed fast SR scheme to more complex regularizations in image or transformed domains. Various experiments presented in Section V demonstrate the efficiency of the proposed fast single image SR scheme. Conclusions and perspectives are finally reported in Section VI.

\section{Image Super-resolution Formulation}

\subsection{Model of Image Formation}
In the single image SR problem, the observed LR image is modeled as a noisy version of the blurred and decimated HR image to be estimated as follows,
\begin{equation}
\bfy = \bfS\bfH\bfx + \bfn
\label{image_form}
\end{equation}
where the vector $\bfy \in \mathbb{R}^{N_l\times 1}$ $(N_l=m_l\times n_l)$ denotes the LR observed image and
$\bfx \in \mathbb{R}^{N_h\times 1}$ $(N_h= m_h\times n_h)$ is the vectorized HR image to be estimated, with $N_h> N_l$.
The vectors $\bfy$ and $\bfx$ are obtained by stacking the corresponding images (LR image $\in \mathbb{R}^{m_l\times n_l}$ and HR image $\in \mathbb{R}^{m_h\times n_h}$) into column vectors in a lexicographic order. \zhaon{Note that} the vector $\bfn \in \mathbb{R}^{N_l \times 1}$ \zhaon{is} an independent identically distributed (\textit{i.i.d.}) additive white Gaussian noise (AWGN) \zhaon{and that the matrices} $\bfS \in \mathbb{R}^{N_l\times N_h}$ and $\bfH \in \mathbb{R}^{N_h\times N_h}$ represent the decimation and the blurring/convolution operations respectively. More specifically, $\bfH$ is a block circulant matrix with circulant blocks, which corresponds to cyclic convolution boundaries, and left multiplying by $\bfS$ performs down-sampling with an integer factor $d$ ($d = d_r \times d_c$), i.e., $N_h = N_l \times d$.
The decimation factors  $d_r$ and $d_c$ represent the numbers of discarded rows and columns from the input images satisfying the following relationships $m_h=m_l \times d_r$ and $n_h = n_l \times d_c$.
\AB{Note that the image formation model \eqref{image_form} has been widely considered in single image SR problems, see, e.g., \cite{Yang2010TIP,SunJ_CVPR_2008,SunJ_TIP_2011,MNg2010SR_TV,Zhang2012SR}}.

We introduce two additional basic assumptions about the blurring and decimation operators. These assumptions have been widely used for image deconvolution or image SR problems (see, e.g., \cite{MElad1997,Farsiu2004_SR,Zeyde2012,Yang2010TIP}) \AB{and are necessary for the proposed fast SR framework}.

\begin{assumption1f}
\label{as:1}
The blurring matrix $\bfH$ is the matrix representation of the cyclic convolution operator, \textit{i.e.}, $\bfH$ is a block circulant matrix with circulant blocks (BCCB).
\end{assumption1f}
\nd{This assumption has been widely resorted in the image processing literature \cite{LinPAMI2004,Robinson2010,Sroubek2011}. It is satisfied provided the underlying blurring kernel is shift-invariant and the boundary conditions make the convolution operator periodic. Note that the BCCB matrix assumption does not depend on the shape of the blurring kernel, i.e., it is satisfied for any kind of blurring, including motion blur, out-of-focus blur, atmospheric turbulence, etc.} Using the cyclic convolution assumption, the blurring matrix and its conjugate transpose can be decomposed as
\begin{align}
\bfH &= \bfF^H \bsLambda \bfF      \label{eq_BF}\\
\bfH^H &= \bfF^H \bsLambda^H \bfF  \label{eq_BCF}
\end{align}
where the matrices $\bfF$ and $\bfF^H$ \zhaon{are associated with} the Fourier and inverse Fourier transforms
(satisfying $\bfF\bfF^H = \bfF^H\bfF = \bfI_{N_h}$) and $\bsLambda = \diag{\bfF\bfh} \in \mathbb{C}^{N_h\times N_h}$ is a diagonal matrix, whose diagonal elements are the Fourier coefficients of the first column of the blurring matrix $\bfH$, \zhaon{denoted as $\bfh$}.
Using the decompositions \eqref{eq_BF} and \eqref{eq_BCF}, the blurring operator $\bfH\bfx$ and its conjugate $\bfH^H\bfx$ can be efficiently computed in the frequency domain, see, e.g., \cite{JamesNg2006,MElad1997SR,NZHAO2015}.

\begin{assumption2f}\label{as:1}
The decimation matrix $\bfS \in \mathbb{R}^{N_l \times N_h}$ is a down-sampling operator,
while its conjugate transpose $\bfS^H\in \mathbb{R}^{N_h \times N_l}$ interpolates the decimated image with zeros.
\end{assumption2f}

\begin{figure}[htpt]
\centering
\includegraphics[width=0.4\linewidth]{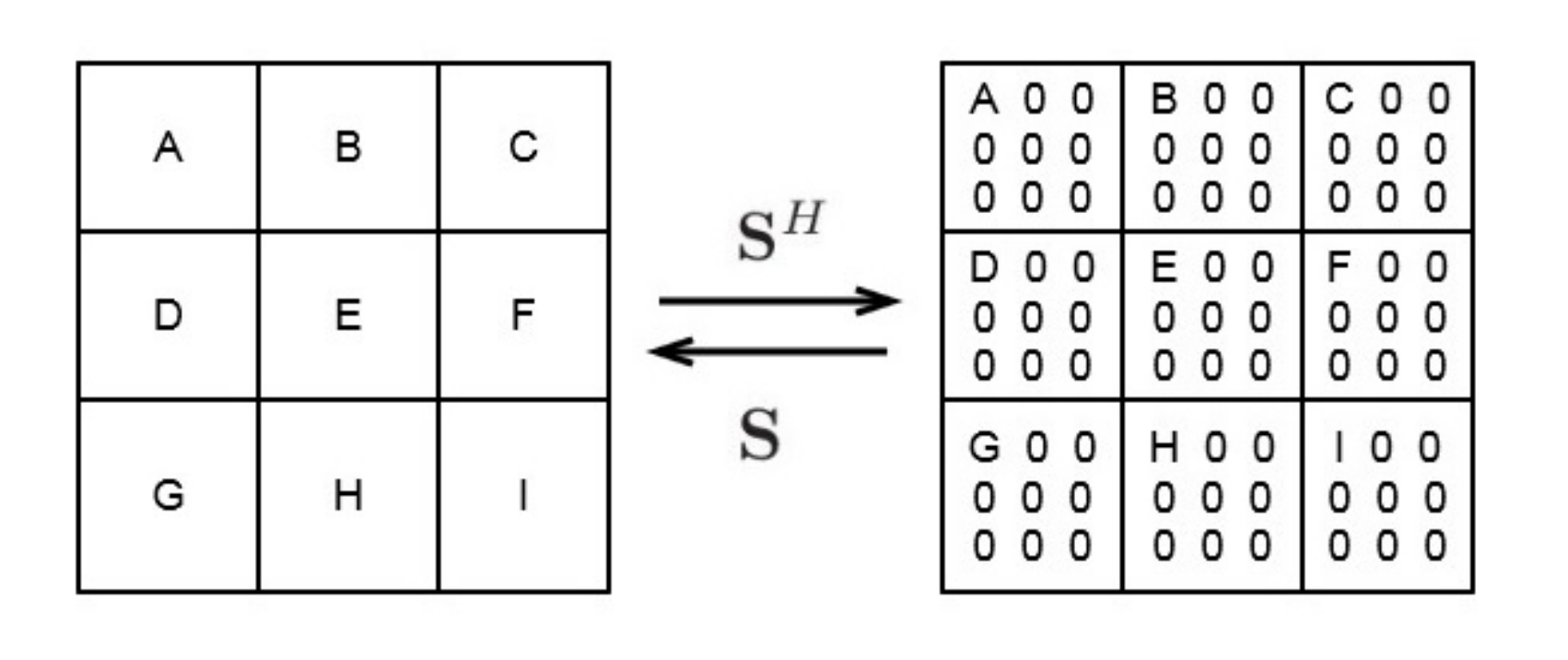}
\caption{Effect of the up-sampling matrix $\bfS^H$ on a $3\times 3$ image and of the down-sampling matrix $\bfS$ on the corresponding $9\times 9$ image (whose scale up factor equals 3).}
\label{fig_sampling_matrix}
\end{figure}

\nd{Once again, numerous research works have used this assumption \cite{Robinson2010,Sroubek2011,Yang2010TIP,Zeyde2012}.} Fig. \ref{fig_sampling_matrix} shows a toy example highlighting the roles of the
decimation matrix $\bfS$ and its conjugate transpose $\bfS^H$. The decimation matrix
satisfies the relationship $\bfS\bfS^H = \bfI_{N_l}$. Denoting $\underline{\bfS} \triangleq \bfS^H\bfS$,
multiplying an image by $\underline{\bfS}$ can be achieved by making an entry-wise
multiplication with an $N_h \times N_h$ mask having ones at the sampled positions and zeros elsewhere.

\subsection{Problem formulation}
Similar to traditional image reconstruction problems, the estimation of an HR image from the observation of an LR image is not invertible, leading to an ill-posed problem. This ill-posedness is classically overcome by incorporating some appropriate prior information or regularization term. The regularization term can be chosen from a specific task of interest, the information resulting from previous experiments or from a perceptual view on the constraints affecting the unknown model parameters \cite{Robert2007,Gelman2013}.
Various priors or regularizations have already been advocated to regularize the image SR problem in the literature include: \AB{(i) traditional generic image priors such as} Tikhonov \cite{Nguyen2001TIP,QiWEI2015FastFusion,Ebrahimi2008}, the total variation (TV) \cite{MNg2010SR_TV,SR_Aly2005,Marquina2008} and the sparsity in transformed domains \cite{GEM_Bioucas-Dias2006,JamesNg2007,JijiCV2004,Figueiredo2003TIP}, (ii) more recently proposed image regularizations such as the gradient profile prior \cite{SunJ_CVPR_2008,SunJ_TIP_2011,YWTai_CVPR_2010} or Fattal's edge statistics \cite{Fattal2009EWA} and (iii) learning-based priors \cite{Roth2005CVPR,Zoran2011ICCV}. The fast approach proposed in the next section is shown to be adapted to many of the existing regularization terms. Note that proposing new regularization terms with improved SR performance is out of the scope of this paper.

Assuming that the noise $\bfn$ in \eqref{image_form} is AWGN and incorporating a proper regularization to the target image $\bfx$, the maximum \textit{a posteriori} (MAP) estimator of $\bfx$ for the single image SR can be obtained by solving the following optimization problem
\begin{equation}
\min_{\bfx} \frac{1}{2}  \underbrace{\|\bfy - \bfS\bfH\bfx\|_2^2}_{\substack{\text{data fidelity}}}
 + \quad \tau \underbrace{\phi(\bfA\bfx)}_{\substack{\text{regularization}}}
\label{eq_Target_Prob}
\end{equation}
where $\|\bfy - \bfS\bfH\bfx\|_2^2$ is a \textit{data fidelity term} associated with the model likelihood and $\phi(\bfA\bfx)$ is related to the image prior information and is referred to as \textit{regularization} or \textit{penalty} \cite{Engl1996}. \AB{Note that the matrix $\bfA$ can be the identity matrix when the regularization is imposed on the SR image itself, the gradient operator, \zhaov{any orthogonal matrix or normalized tight frame}, depending on the addressed application and the properties of the target image.} The role of the regularization parameter $\tau$ is to weight the importance of the regularization term with respect to (w.r.t.) the data fidelity term. \nd{The next section derives a closed-form solution of the problem \eqref{eq_Target_Prob} for a quadratic regularizing operator $\phi(\cdot)$ when the assumptions 1 and 2 hold.}

\section{Proposed fast super-resolution using an $\ell_2$-regularization}
\label{sec:TikAS}

\AB{Before proceeding to more complicated regularizations investigated in Section \ref{sec:generalization}, we first consider the basic
$\ell_2$-norm regularization defined by
\begin{equation}
\label{eq:l2_regularizer}
  \phi(\bfA\bfx) = \|\bfA\bfx - {\bfv}\|_2^2
\end{equation}
\Qi{where the matrix $\bfA^H \bfA$ is assumed, unless otherwise specified, to be invertible. Typical examples of matrices $\bfA$ include the Fourier transform matrix, the wavelet transform matrix, etc.}
\Qi{Under this $\ell_2$-norm regularization}, a generic form of a fast solution \Qi{to problem \eqref{eq_Target_Prob} will be} derived in Section \ref{subsec:l2_general}. Then, two particular cases of this regularization widely used in the literature will be discussed in Sections \ref{subsec:l2_image} and \ref{subsec:l2_gradient}.}

\subsection{\nd{Proposed closed-form solution for the $\ell_2-\ell_2$ problem}}
\label{subsec:l2_general}
With the regularization \eqref{eq:l2_regularizer}, the problem \eqref{eq_Target_Prob} transforms to
\begin{equation}
\min_{\bfx} \frac{1}{2} \|\bfy - \bfS\bfH\bfx\|_2^2 + \tau \| \bfA \bfx - \bfv\|_2^2
\label{l2_problem_generic}
\end{equation}
whose solution is given by
\begin{equation}
\hat{\bfx} = (\bfH^H\underline{\bfS}\bfH +2\tau \bfA^H\bfA)^{-1}(\bfH^H\bfS^H\bfy +2\tau\bfA^H\bfv)
\label{l2_anas_generic}
\end{equation}
with $\underline{\bfS}=\bfS^H\bfS$.

Direct computation of the analytical solution \zhaon{\eqref{l2_anas_generic}} requires the inversion of a high dimensional matrix, whose computational complexity is of order $\mathcal{O} (N_h^3)$. \zhaon{One can think of using} optimization or simulation-based methods to overcome this computational difficulty. The optimization-based methods, such as the gradient-based methods \cite{SunJ_TIP_2011} or, more recently, the ADMM \cite{MNg2010SR_TV} and SB \cite{Yanovsky2015} method approximate the solution of \eqref{l2_problem_generic} by iterative updates. The simulation-based methods, e.g., the Markov Chain Monte Carlo methods \cite{Feron2015,Orieux2012,Idier2015}, are drawing samples from a multivariate posterior distribution (which is Gaussian for a Tikhonov regularization) and compute the average of the generated samples to approximate the minimum mean square error (MMSE) estimator of $\bfx$. However, simulation-based methods have the major drawback of being computationally expensive, which prevents their effective use when processing large images. \nd{Moreover, because of the particular structure of the decimation matrix, the joint operator $\bfS\bfH$ cannot be diagonalized in the frequency domain, which prevents any direct implementation of the solution \eqref{l2_anas_generic} in this domain.} The main contribution in this work is proposing a new scheme to compute \eqref{l2_anas_generic} explicitly, getting rid of any statistically sampling or iterative update and leading to a fast SR method.

In order to compute the analytical solution \eqref{l2_anas_generic}, a property of the decimation matrix in the frequency domain is first stated in Lemma \ref{lemma:1}.

\begin{lemma}[Wei \textit{et al.}, \cite{QiWEI2015FastFusion}]
The following equality holds
\begin{equation}
\bfF\underline{\bfS}\bfF^H = \frac{1}{d}\bfJ_d \otimes \bfI_{N_l}
\end{equation}
where $\bfJ_d \in \mathbb{R}^{d\times d}$ is a matrix of ones,
$\bfI_{N_l} \in \mathbb{R}^{N_l\times N_l}$ is the $N_l\times N_l$ identity matrix and $\otimes$ is the Kronecker product.
\label{lemma:1}
\end{lemma}

Using the property of the matrix $\bfF\underline{\bfS}\bfF^H$ given in Lemma \ref{lemma:1} and taking into account the assumptions mentioned above, the analytical solution \eqref{l2_anas_generic} can be rewritten as
\begin{equation}
\hat{\bfx} =\bfF^H \left(\frac{1}{d} \EigSq^H \EigSq +2\tau \Qi{\bfF\bfA^H\bfA\bfF^H} \right)^{-1}\bfF\left(\bfH^H\bfS^H\bfy +2\tau\bfA^H{\bfv}\right)
\label{l2_anas_generic_v2}
\end{equation}
where the matrix $\EigSq \in \mathbb{C}^{N_l\times N_h}$ is defined as
\begin{equation}
\EigSq=[\bsLambda_1,\bsLambda_2, \cdots,\bsLambda_d]
\label{eq_lambda_underline}
\end{equation}
\revAQ{and where the blocks $\bsLambda_i \in \mathbb{C}^{N_l\times N_l}$ ($i = 1,\cdots,d$) satisfy the relationship}
\begin{equation}
\revAQ{\diag{\bsLambda_1,\cdots,\bsLambda_d}=\bsLambda.}
\end{equation}
The readers may refer to the Appendix \ref{app:theorem1} for more details about the derivation of \eqref{l2_anas_generic_v2} from \eqref{l2_anas_generic}.

To further simply the expression \eqref{l2_anas_generic_v2}, we propose to use the following Woodbury inverse formula.
\begin{lemma}[Woodbury formula \cite{Hager1989}]
The following equality holds conditional on the existence of $\bfA_1^{-1}$
and $\bfA_3^{-1}$
\begin{equation}
\begin{split}
&(\bfA_1+\bfA_2 \bfA_3 \bfA_4)^{-1} \\
&= \bfA_1^{-1} - \bfA_1^{-1}\bfA_2(\bfA_3^{-1} + \bfA_4 \bfA_1^{-1}\bfA_2)^{-1}\bfA_4 \bfA_1^{-1}
\end{split}
\label{eq_lemma_2}
\end{equation}
where $\bfA_1$, $\bfA_2$, $\bfA_3$ and $\bfA_4$ are matrices of correct sizes.
\label{lemma:2}
\end{lemma}
Taking into account the Woodbury formula of Lemma \ref{lemma:2}, the analytical solution \eqref{l2_anas_generic_v2} can be computed very efficiently
as stated in the following theorem.

\begin{theorem}
\label{the:Ubar}
When Assumptions 1 and 2 are satisfied, the solution of Problem \eqref{l2_problem_generic} can be computed using the following closed-form expression
\begin{equation}
\begin{split}
\hat{\bfx} &= \frac{1}{2\tau }\bfF^H\bsPsi\bfF\bfr \\
&- \frac{1}{2\tau } \bfF^H\bsPsi\EigSq^H\left(2\tau d\bfI_{N_l} + \EigSq\bsPsi\EigSq^H  \right)^{-1}\EigSq\bsPsi\bfF\bfr
\label{eq_anas_gl2}
\end{split}
\end{equation}
where $\bfr = \bfH^H\bfS^H\bfy +2\tau\bfA^H\bfv$, $\bsPsi = \Qi{\bfF\left(\bfA^H\bfA\right)^{-1}\bfF^H}$ and $\EigSq$ is defined in \eqref{eq_lambda_underline}.
\end{theorem}
\begin{proof}
See Appendix \ref{app:theorem1}.
\end{proof}

\vspace{+0.2cm}
\noindent \textbf{\textit{Complexity Analysis}}\\
\revAQ{The most computationally expensive part for the computation of \eqref{eq_anas_gl2} in Theorem \ref{the:Ubar} is the implementation of FFT/iFFT.} In total, four FFT/iFFT computations are required in our implementation.
Comparing with the original problem \eqref{l2_anas_generic}, the order of computation complexity has decreased significantly from $\mathcal{O}(N_h^3)$ to $\mathcal{O}(N_h\log N_h)$, which allows the analytical solution \eqref{eq_anas_gl2} to be computed efficiently.
\revAQ{Note that \cite{Robinson2010,Sroubek2011} also addressed image SR problems by using the properties of $\underline{\bfS}$ in the frequency domain, where $N_l$ small matrices of size $d\times d$ were inverted. The total computational complexity of the methods investigated in \cite{Robinson2010,Sroubek2011} is $\mathcal{O}(N_h\log N_h + N_h d^2)$.} \nd{Another important difference with our work is that the authors of \cite{Robinson2010} and \cite{Sroubek2011} decomposed the SR problem into an upsampling (including motion estimation which is not considered in this work) and a deblurring step. The operators $\bfH$ and $\bfS$ were thus considered separately, thus requiring two $\ell_2$ regularizations for the blurred image (referred to as $\bfz$ in \cite{Robinson2010}) and the ground-truth image (referred to as $\bfx$ in \cite{Robinson2010}). On the contrary, this work considers the blurring and downsampling jointly and achieve the SR in one step, requiring only one regularization term for the unknown image. It is worthy to mention that the proposed SR solution can be extended to incorporate the warping operator considered in \cite{Robinson2010,Sroubek2011}, which can also be modelled as a BCCB matrix. This is not included in this paper but will be considered in future work.}

In the sequel of this section, two particular instances of the $\ell_2$-norm regularization are considered, defined in the image and gradient domains, respectively.

\subsection{Solution of the $\ell_2-\ell_2$ problem in the image domain \label{subsec:l2_image}}
First, we consider the specific case where $\bfA =  \Id{N_h}$ and $\bfv = \bar{\bfx}$, i.e.,  the problem \eqref{l2_problem_generic} turns to
\begin{equation}
\min_{\bfx} \frac{1}{2} \|\bfy - \bfS\bfH\bfx\|_2^2 + \tau \|  \bfx - \bar{\bfx}\|_2^2.
\label{l2_problem_image}
\end{equation}
This implies that the target image $\bfx$ is \emph{a priori} close to the image $\bar{\bfx}$. The image $\bar{\bfx}$ can be an estimation of the HR image, e.g., an interpolated version
of the observed image, \AB{a restored image obtained with learning-based algorithms \cite{Yang2010TIP}} or a cleaner image obtained from other sensors \cite{QiWEI2015,QiWEI2015FastFusion,Ebrahimi2008}.
In such case, using Theorem \ref{the:Ubar}, the solution of the problem \eqref{l2_problem_image} is
\begin{equation}
\hat{\bfx}= \frac{1}{2\tau}\bfr - \frac{1}{2\tau}\bfF^H\EigSq^H
\left( 2\tau d\bfI_{N_l} + \EigSq\EigSq^H \right)^{-1}\EigSq\bfF\bfr
\label{l2_analytic_image}
\end{equation}
with $\bfr= \bfH^H\bfS^H\bfy +2\tau\bar{\bfx}$.

Algorithm \ref{Algo:FastSR} summarizes the implementation of the proposed SR solution \eqref{l2_analytic_image}, which is referred to as \textit{fast super-resolution (FSR)} approach.

\begin{algorithm}[h!]
\label{Algo:FastSR}
\KwIn{$\bfy$, $\bfH$, $\bfS$, $\bar{\bfx}$, $\tau$, $d$}
\BlankLine
 \tcpp{Factorization of $\bfH$ (FFT of the blurring kernel)}
$\bfH= {\bfF^H \bsLambda \bfF}$\;
 \tcpp{Compute $\EigSq$}
$\EigSq = [\bsLambda_1,\bsLambda_2, \cdots,\bsLambda_d]$\;
 \tcpp{Calculate FFT of $\bfr$ denoted as $\bfF\bfr$}
$\bfF\bfr = \bfF(\bfH^H\bfS^H\bfy +2\tau\bar{\bfx})$\;
 \tcpp{Hadamard (or entrywise) product in frequency domain}
$\bfx_f = \left(\EigSq^H \left( 2\tau d\bfI_{N_l} + \EigSq\EigSq^H \right)^{-1}\EigSq\right)\bfF\bfr$\;
 \tcpp{Compute the analytical solution}
$\hat{\bfx} = \frac{1}{2\tau} \left(\bfr - \bfF^H\bfx_f \right)$ \;
\KwOut{$\hat{\bfx}$}
\caption{\AB{FSR with image-domain $\ell_2$-regularization: implementation of the analytical solution \eqref{l2_analytic_image}}}
\DecMargin{1em}
\end{algorithm}

\subsection{Solution of the $\ell_2-\ell_2$ problem in the gradient domain \label{subsec:l2_gradient}}
Generic image priors defined in the gradient domain have been successfully used for image reconstruction, avoiding the common ringing artifacts see, e.g., \cite{SunJ_CVPR_2008,SunJ_TIP_2011,YWTai_CVPR_2010}. In this part, we focus on the gradient profile prior proposed in \cite{SunJ_TIP_2011} for the single image SR problem. This prior consists of considering the regularizing term $\| \nabla \bfx - \bar{\nabla\bfx}\|_2^2$, thus the problem \eqref{l2_problem_generic} turns to
\begin{equation}
\min_{\bfx} \frac{1}{2} \|\bfy - \bfS\bfH\bfx\|_2^2 + \tau \| \nabla \bfx - \bar{\nabla\bfx}\|_2^2
\label{l2_problem_gradient}
\end{equation}
where $\nabla$ is the discrete version of the gradient $\nabla:=[\partial_{\rm{h}}, \partial_{\rm{v}}]^T$ and $\bar{\nabla\bfx}$ is \zhaov{the estimated gradient field}. More explanations about the motivations for using the gradient field may be found in \cite{SunJ_CVPR_2008,SunJ_TIP_2011}. For an image $\bsx \in \mathbb{R}^{m\times n}$, under the \Qi{periodic} boundary conditions, the numerical definitions of the gradient operators are
\begin{align}
(\partial_{\rm{h}} \bsx)(i,j) &= \bsx((i+1) \ \mathrm{mod} \ m,j) - \bsx(i,j) \; &\textrm{if} \; i \leq m 
\nonumber
\\
(\partial_{\rm{v}} \bsx)(i,j) &= \bsx(i,(j+1) \ \mathrm{mod} \  n) - \bsx(i,j) \; &\textrm{if} \; j \leq n
\nonumber
\end{align}
\Qi{where $\partial_{\rm{h}}$ and $\partial_{\rm{v}}$ are the horizontal and vertical gradients}. \Qi{The gradient operators can be rewritten as two BCCB matrices $\bfD_{\rm{h}}$ and $\bfD_{\rm{v}}$ corresponding to
the horizontal and vertical discrete differences of an image, respectively.} Therefore, two diagonal matrices $\bsSigma_{\rm{h}}$ and $\bsSigma_{\rm{v}}$ ($\mathbb{C}^{N_h\times N_h}$) are obtained by decomposing $\bfD_{\rm{h}}$ and $\bfD_{\rm{v}}$ in the frequency domain, i.e.,
\begin{equation}
\bfD_{\rm{h}} = \bfF^H\bsSigma_{\rm{h}}\bfF \; \textrm{and} \; \bfD_{\rm{v}} = \bfF^H\bsSigma_{\rm{v}}\bfF.
\end{equation}

Thus, the problem \eqref{l2_problem_gradient} can be transformed into
\begin{equation}
\min_{\bfx} \frac{1}{2} \|\bfy - \bfS\bfH\bfx\|_2^2 + \tau \| \Qi{\bfA}  \bfx - \bfv\|_2^2
\label{eq_probform_gradl2_v2}
\end{equation}
with $\Qi{\bfA} = [\bfD_{\rm{h}}^T,\bfD_{\rm{v}}^T] \in \mathbb{R}^{2N_h \times N_h}$ and using the notation $\bar{\nabla\bfx} = \bfv = [\bfv_{\rm{h}}, \bfv_{\rm{v}}]^T \in\mathbb{R}^{2N_h \times 1}$.
\Qi{Note that the invertibility of $\bfA^H \bfA$ is violated here because of the periodic boundary assumption. Thus, adding a small $\ell_2$-norm regularization $\tau \sigma \|\bfx\|_2^2$ (where $\sigma$ is a very small constant) to \eqref{eq_probform_gradl2_v2} can circumvent this invertibility problem while keeping the solution close to the original regularization.}
Using Theorem \ref{the:Ubar}, the analytical solution of \eqref{eq_probform_gradl2_v2} (\Qi{including the additional small $\ell_2$-norm term}) is given by \eqref{eq_anas_gl2} with $\bsPsi = \left(\bsSigma_{\rm{h}}^H \bsSigma_{\rm{h}} + \bsSigma_{\rm{v}}^H \bsSigma_{\rm{v}} +\Qi{\sigma\Id{N_h}}\right)^{-1}$.

\AB{The pseudocode used to implement this solution is summarized in Algo. \ref{Algo_FSR_Gl2}.}
\begin{algorithm}[h!]
\label{Algo_FSR_Gl2}
\KwIn{$\bfy$, $\bfH$, $\bfS$, $\bfD_{\rm{h}}$, $\bfD_{\rm{v}}$, $\bar{\nabla\bfx}$, $\tau$, $d$}
\BlankLine
 \tcpp{Factorizations of matrices $\bfH$, $\bfD_h$, $\bfD_v$}
$\bfH= \bfF^{H} \bsLambda \bfF$\;
$\bfD_{\rm{h}} = \bfF^H\bsSigma_{\rm{h}}\bfF$\;
$\bfD_{\rm{v}} = \bfF^H\bsSigma_{\rm{v}}\bfF$\;
 \tcpp{Compute $\EigSq$ and $\bsPsi$}
$\EigSq = [\bsLambda_1,\bsLambda_2, \cdots,\bsLambda_d]$\;
$\bsPsi = (\bsSigma_{\rm{h}}^H\bsSigma_{\rm{h}} + \bsSigma_{\rm{v}}^H\bsSigma_{\rm{v}}+\Qi{\sigma\Id{N_h}})^{-1}$\;
 \tcpp{Calculate FFT of $\bfr$ denoted as $\bfF\bfr$}
$\bfF \bfr = \bfF(\bfH^H\bfS^H\bfy +2\tau\bfD^H\bfv)$\;
 \tcpp{Hadamard (or entrywise) product in the frequency domain}
$\bfx_f = \left[\bsPsi\EigSq^H \left( \mu d\bfI_{N_l} + \EigSq\bsPsi\EigSq^H \right)^{-1}\EigSq\bsPsi\right]\bfF\bfr$\;
 \tcpp{Compute the analytical solution}
$\hat{\bfx} = \frac{1}{2\tau} \left(\bfF^H\bsPsi\bfF \bfr-\bfF^H\bfx_f \right)$\;
\KwOut{$\hat{\bfx}$}
\caption{\AB{FSR with gradient-domain $\ell_2$-regularization: implementation of the analytical solution of \eqref{l2_problem_gradient}}}
\DecMargin{1em}
\end{algorithm}

\section{Generalized fast super-resolution} \label{sec:generalization}
As mentioned previously, a large variety of non-Gaussian regularizations has been proposed for the single image SR problem, in both image or transformed domains. \nd{Many SR algorithms, e.g., \cite{MNg2010SR_TV,Yanovsky2015}, require to solve an $\ell_2-\ell_2$ problem similar to \eqref{l2_problem_generic} as an intermediate step. This section shows that the solution \eqref{eq_anas_gl2} derived in Section \ref{sec:TikAS} can be combined with existing SR iterative methods to significantly lighten their computational costs.}

\subsection{General form of the proposed algorithm \label{subsec:gene_propo}}
\revAQ{In order to use the analytical solution \eqref{eq_anas_gl2}} \nd{derived for the $\ell_2$-regularized SR problem into an ADMM framework,} the problem \eqref{eq_Target_Prob} is rewritten as the following constrained optimization problem
\begin{eqnarray}
&\min_{\bfx,\bfu}& \frac{1}{2} \|\bfy - \bfS\bfH\bfx \|_2^2 + \tau \phi(\bfu) \notag \\
&\textrm{subject to} &\bfA \bfx = \bfu.
\label{eq_Target_Prob_ctr}
\end{eqnarray}
The AL function associated with this problem is
\begin{equation*}
\begin{split}
&\mathcal{L}(\bfx,\bfu,\bslambda)= \\
&\frac{1}{2} \|\bfy - \bfS\bfH\bfu \|_2^2 + \tau \phi(\bfu)
+ \bslambda^T(\bfA\bfx - \bfu) + \frac{\mu}{2} \|\bfA\bfx - \bfu \|_2^2
\end{split}
\end{equation*}
or equivalently
\begin{equation}
\mathcal{L}(\bfx,\bfu,\bfd)= \frac{1}{2} \|\bfy - \bfS\bfH\bfu \|_2^2 + \tau \phi(\bfu)+ \frac{\mu}{2} \|\bfA\bfx - \bfu+\bfd \|_2^2.
\end{equation}
\Qi{To solve problem \eqref{eq_Target_Prob_ctr}, we need to minimize $\mathcal{L}(\bfx,\bfu,\bfd)$ w.r.t. $\bfx$ and $\bfu$ and 
update the scaled dual variable $\bfd$ iteratively as summarized in Algo. \ref{alg3_salsa_fsr}.}

Note that the 3rd step updating the HR image $\bfx$ can be solved analytically using Theorem 1.
The variable $\bfu$ is updated at the 4th step using the Moreau proximity operator whose definition is given by
\begin{equation}
\prox_{\lambda,\phi}(\bsnu) = \argmin_x \phi(\bsx) + \frac{1}{2\lambda} \| \bsx- \bsnu\|^2.
\end{equation}
The generic optimization scheme given in Algo. \ref{alg3_salsa_fsr}, including the non-iterative update of the HR image following Theorem 1, is detailed hereafter for three widely used regularization techniques, namely for the TV regularization \cite{MNg2010SR_TV}, the $\ell_1$-norm regularization in the wavelet domain \cite{JijiCV2004} and the learning-based method in \cite{Yang2010TIP}.

\begin{algorithm}
\label{alg3_salsa_fsr}
\KwIn{$\bfy$, $\bfS$, $\bfH$, $d$, $\tau$\;}
Set $k = 0$, choose $\mu>0$, $\bfu^0$, $\bfd^0$\;
\textbf{Repeat}\\
$\bfx^{k+1}=\argmin_{\bfx} \|\bfy-\bfS\bfH\bfx\|_2^2+\mu\|\bfA\bfx-\bfu^k+\bfd^k \|_2^2$\;
$\bfu^{k+1} = \argmin_{\bfu} \tau \phi(\bfu) + \frac{\mu}{2}\|\bfA\bfx^{k+1}-\bfu+\bfd^k \|_2^2$\;
$\bfd^{k+1} = \bfd^{k} + (\bfA\bfx^{k+1}-\bfu^{k+1})$\;
\textbf{until} stopping criterion is satisfied.
\caption{Proposed generalized fast super-resolution (FSR) scheme}
\DecMargin{1em}
\end{algorithm}

\subsection{TV regularization}
Using a TV prior, problem \eqref{eq_Target_Prob} can be rewritten as
\begin{equation}
\min_{\bfx} \frac{1}{2} \|\bfy - \bfS\bfH\bfx\|_2^2 + \tau\phi(\bfA\bfx)
\label{eq_probform_TV}
\end{equation}
where the regularization term is given by
\begin{equation}
\phi(\bfA\bfx) = \| \bfx\|_{\rm{TV}} = \sqrt{\|\bfD_{\rm{h}}\bfx\|^2 + \|\bfD_{\rm{v}}\bfx\|^2}
\end{equation}
with $\bfA = [\bfD_{\rm{h}},\bfD_{\rm{v}}]^T \in \mathbb{R}^{2N_h\times N_h}$. We can 
solve \eqref{eq_probform_TV} using Algo. \ref{alg3_salsa_fsr}, with the auxiliary variable
$\bfu = [\bfu_{\rm{h}}, \bfu_{\rm{v}}]^T \in \mathbb{R}^{2N_h \times 1}$ such that $\bfA\bfx = \bfu$.
The resulting pseudocodes of the proposed fast SR approach for solving \eqref{eq_probform_TV} are detailed
in Algo. \ref{alg3_admm_TV_fast}, which is reported in Appendix \ref{app:psd_code}.

\subsection{$\ell_1$-norm regularization in the wavelet domain}
Assuming that $\bfx$ can be decomposed as a linear combination of wavelets (e.g., as in\cite{GEM_Bioucas-Dias2006}), the SR can be conducted in the wavelet domain. 
Denote as $\bfx = \bfW\bstheta$ the wavelet decomposition of $\bfx$, where $\bstheta \in \mathbb{R}^{N_h\times1}$ is the vector containing the wavelet coefficients
and multiplying by the matrices $\bfW^H$ and $\bfW$ ($\in \mathbb{R}^{N_h \times N_h}$) represent the wavelet and inverse wavelet \Qi{transforms} 
(satisfying that $\bfW\bfW^H = \bfW^H\bfW=\bfI_{N_h}$). \Qi{The single image SR with $\ell_1$-norm regularization in the wavelet domain 
can be formulated} as follows
\Qi{
\begin{equation}
\min_{\bfx} \frac{1}{2} \|\bfy - \bfS\bfH\bfx\|_2^2 + \tau \| \bfA \bfx \|_1
\label{eq_probform_WT_l1}
\end{equation}
where $\bfA=\bfW^H$.
By introducing the additional variable $\bfu = \bfW^H \bfx$, the problem \eqref{eq_probform_WT_l1} can be solved
using Algo. \ref{alg3_salsa_fsr}. The corresponding pseudocodes of the resulting fast SR algorithm with an $\ell_1$-norm regularization in the wavelet domain
are detailed in Algo. \ref{alg3_admm_l1_fast} in Appendix \ref{app:psd_code}.}

\subsection{Learning-based $\ell_2$-norm regularization}
\AB{The effectiveness of the learning-based regularization for image reconstruction has been proved in several studies. In particular, Yang \textit{et. al.} \cite{Yang2010TIP} solved the single image SR problem by jointly training two dictionaries for the LR and HR image patches and by applying sparse coding (SC). Interestingly, the HR image $\bfx_0$ obtained by sparse coding was projected onto the solution space satisfying \eqref{image_form}, leading to the following optimization problem
\begin{equation}
\hat{\bfx} = \argmin_{\bfx} \frac{1}{2} \|\bfy - \bfS\bfH\bfx\|_2^2 + \tau \|\bfx - \bfx_0\|_2^2.
\label{eq_yang}
\end{equation}
This optimization problem was solved using a gradient descent approach in  \cite{Yang2010TIP}. However, it can benefit from 
the analytical solution provided by Theorem \ref{the:Ubar} that can be implemented using Algo. \ref{Algo:FastSR}.}

\section{Experimental Results}
\AB{This section demonstrates the efficiency of the proposed fast SR strategy by testing it on various images with different regularization terms. The performance of the single image SR algorithms is evaluated in terms of reconstruction quality and computational load. Given the ability of our algorithm to solve the SR problem with less complexity than the existing methods, one may expect a gain in computational time and convergence properties.} All the experiments were performed using MATLAB 2013A on a computer with Windows 7, Intel(R) Core(TM) i7-4770 CPU @3.40GHz and 8 GB RAM\footnote{The MATLAB codes are available in the first author's homepage http://zhao.perso.enseeiht.fr/}. \AB{Color images were processed using the illuminate channel only, as in \cite{Yang2010TIP}. Precisely, the RGB images were transformed into YUV coordinates and the color channels (Cb,Cr) were up-sampled using bicubic interpolation.} In the illuminate channel, the HR image \nd{was blurred and down-sampled in each spatial direction with factors $d_r$ and $d_c$.} The resulting blurred and decimated images were then contaminated by AWGN \nd{of variance $\sigma_n^2$} with a blurred-signal-to-noise ratio defined by
\begin{equation}
\textrm{BSNR}=10 \log_{10} \left(\frac{\|\bfS\bfH\bfx - E(\bfS\bfH\bfx)\|_2^2}{N \sigma_n^2}\right)
\end{equation}
where $N$ is the total number of pixels of the observed image and $E(\cdot)$ is the arithmetic mean operator.

\nd{Unless explicitly specified, the blurring kernel is a $2$D-Gaussian filter of size $9\times 9$ with variance $\sigma_h^2=3$, the decimation factors are $d_r=d_c=4$ and the noise level is BSNR $=30$dB.}

The performances of the different SR algorithms are evaluated both visually and quantitatively in terms of the following metrics: root mean square error (RMSE), \AB{peak signal-to-noise ratio (PSNR)}, improved signal-to-noise ratio (ISNR) and mean structural similarity (MSSIM). The definitions of these metrics, widely used to evaluate image reconstruction methods, are given below

\begin{align}
\textrm{RMSE} &= \sqrt{\|\bfx-\hat{\bfx}\|^2}\\
\textrm{PSNR} &= 20\log_{10}\frac{\textrm{max}(\bfx,\hat{\bfx})}{\textrm{RMSE}} \\
\textrm{ISNR} &= 10\log_{10}\frac{\|\bfx-\bar{\bfy}\|^2}{\|\bfx-\hat{\bfx}\|^2}\\
\textrm{MSSIM} & = \frac{1}{M} \sum_{j=1}^{M} \textrm{SSIM}(\bfx_j,\hat{\bfx}_j)
\end{align}
where the vectors $\bfx,\bar{\bfy},\hat{\bfx}$ are the ground truth (reference image/HR image), the bicubic interpolated image and the restored SR image respectively and $\textrm{max}(\bfx,\hat{\bfx})$ defines the largest value of $\bfx$ and $\hat{\bfx}$.
Note that MSSIM is implemented blockwise, with $M$ the number of local windows, $\bfx_j$ and $\hat{\bfx}_j$ are local regions extracted from $\bfx$ and $\hat{\bfx}$ and SSIM is the structural similarity measure of each window (defined in \cite{Wang2004}). Note that it is nonsensical to compute the ISNR for bicubic interpolation (always be 0) due to its definition.

\subsection{Fast SR using $\ell_2$-regularizations}

\subsubsection{$\ell_2-\ell_2$ model in the image domain}

\paragraph{\revAQ{Gaussian blurring kernel}}
We first explore the single image SR problem with the ``pepper" image and standard Tikhonov/Gaussian regularization corresponding to
the optimization problem formulated in \eqref{l2_problem_image}. The size of the ground truth HR image shown in Fig. \ref{fig_l2_truth} is
$512\times 512$. Fig. \ref{fig_l2_bicubic}--\ref{fig_l2estal_xtrue} show the restored images with bicubic interpolation, the proposed analytical solution given in Algo. \ref{Algo:FastSR} and the splitting algorithm ADMM of \cite{MNg2010SR_TV} adapted to a Gaussian prior. The prior mean image $\bar{\bfx}$ (approximated HR image) is the up-sampled version of the LR image by bicubic interpolation (Case 1) with the results in Fig. \ref{fig_l2est_xup} and \ref{fig_l2estal_xup}, whereas  $\bar{\bfx}$ is the ground truth (Case 2)  with the results in Fig. \ref{fig_l2est_xtrue} and \ref{fig_l2estal_xtrue}. The regularization parameter was $\tau=1$ in Case 1 and $\tau=0.1$ in Case 2. The numerical results corresponding to this experiment are summarized in Table \ref{tab_l2}. The visual impression and the numerical results show that the reconstructed HR images obtained with our method are similar to those obtained with ADMM. However, the  proposed FSR method performs much faster than ADMM. More precisely, the computational time with our method is divided by a factor of $60$ for Case 1 and by a factor of $80$ for Case 2. Note also that the restored images obtained with Case 2 ($\bar{\bfx}$ set to the ground truth) are visually much better than the ones obtained with Case 1 ($\bar{\bfx}$ equal to the interpolated LR image), as expected.

\begin{figure*}[t!]
\begin{center}
\subfigure[Observation$^\dag$]{\includegraphics[width=0.18\linewidth]{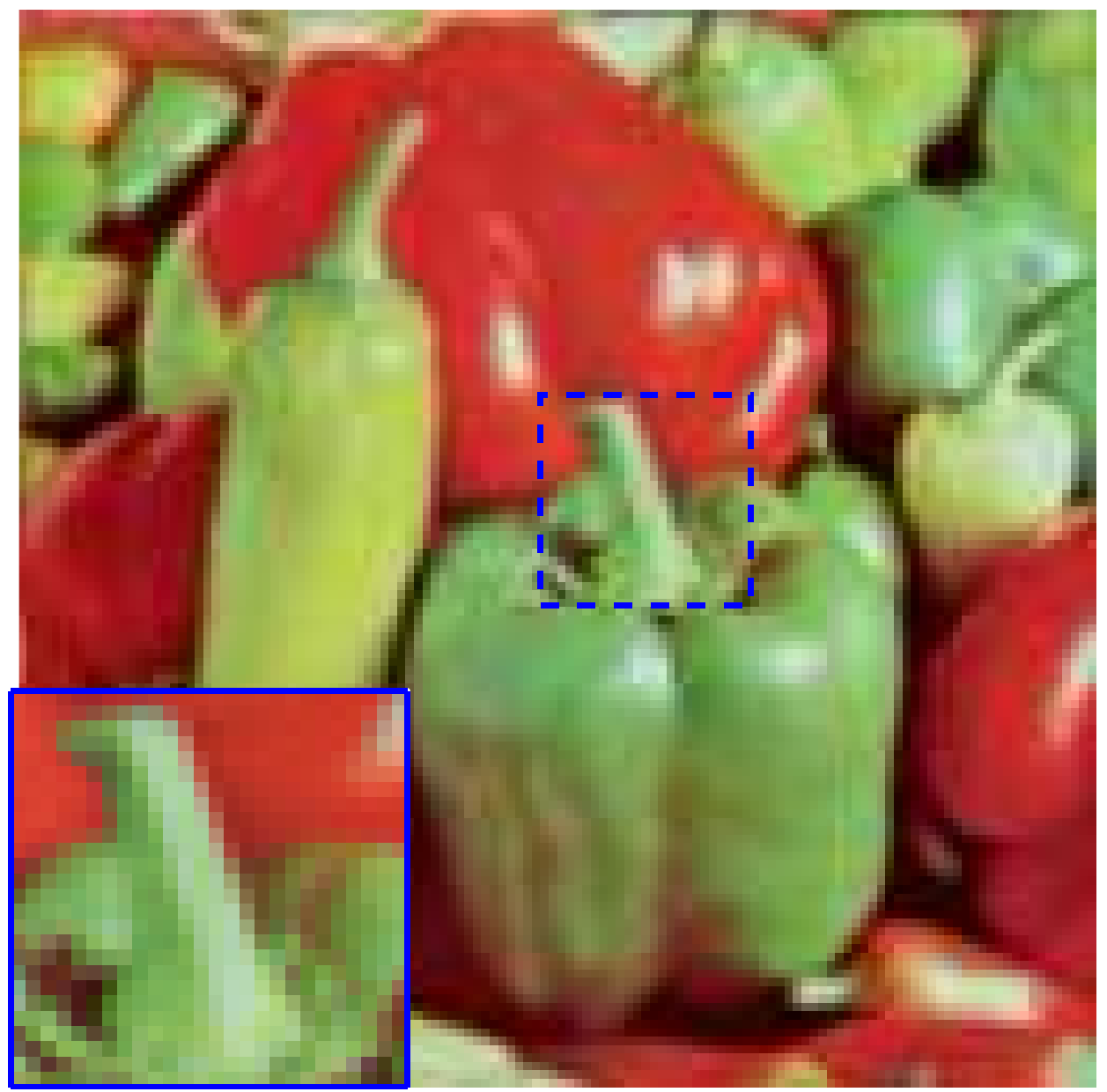} \label{fig_l2_obs}}
\subfigure[Ground truth]{\includegraphics[width=0.18\linewidth]{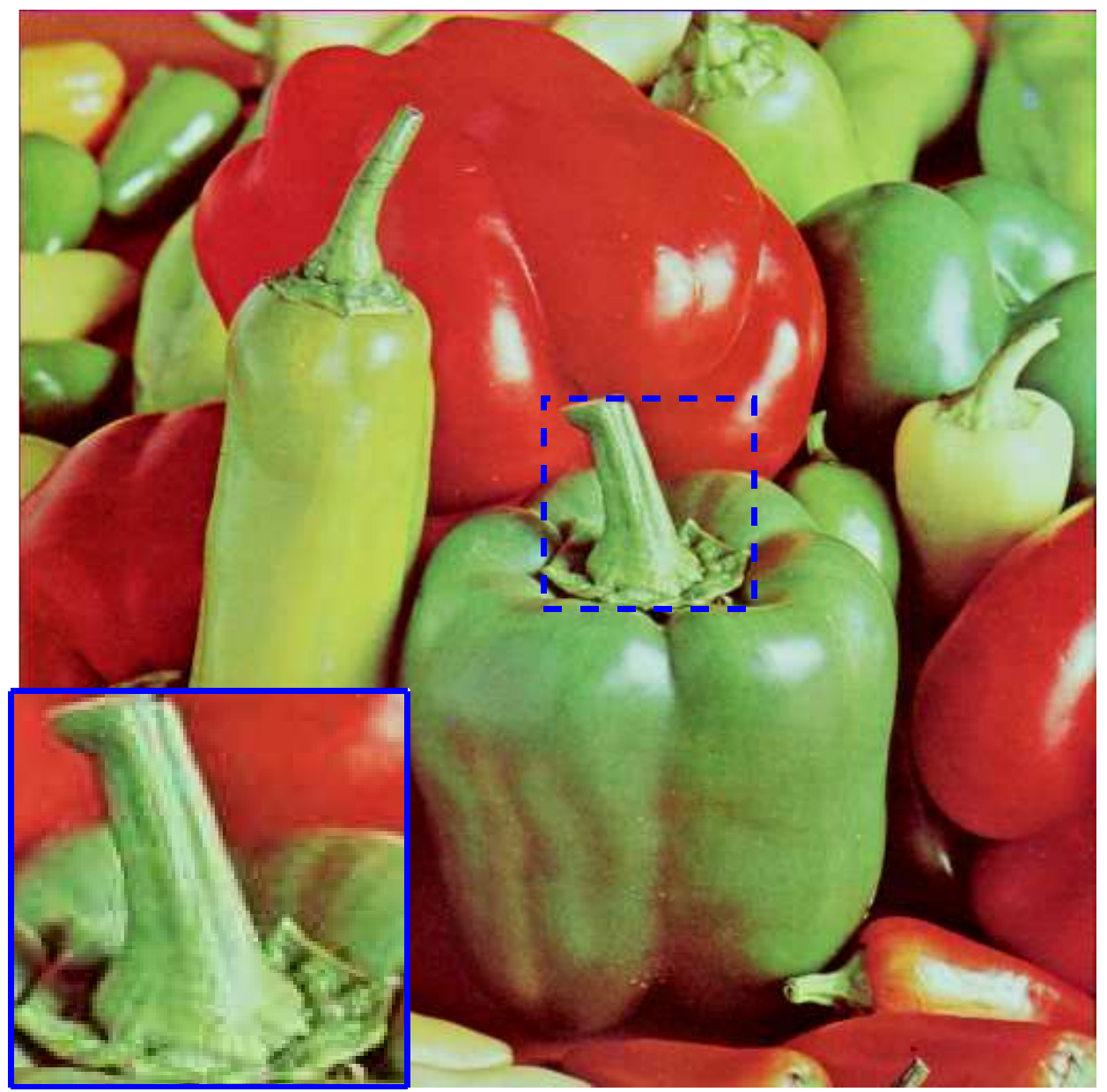} \label{fig_l2_truth}}
\subfigure[Bicubic interpolation]{\includegraphics[width=0.18\linewidth]{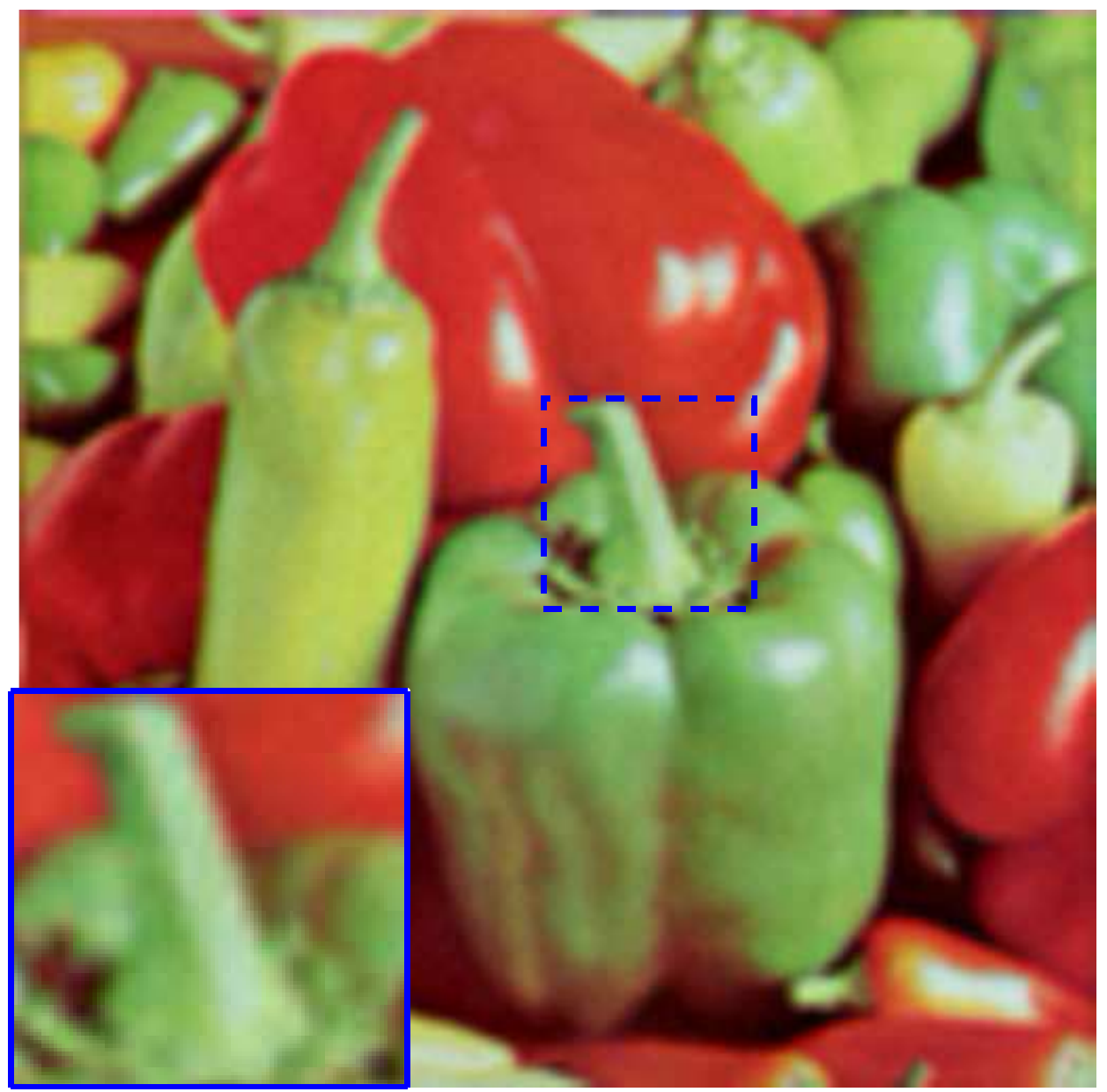} \label{fig_l2_bicubic}}\\
\subfigure[Case 1: ADMM]{\includegraphics[width=0.18\linewidth]{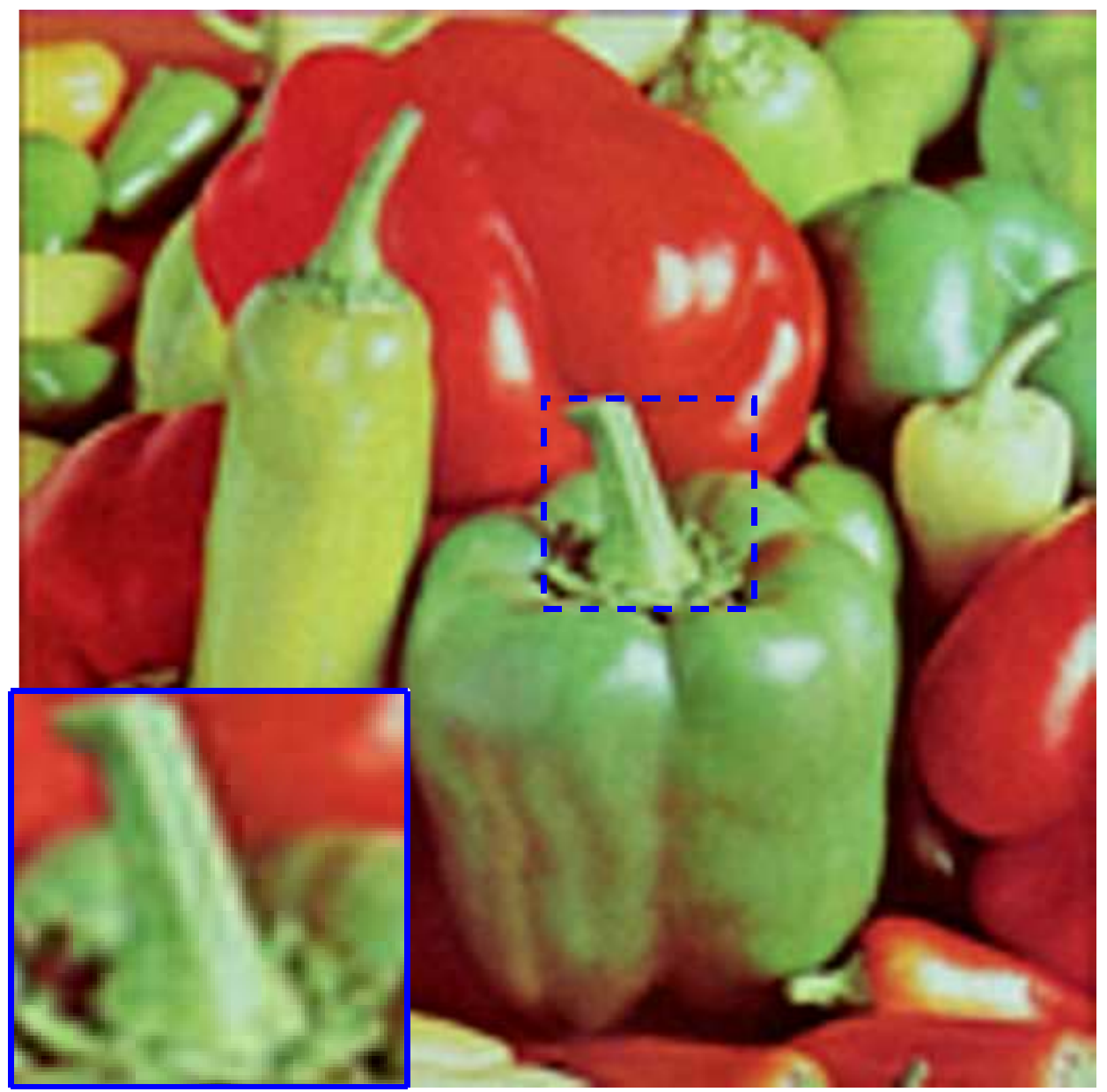} \label{fig_l2estal_xup}}
\subfigure[Case 1: Algo. \ref{Algo:FastSR}]{\includegraphics[width=0.18\linewidth]{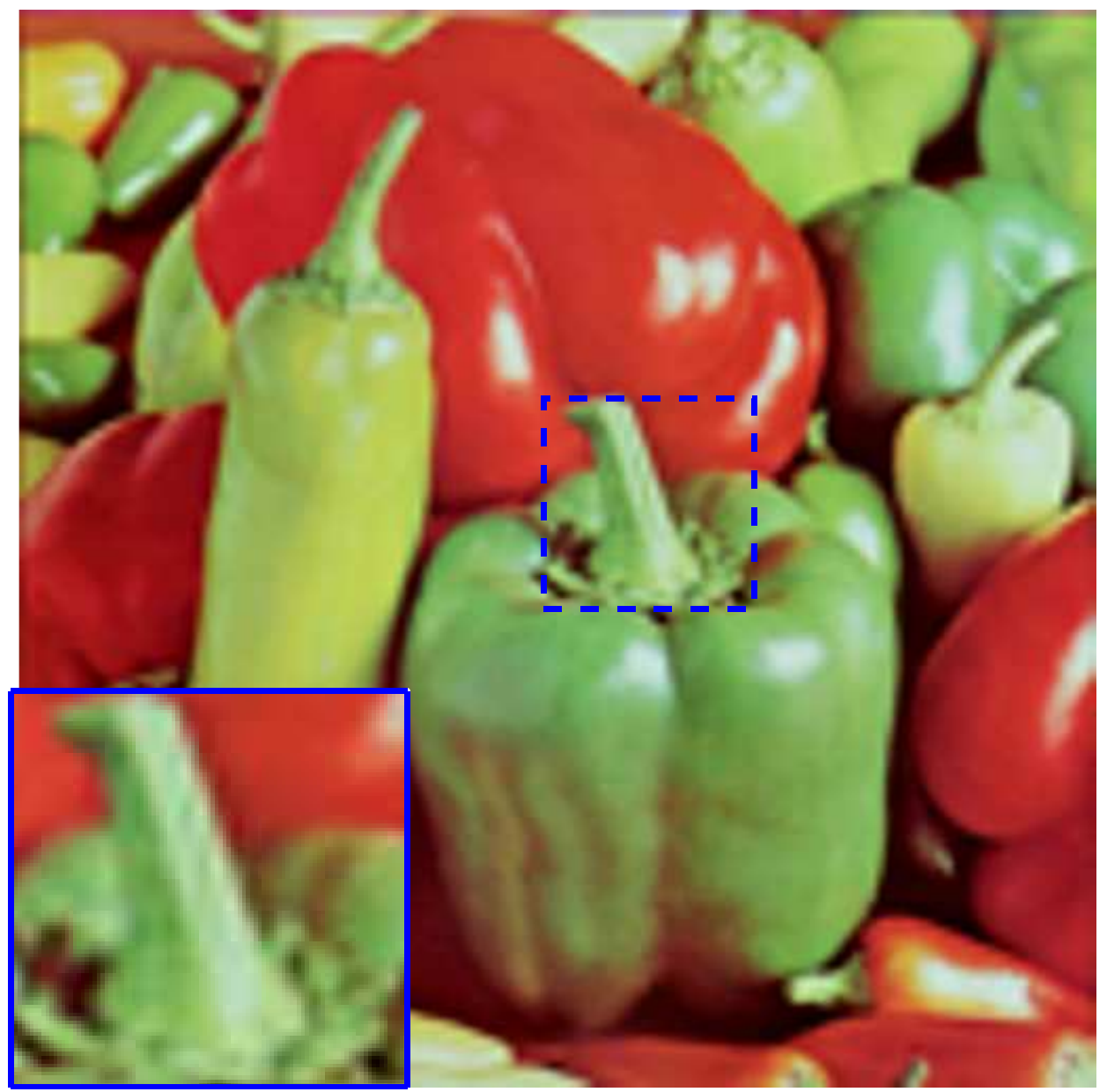} \label{fig_l2est_xup}}
\subfigure[Case 2: ADMM]{\includegraphics[width=0.18\linewidth]{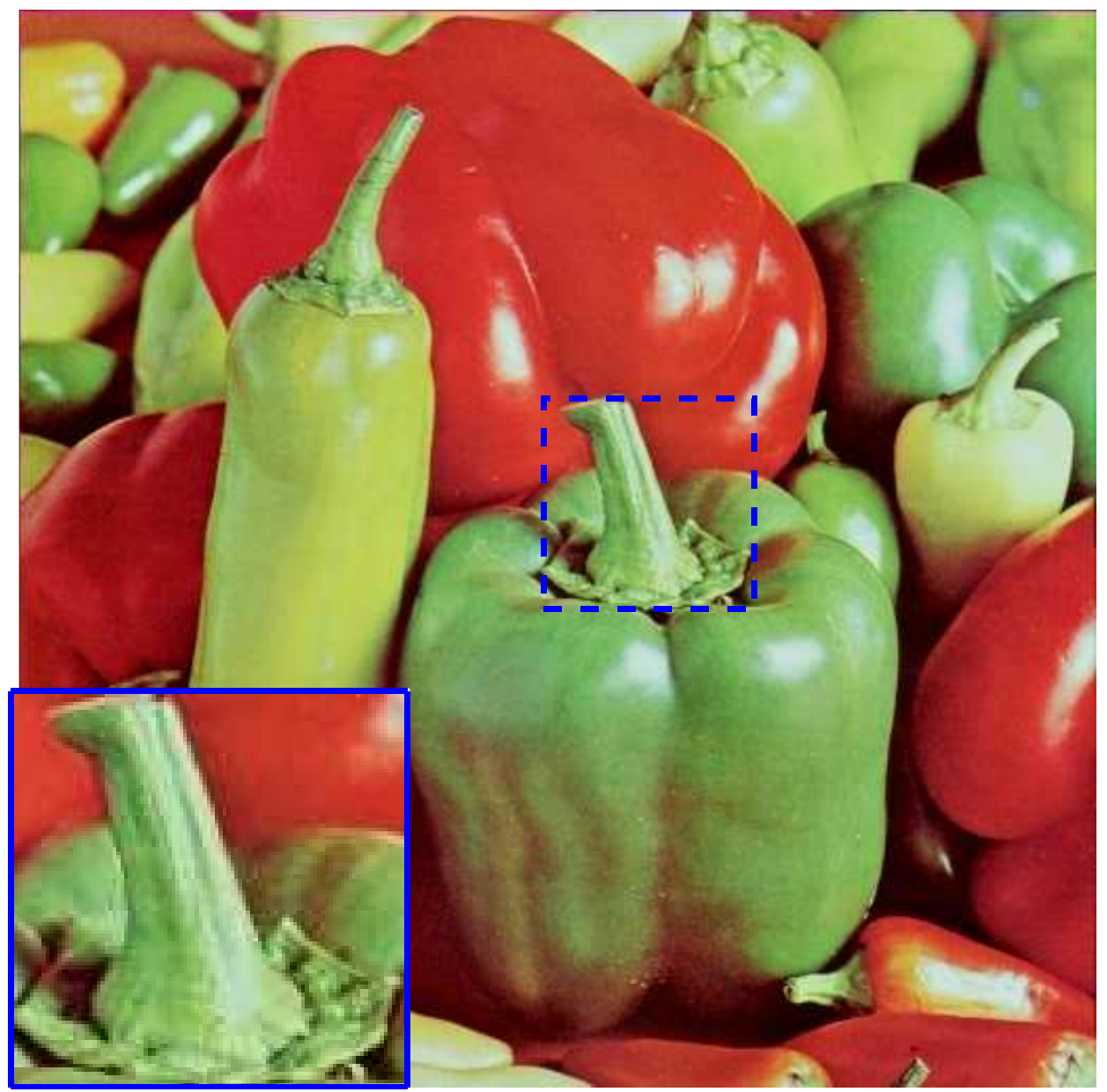} \label{fig_l2estal_xtrue}}
\subfigure[Case 2: Algo. \ref{Algo:FastSR}]{\includegraphics[width=0.18\linewidth]{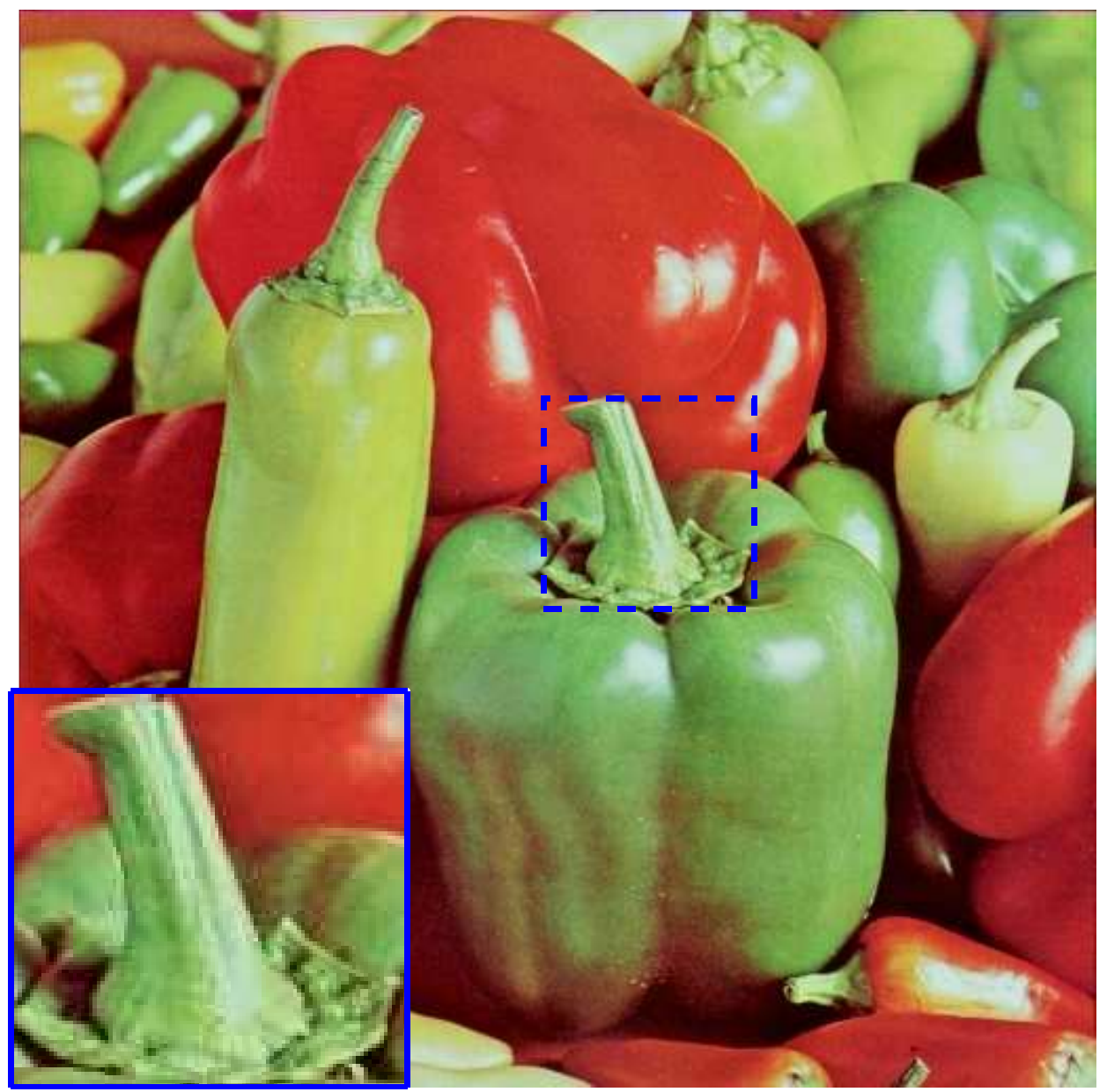} \label{fig_l2est_xtrue}}
\caption[The LOF caption]{SR of the pepper image when considering an $\ell_2-\ell_2$-model in the image domain: visual results. The prior image mean $\bar{\bfx}$ is defined as the bicubic interpolated LR image in Case 1 and as the ground truth HR image in Case 2.}
\label{ex1_pepper}
\end{center}
{\footnotesize \rule{0in}{1.2em}$^{\dag}$\scriptsize Note that the LR images have been scaled for better visualization in this figure (i.e., the actual LR images contain $d$ times fewer pixels than the corresponding HR images).}
\end{figure*}

\begin{table}[h!]
\setlength{\tabcolsep}{3pt}
\begin{center}
\caption{SR of the pepper image when considering an $\ell_2-\ell_2$-model in the image domain: quantitative results.}
\label{tab_l2}
\begin{tabular}{|c|c|c|c|c|}
\hline
 Method  & PSNR (dB) &ISNR (dB) & MSSIM  & Time (s.) \\
\hline
\hline
Bicubic  & 25.37 & -      & 0.59 & 0.002 \\
\hline
\multicolumn{5}{|c|}{Case 1} \\ \hline
ADMM     & 29.26 & 4.01    & 0.67 & 1.92 \\
Algo. \ref{Algo:FastSR} & 29.27 & 4.01    & 0.67 & \textbf{0.02} \\
\hline
\multicolumn{5}{|c|}{Case 2} \\ \hline
ADMM     & 53.84 & 29.27   & 1    & 0.5 \\
Algo. \ref{Algo:FastSR} & 53.74 & 29.55   & 1    & \textbf{0.02} \\
\hline
\end{tabular}
\end{center}
\end{table}

\nd{The performance of the proposed method has been also evaluated with various experimental parameters, namely, the BSNR level, the size of the blurring kernel and the decimation factors. The corresponding RMSEs are depicted in Fig. \ref{ex1_pepper_curve} as functions of the regularization parameter $\tau$ for the two considered scenarios (Cases 1 and 2). Note that the same performance is obtained by the ADMM-based SR technique since it solves the same optimization problem.}

\begin{figure*}
\settoheight{\tempdima}{\includegraphics[width=.25\linewidth]{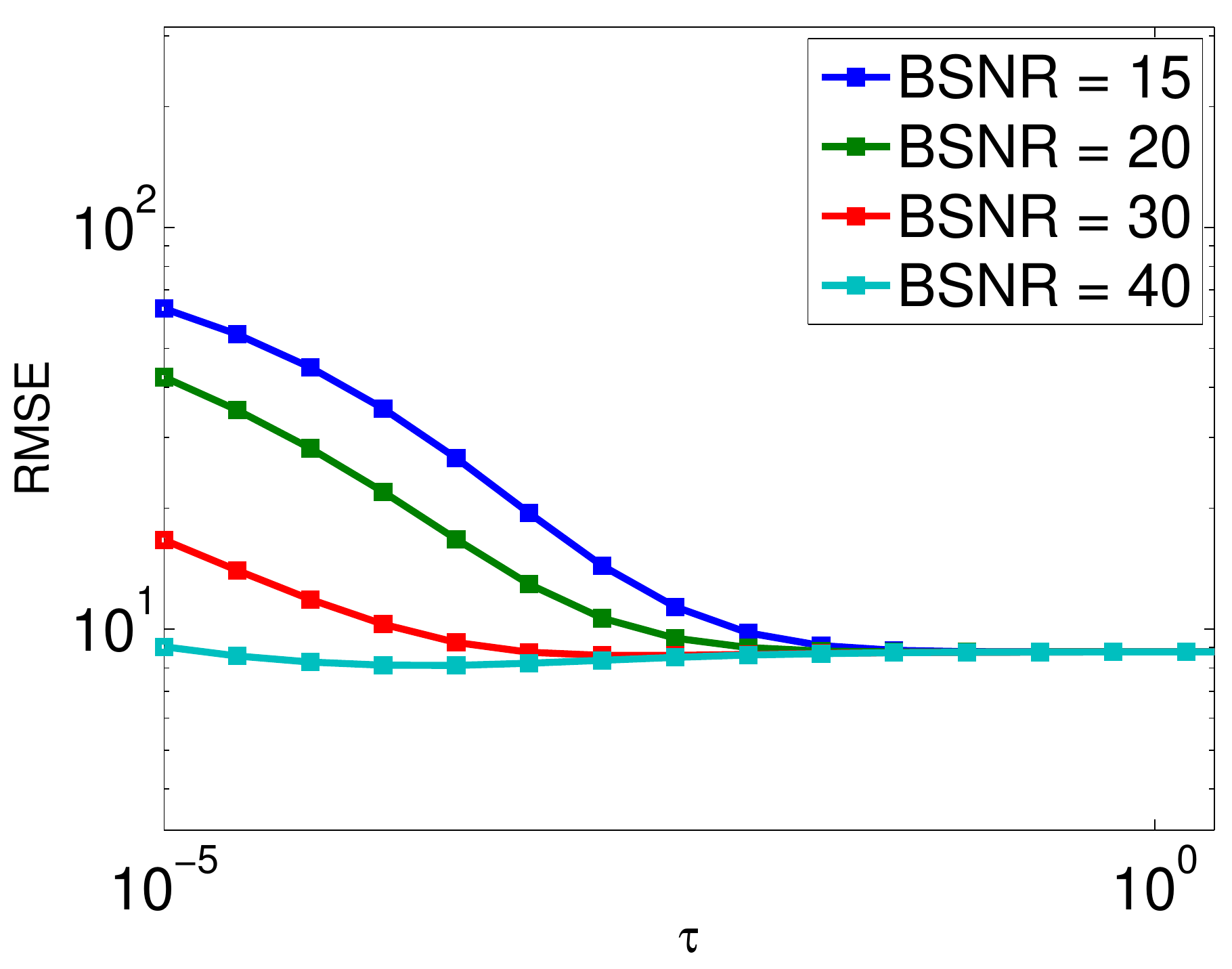}}%
\centering\begin{tabular}{@{}c@{}c@{}c@{}c@{}}
 \rowname{Case 1 }&
\includegraphics[width=0.25\linewidth]{figures/ex1_gaussian/case1_rmse_bsnr}
\label{case1_rmse_bsnr}&
\includegraphics[width=0.25\linewidth]{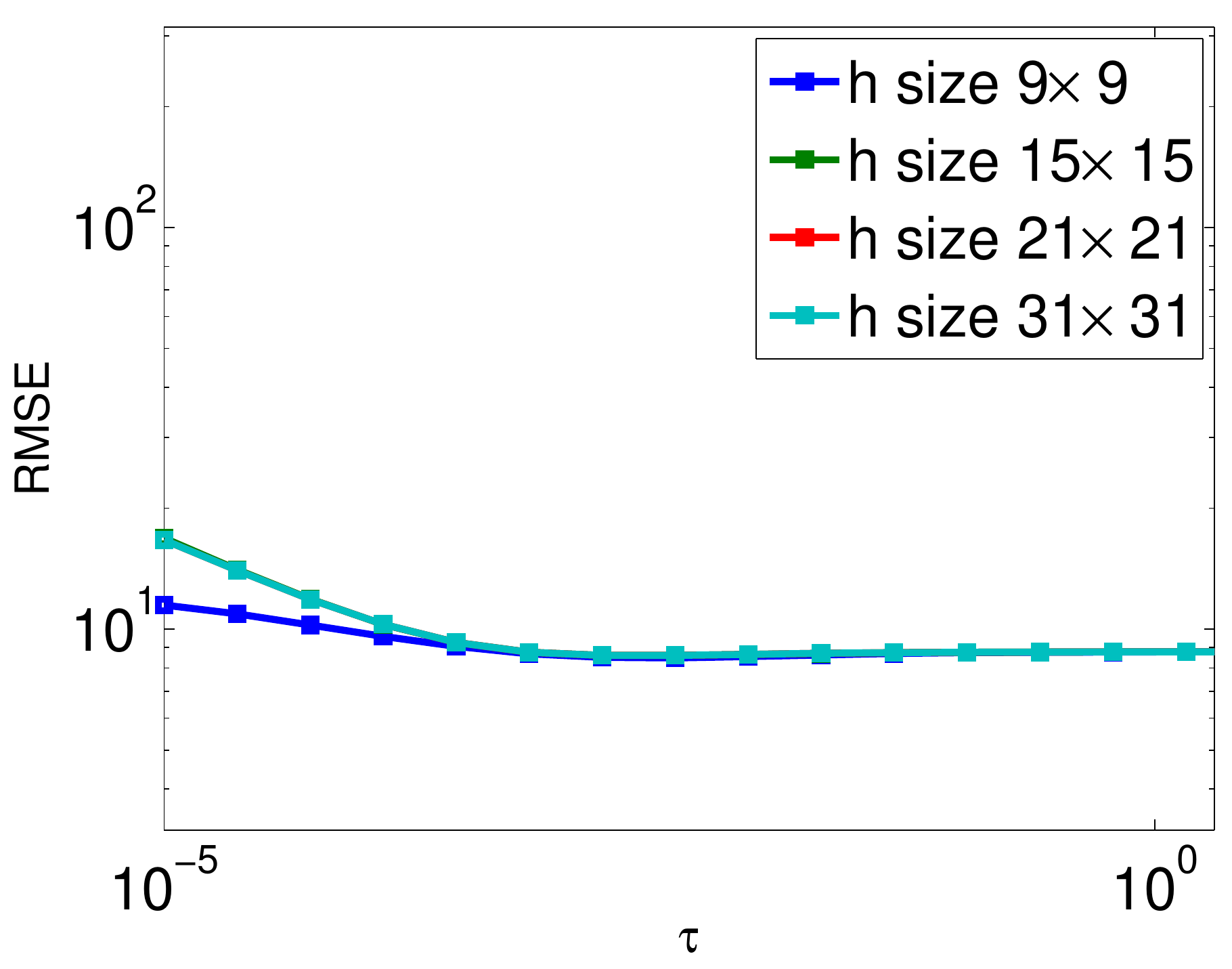} \label{case1_rmse_h}&
\includegraphics[width=0.25\linewidth]{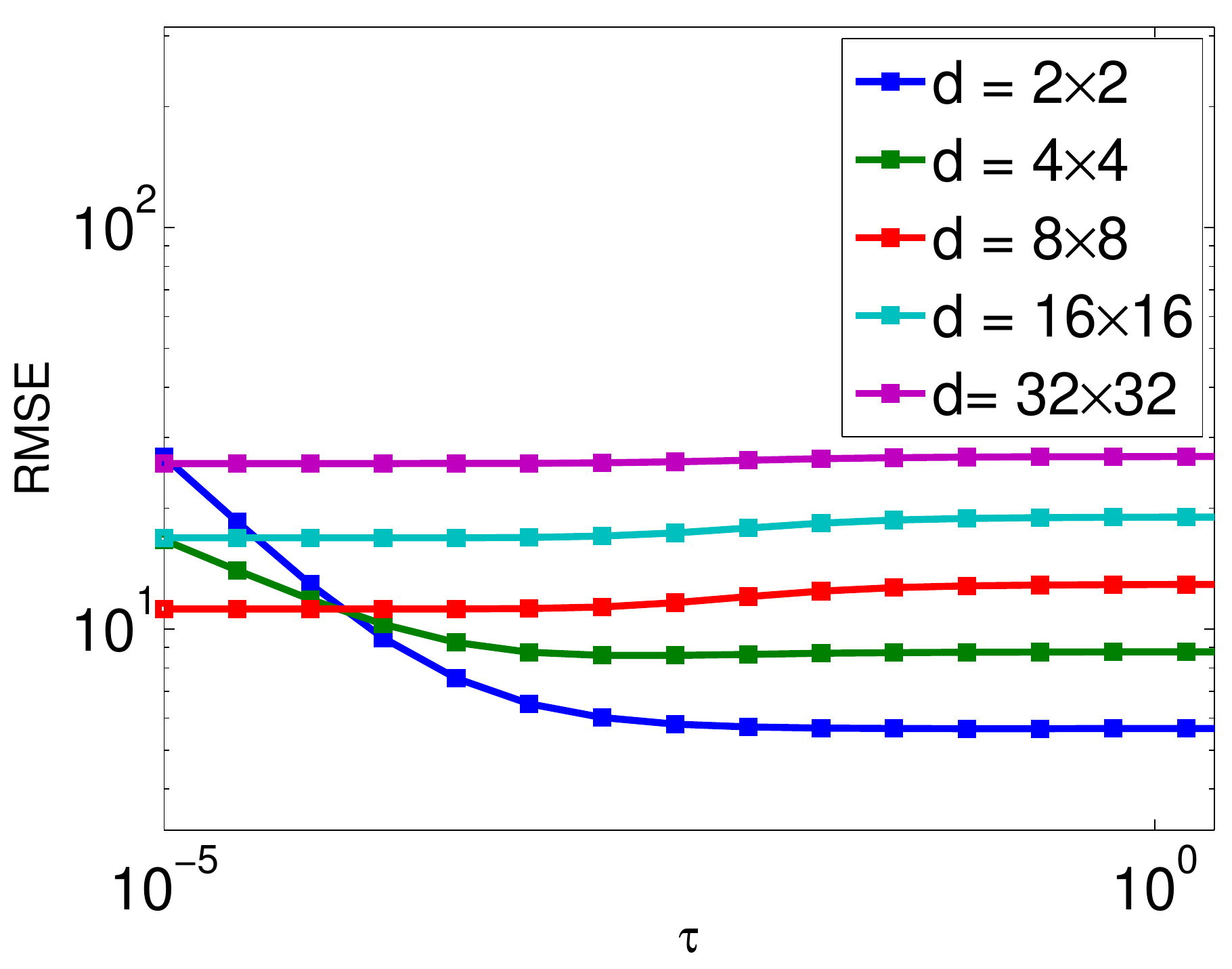} \label{case1_rmse_d}\\
\rowname{Case 2}&
\includegraphics[width=0.25\linewidth]{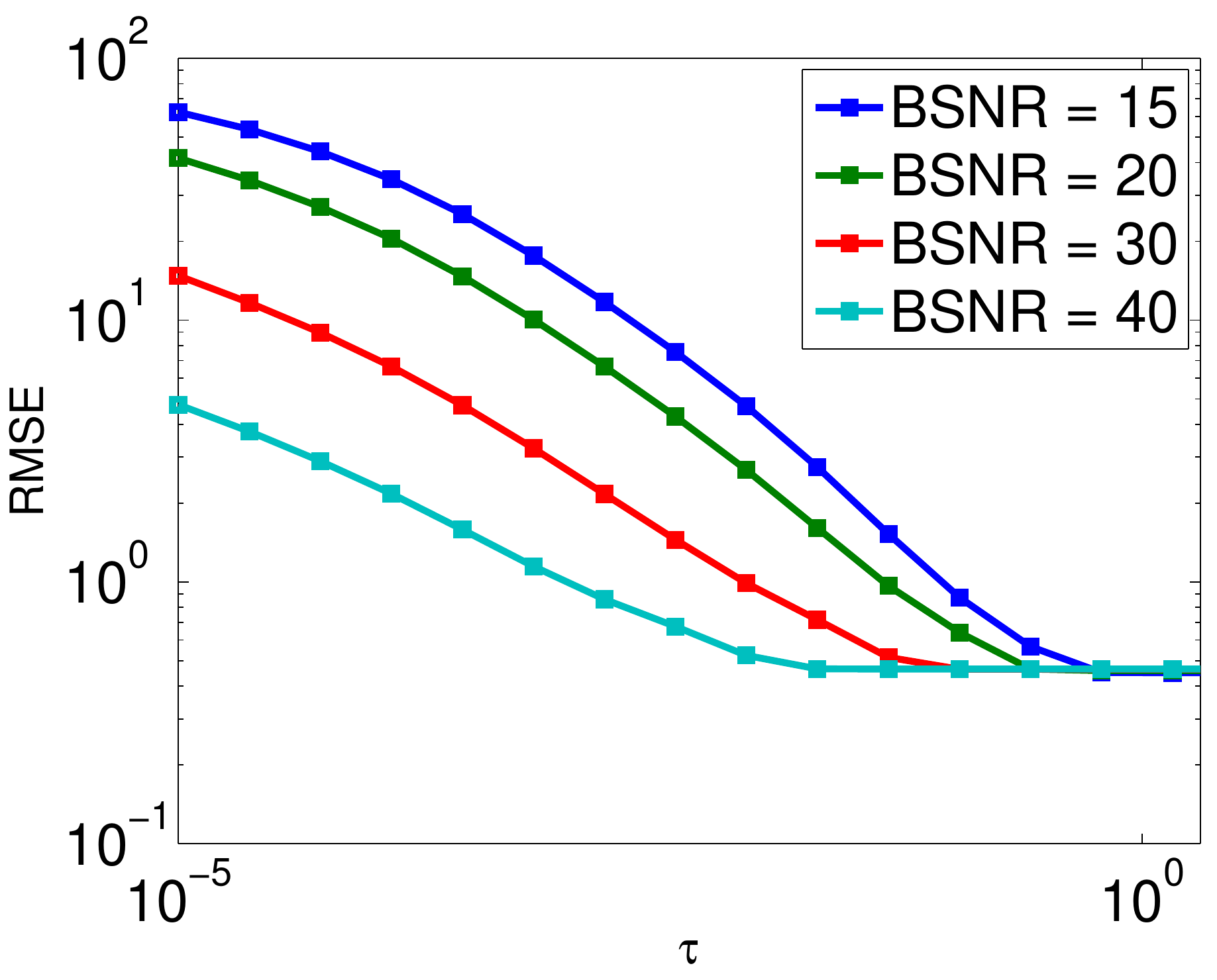} \label{case2_rmse_bsnr}&
\includegraphics[width=0.25\linewidth]{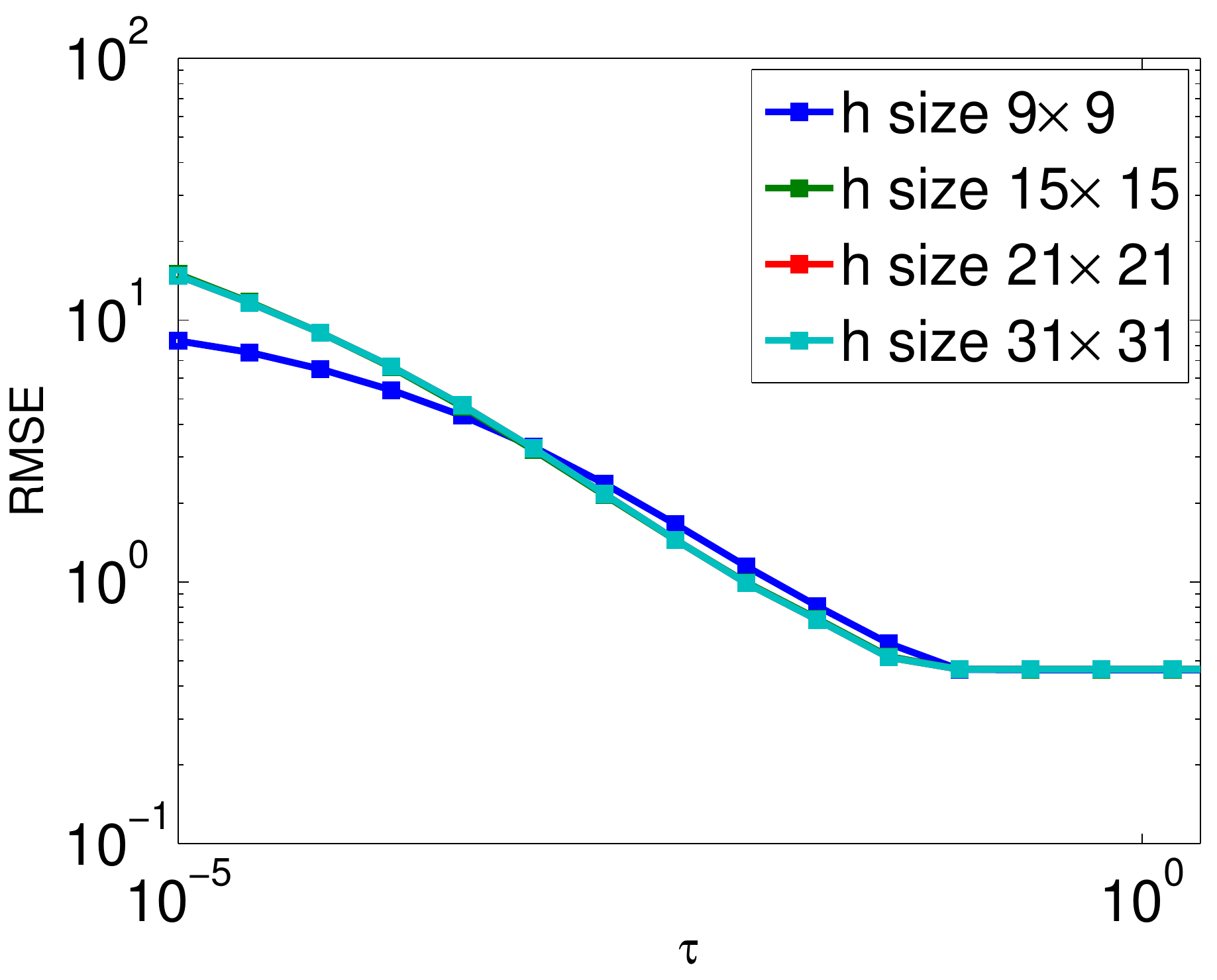} \label{case2_rmse_h}&
\includegraphics[width=0.25\linewidth]{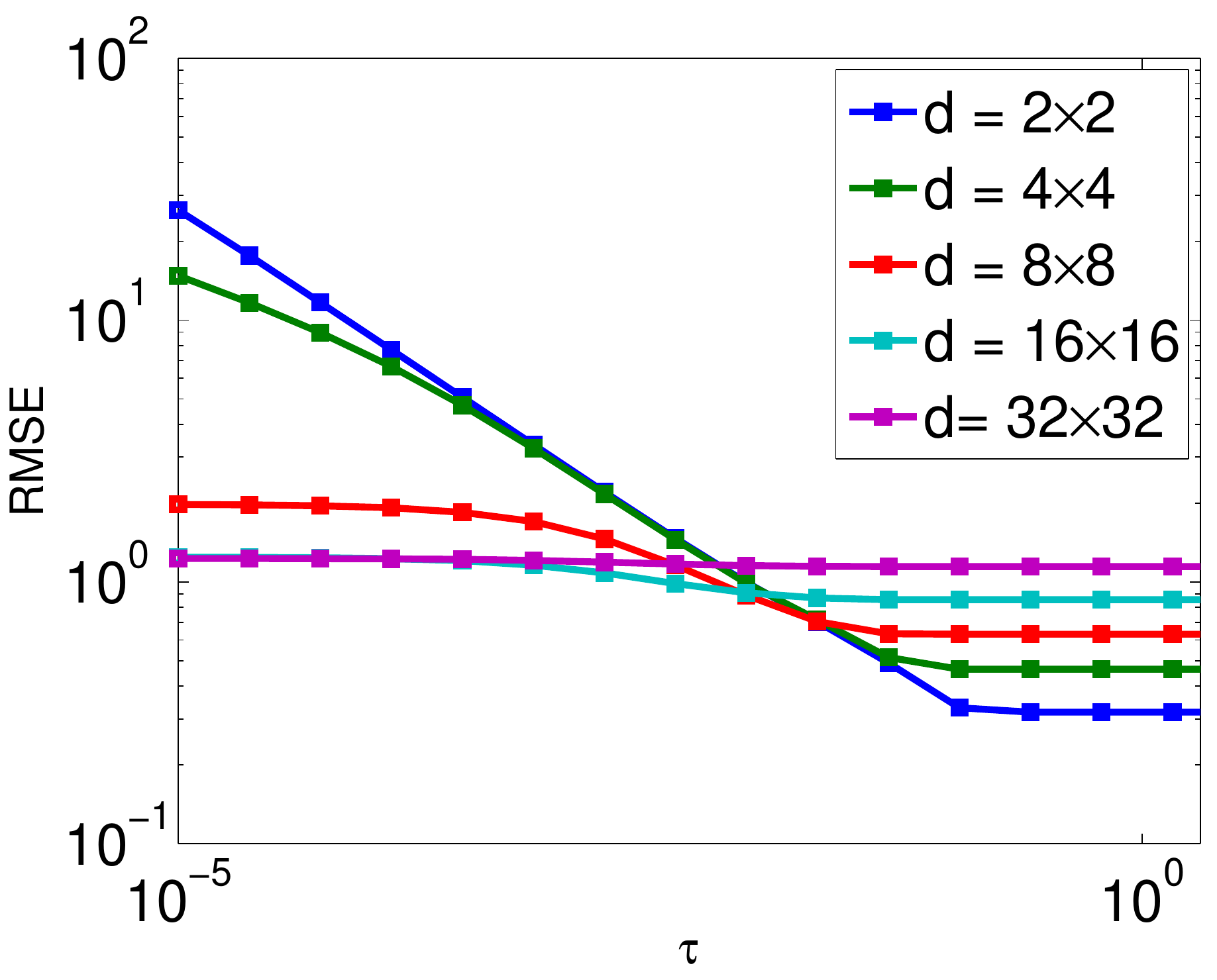} \label{case2_rmse_d}
\end{tabular}
\caption{SR of the ``pepper" image when considering the $\ell_2-\ell_2$ model in the image domain: RMSE as functions of the regularization parameter $\tau$ for various noise levels (1st column), blurring kernel sizes (2nd column) and decimation factors (3rd column). The results in the 1st column were obtained for $d_r=d_c=4$ and $9\times9$ kernel size; in the 2nd column, $d_r=d_c=4$ and BSNR$=30$ dB; in the 3rd column, the kernel size was $9\times9$ and BSNR$=30$ dB.}
\label{ex1_pepper_curve}
\end{figure*}

%

\paragraph{\revAQ{Motion blurring kernel}}
\nd{This paragraph considers a dataset composed of images that have been captured by a camera placed on a tripod, whose Z-axis rotation handle has been locked and X- and Y-axis rotation handles have been loosen \cite{Levin2009}. The corresponding dataset is available online\footnote{Available online at \url{http://www.wisdom.weizmann.ac.il/~levina/papers/LevinEtalCVPR09Data.zip}}. The observed LR image, motion kernel and corresponding SR results are shown in Fig. \ref{ex1_motion}. The size of the motion kernel is $19\times 19$. As in the previous paragraph, the prior mean image $\bar{\bfx}$ is the bicubic interpolation of the LR image in Case 1, while $\bar{\bfx}$ is the ground truth in Case 2. The regularization parameter is set to $\tau=0.01$ and $\tau=0.1$ in Cases 1 and 2, respectively. Quantitative results are reported in Table \ref{tab_motion_blur} and show that the proposed method provides competitive results w.r.t. the other methods, while being more computational efficient.}

\begin{figure*}[t!]
\begin{center}
\subfigure[Observation]{\includegraphics[width=0.235\linewidth]{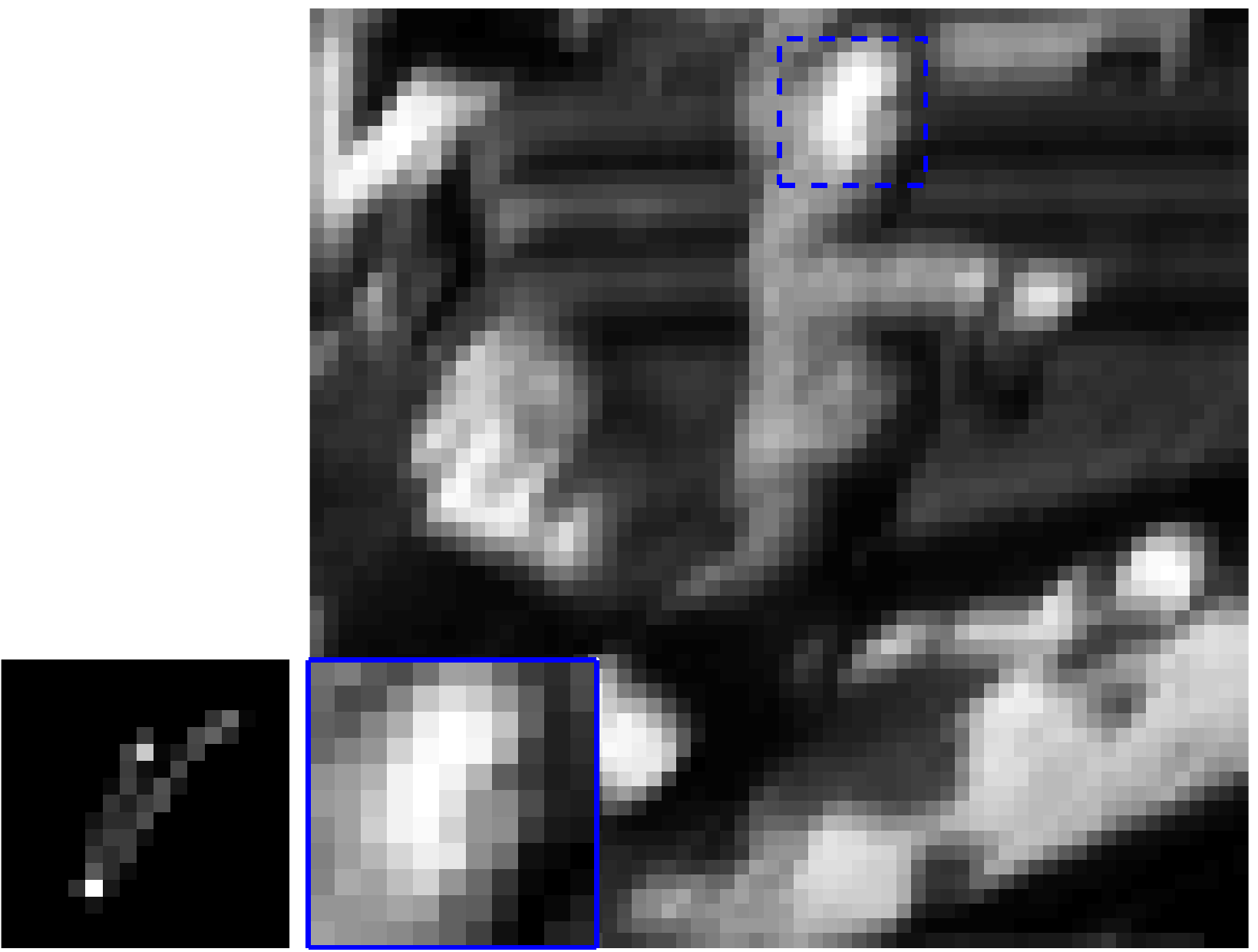}}
\subfigure[Ground truth]{\includegraphics[width=0.18\linewidth]{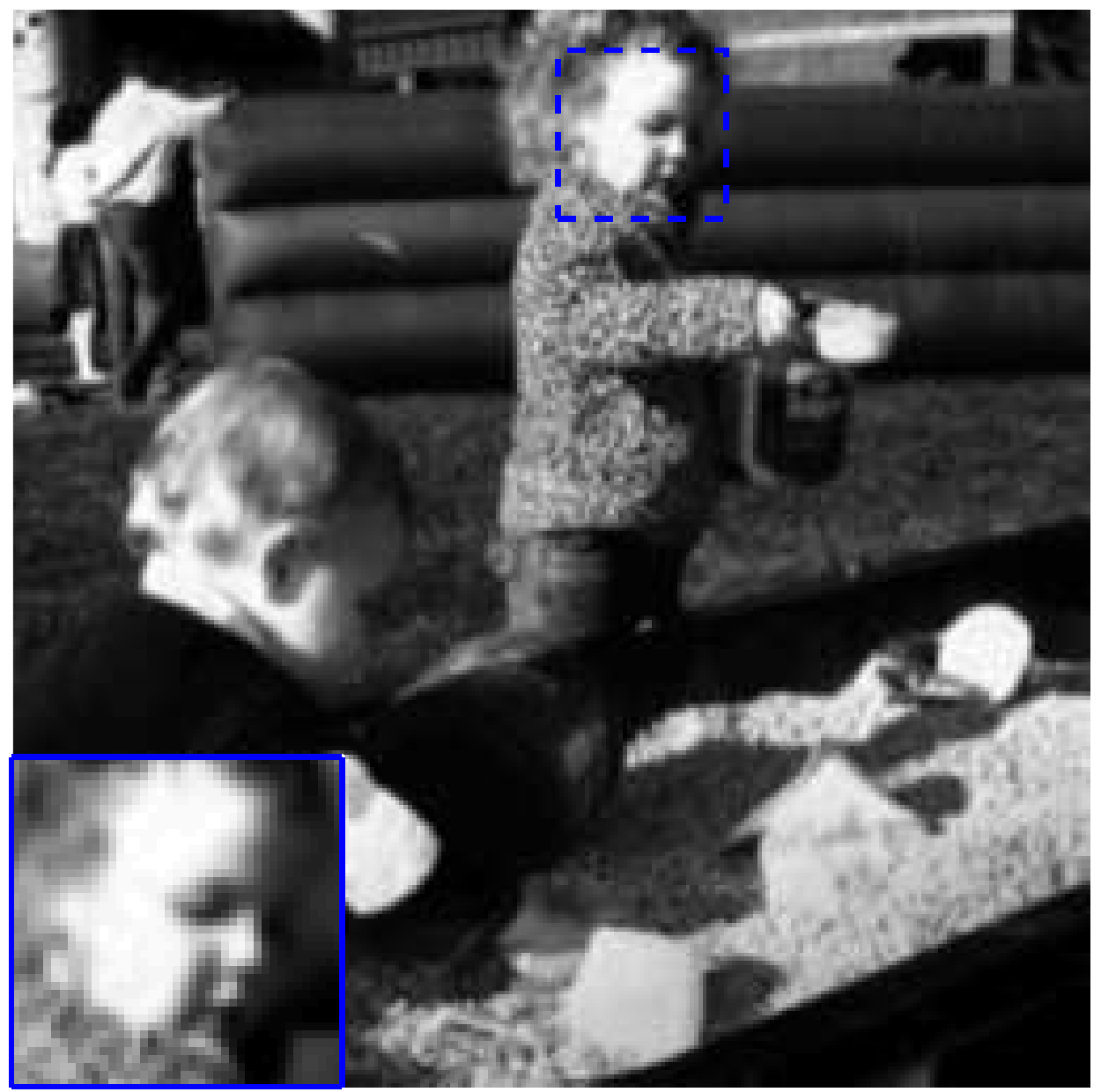}}
\subfigure[Bicubic interpolation]{\includegraphics[width=0.18\linewidth]{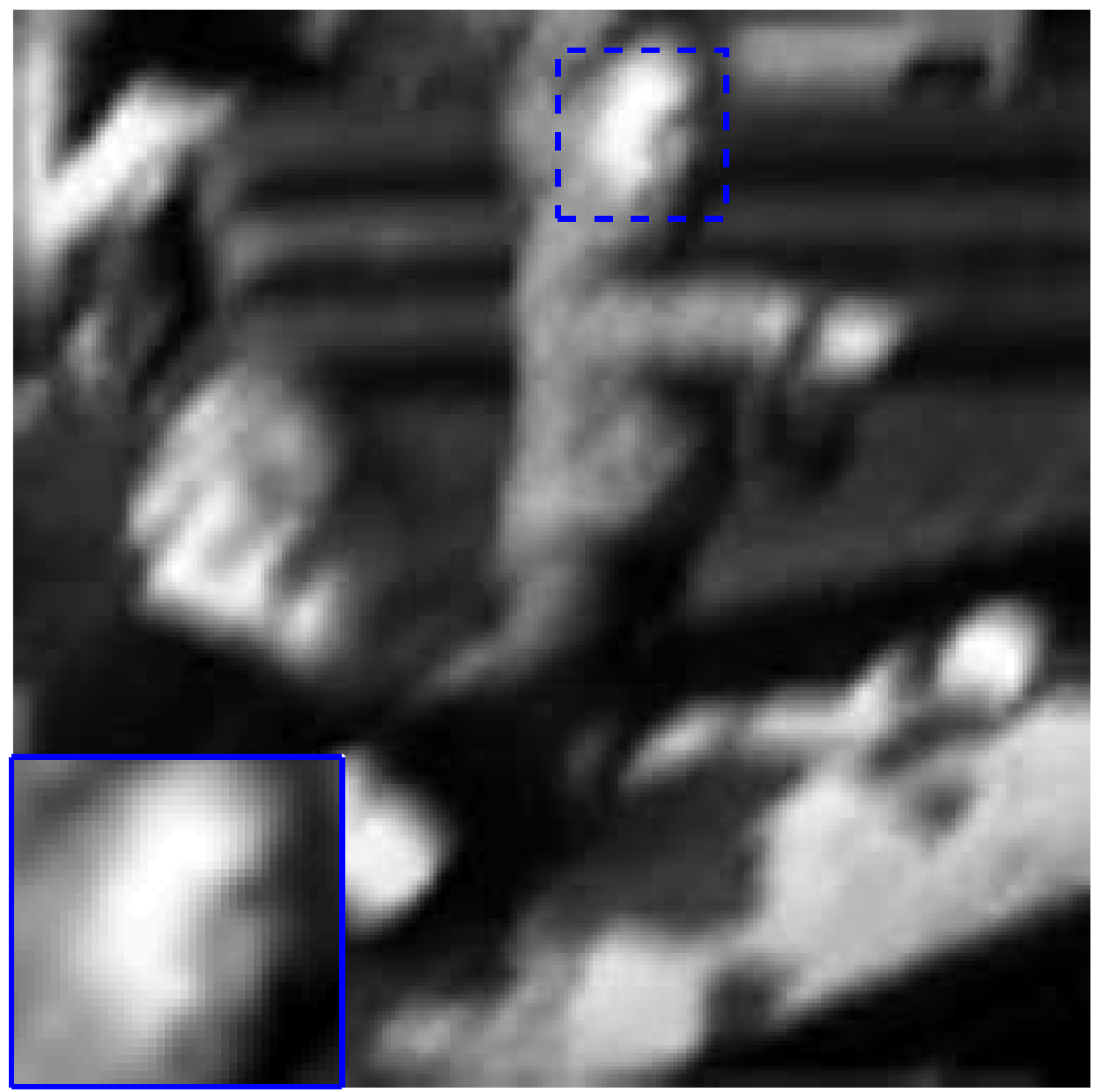}}\\
\hspace{0.7cm}
\subfigure[Case 1: ADMM]{\includegraphics[width=0.18\linewidth]{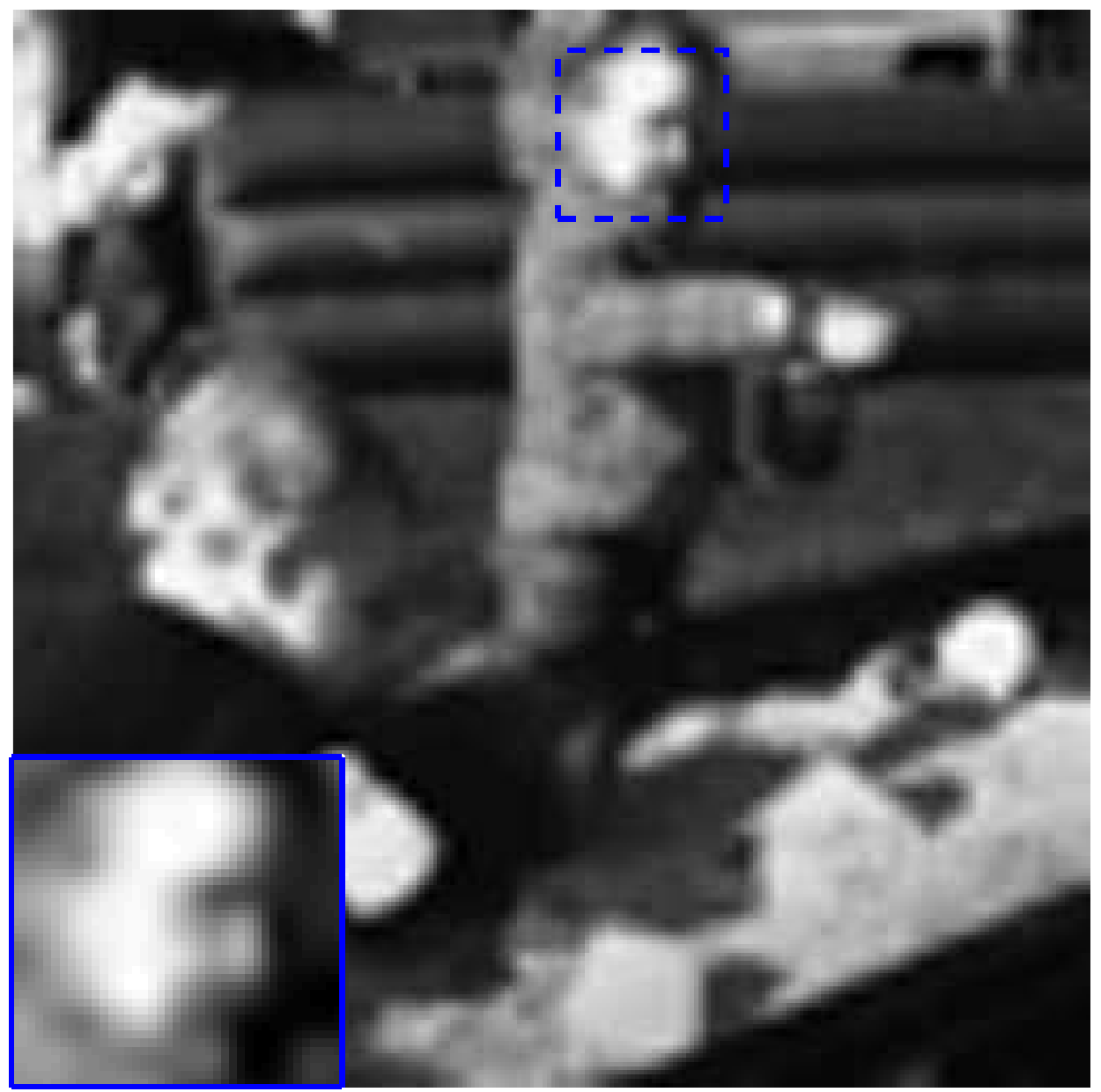}}
\subfigure[Case 1: Algo. \ref{Algo:FastSR}]{\includegraphics[width=0.18\linewidth]{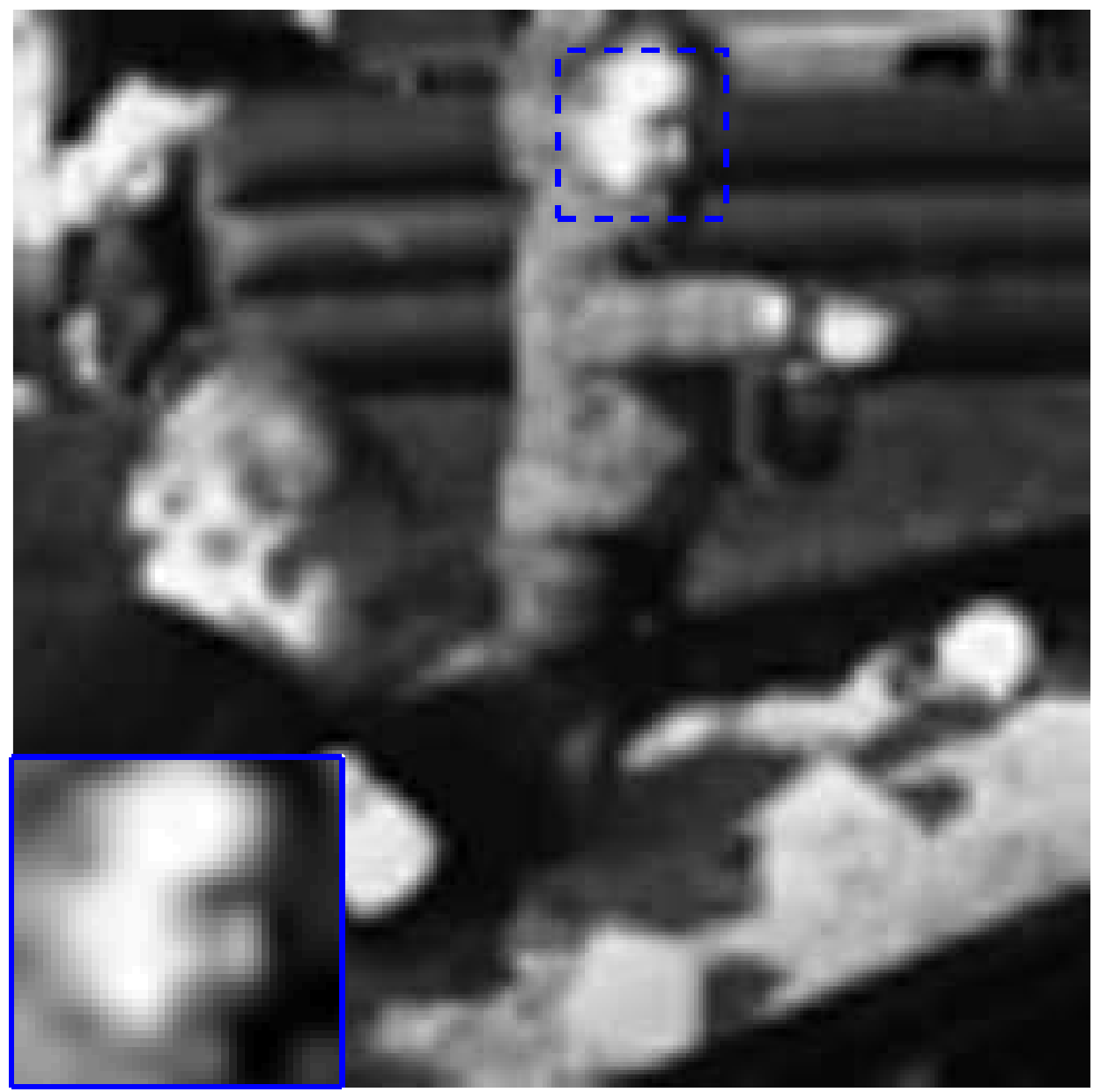}}
\subfigure[Case 2: ADMM]{\includegraphics[width=0.18\linewidth]{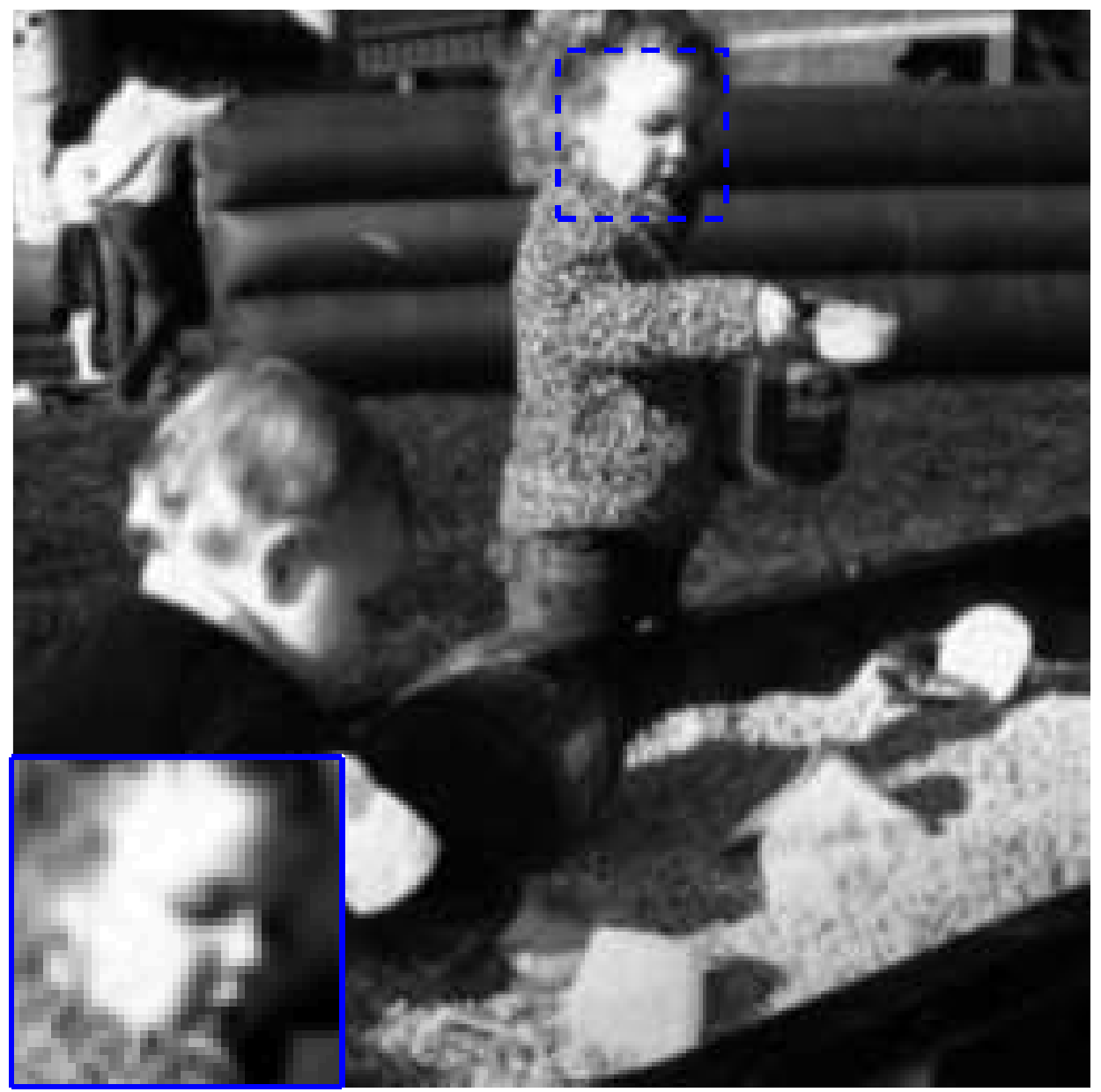} }
\subfigure[Case 2: Algo. \ref{Algo:FastSR}]{\includegraphics[width=0.18\linewidth]{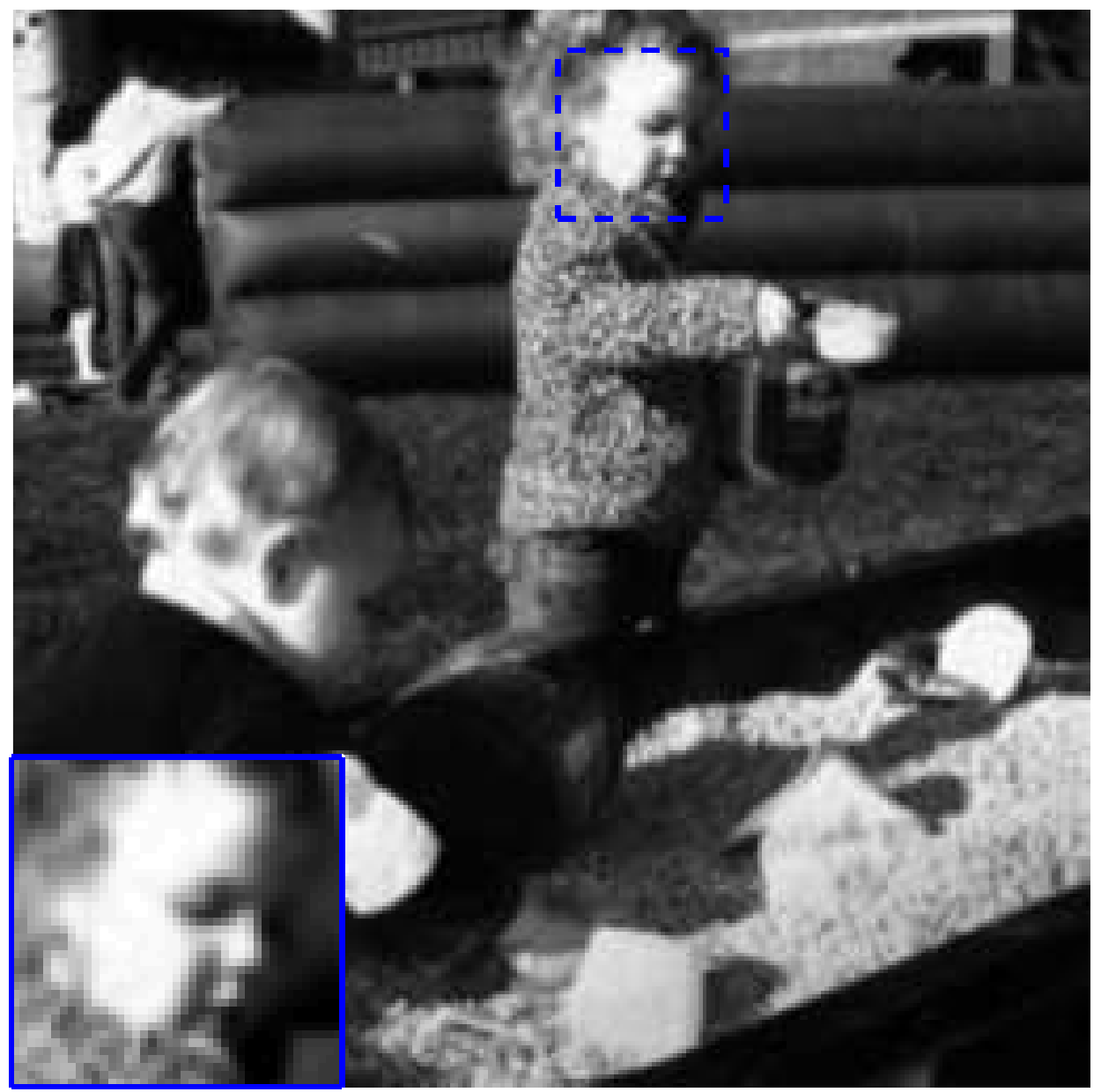} }
\caption[The LOF caption]{SR of the motion blurred image when considering an $\ell_2-\ell_2$-model in the image domain: visual results. The prior image mean $\bar{\bfx}$ is defined as the bicubic interpolated LR image in Case 1 and as the ground truth HR image in Case 2.}
\label{ex1_motion}
\end{center}
\end{figure*}


\begin{table}[h!]
\setlength{\tabcolsep}{3pt}
\begin{center}
\caption{SR of the motion blurred image when considering an $\ell_2-\ell_2$-model in the image domain: quantitative results.}
\label{tab_motion_blur}
\begin{tabular}{|c|c|c|c|c|}
\hline
 Method  & PSNR (dB) &ISNR (dB) & MSSIM  & Time (s.) \\
\hline
\hline
Bicubic  & 21.15 & -      & 0.91 & 0.002 \\
\hline
\multicolumn{5}{|c|}{Case 1} \\ \hline
ADMM      & 27.11 & 5.96 & 0.96& 0.11\\
Algo. \ref{Algo:FastSR} & 27.11  & 5.96    & 0.96 & \textbf{0.01} \\
\hline
\multicolumn{5}{|c|}{Case 2} \\ \hline
ADMM      & 53.23 & 32.08   & 1    & 0.42 \\
Algo. \ref{Algo:FastSR} & 53.23 & 32.08      & 1    & \textbf{0.01} \\
\hline
\end{tabular}
\end{center}
\end{table}


\subsubsection{$\ell_2-\ell_2$ model in the gradient domain}
\nd{This section compares the performance of the proposed fast SR strategy with the gradient profile regularization proposed in \cite{SunJ_CVPR_2008}. As shown in Section \ref{subsec:l2_general}, Theorem 1 allows the analytical SR solution to be computed. The ``face" image (of size $276\times 276$) shown in Fig. \ref{fig_gp_HR} was used for these tests. In this experiment, $\bar{\nabla \bfx}$ is calculated using the reference HR image and the regularization parameters have been set to $\tau=10^{-3}$ and $\sigma=10^{-8}$. The proposed method is compared with the ADMM as well as the CG method instead of the gradient descent (GD) method initially proposed in \cite{SunJ_TIP_2011} (since CG has shown to be much more efficient than GD in this experiment). The restored images using bicubic interpolation, ADMM, the CG method and the proposed Algo. \ref{Algo_FSR_Gl2} are shown in Fig. \ref{fig_gp_bicubic}-\ref{fig_gp_gd}. The corresponding numerical results are reported in Table \ref{tab_gp}. These results illustrate the superiority of the approach in terms of computational time. This significant difference can be explained by the non-iterative nature of the proposed method compared to CG and ADMM.} Moreover, all the three methods converge to the same global minima as shown by the objective curves in Fig. \ref{ex1_face_obj}. The convergence of the objective curves is in agreement with the visual and numerical results.

\begin{figure}[h!]
\begin{center}
\subfigure[Observation]{\includegraphics[width=0.2\linewidth]{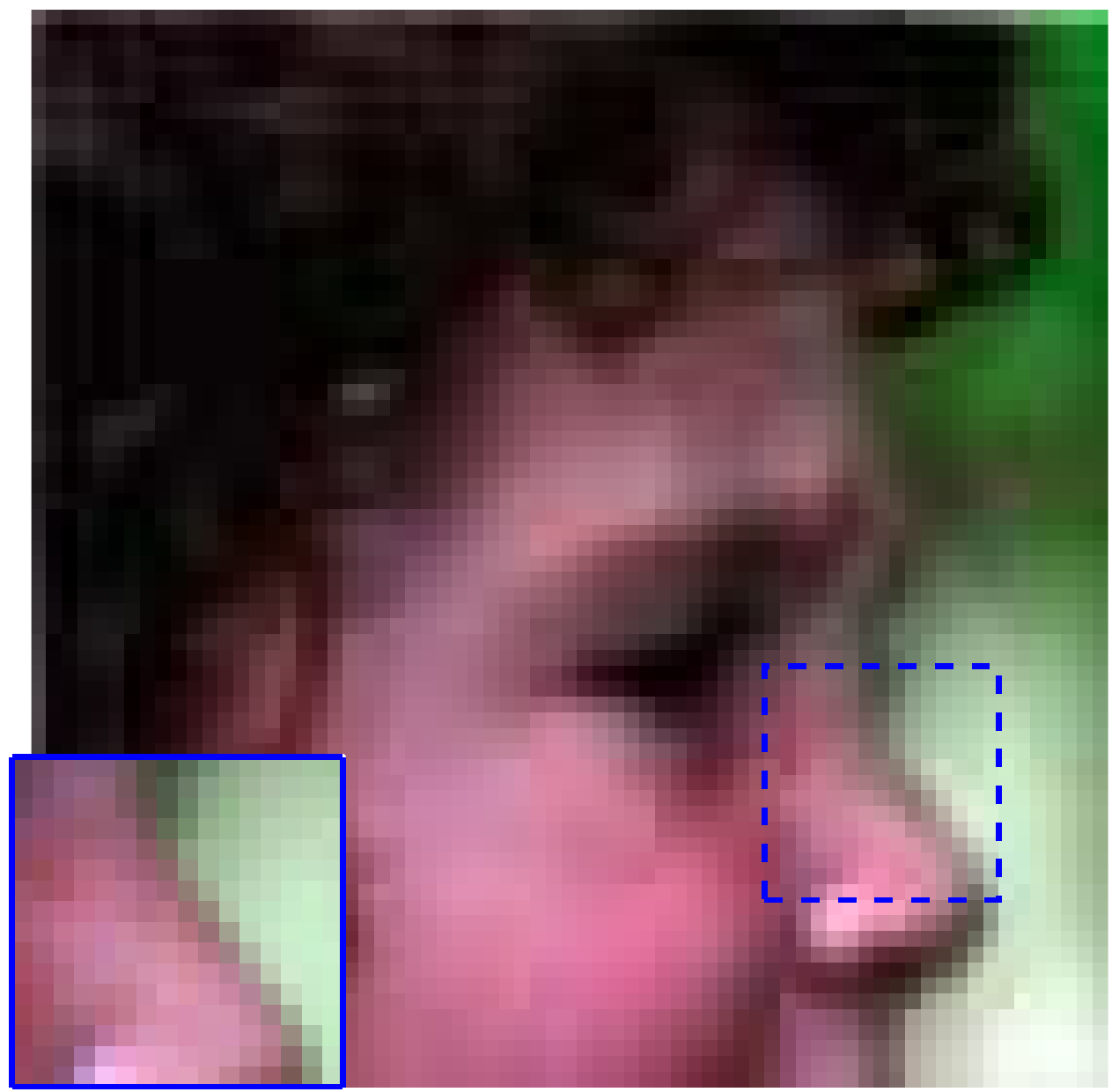} \label{fig_gp_LR}}
\subfigure[Ground truth]{\includegraphics[width=0.2\linewidth]{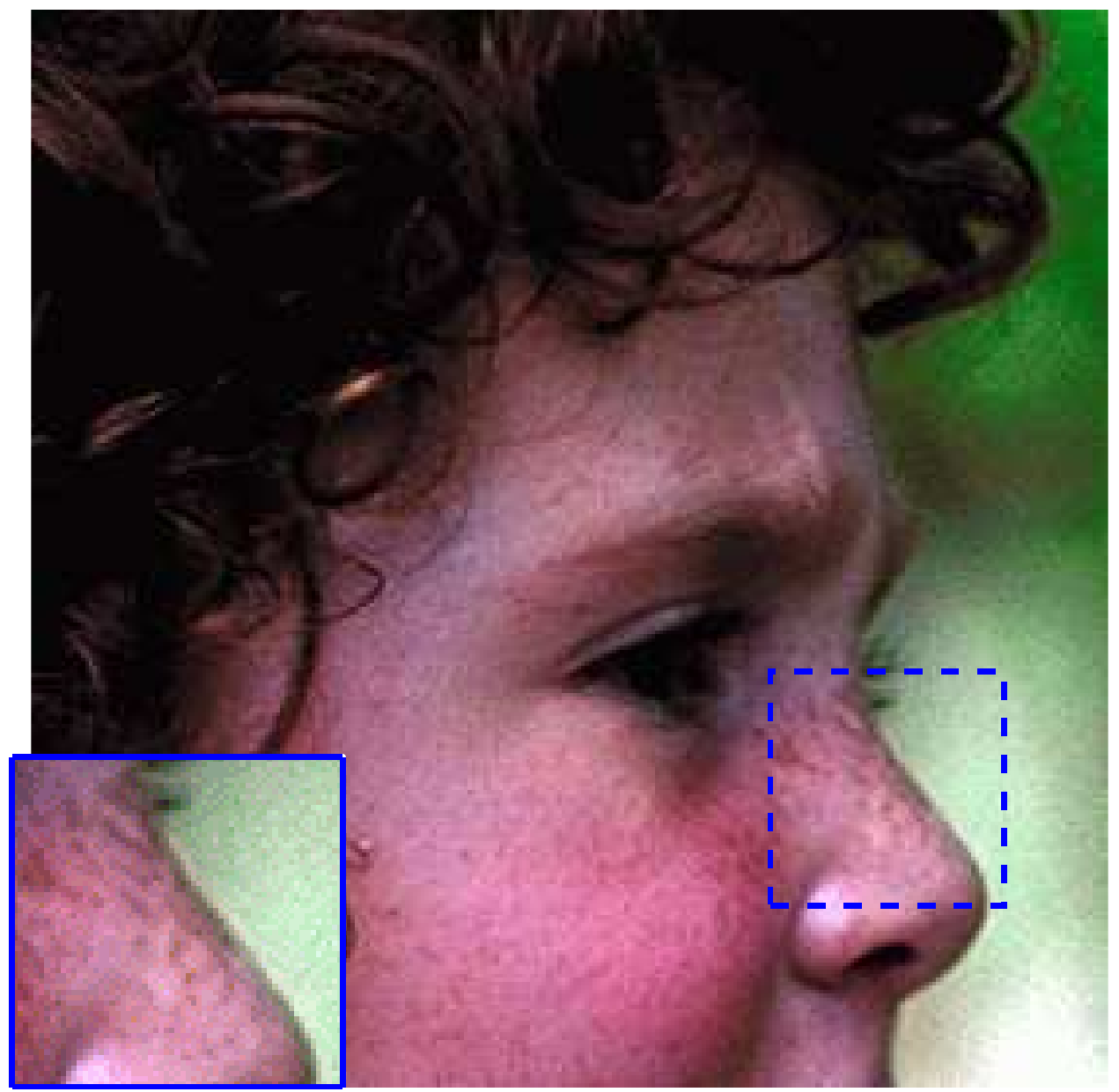} \label{fig_gp_HR}}
\subfigure[Bicubic]{\includegraphics[width=0.2\linewidth]{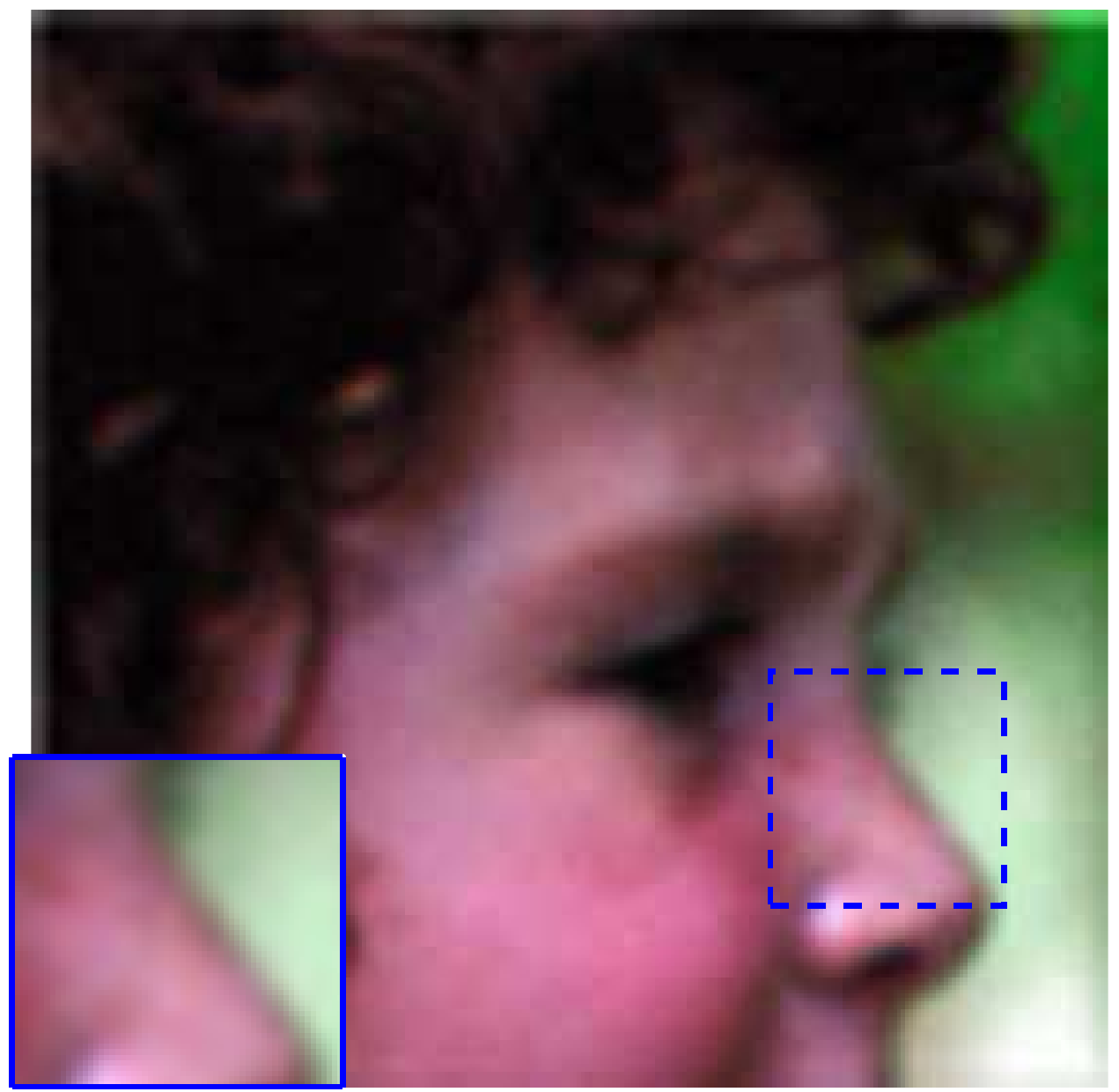} \label{fig_gp_bicubic}} \\
\subfigure[ADMM]{\includegraphics[width=0.2\linewidth]{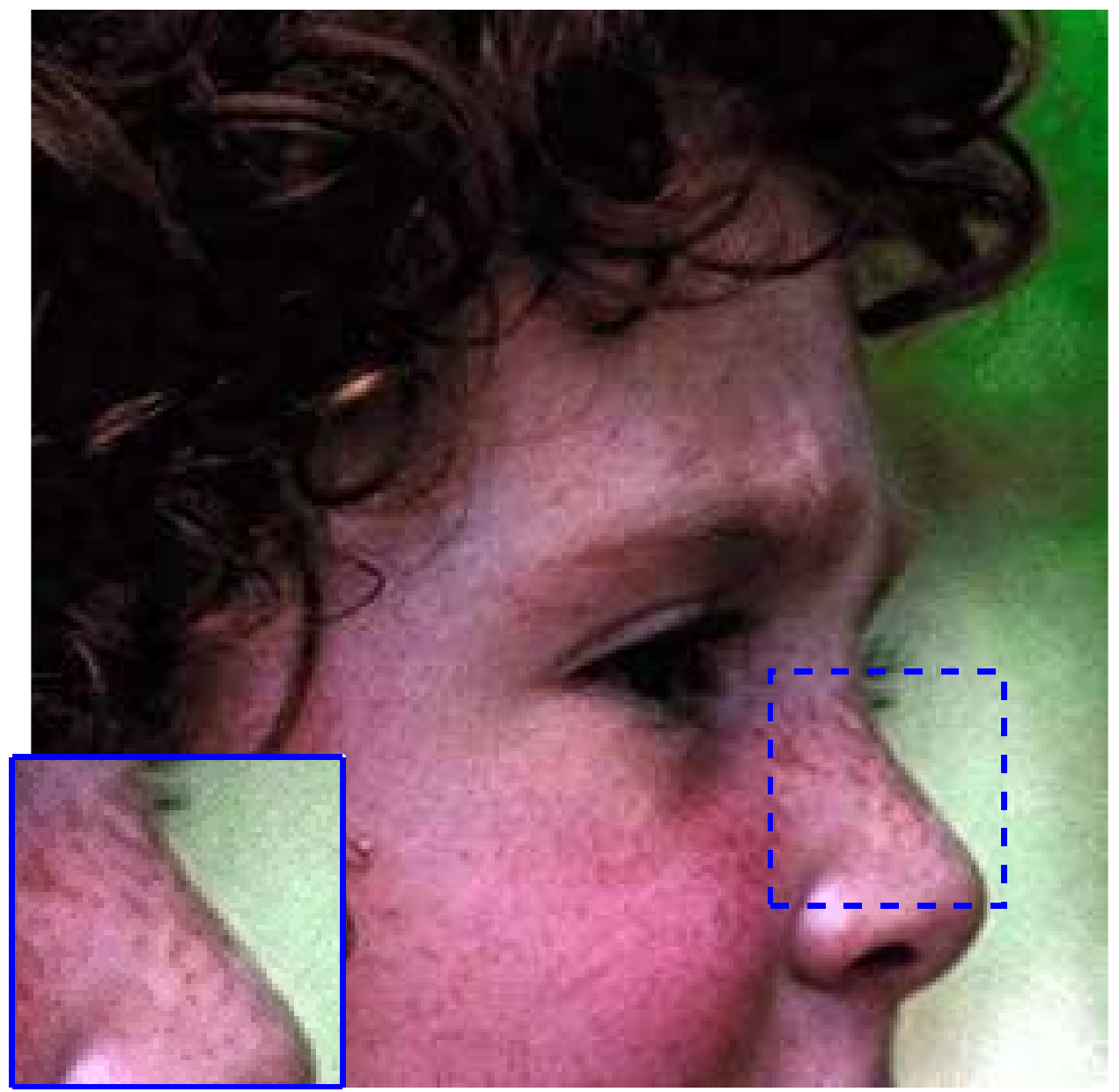} \label{fig_gp_al}}
\subfigure[CG \cite{SunJ_CVPR_2008}]{\includegraphics[width=0.2\linewidth]{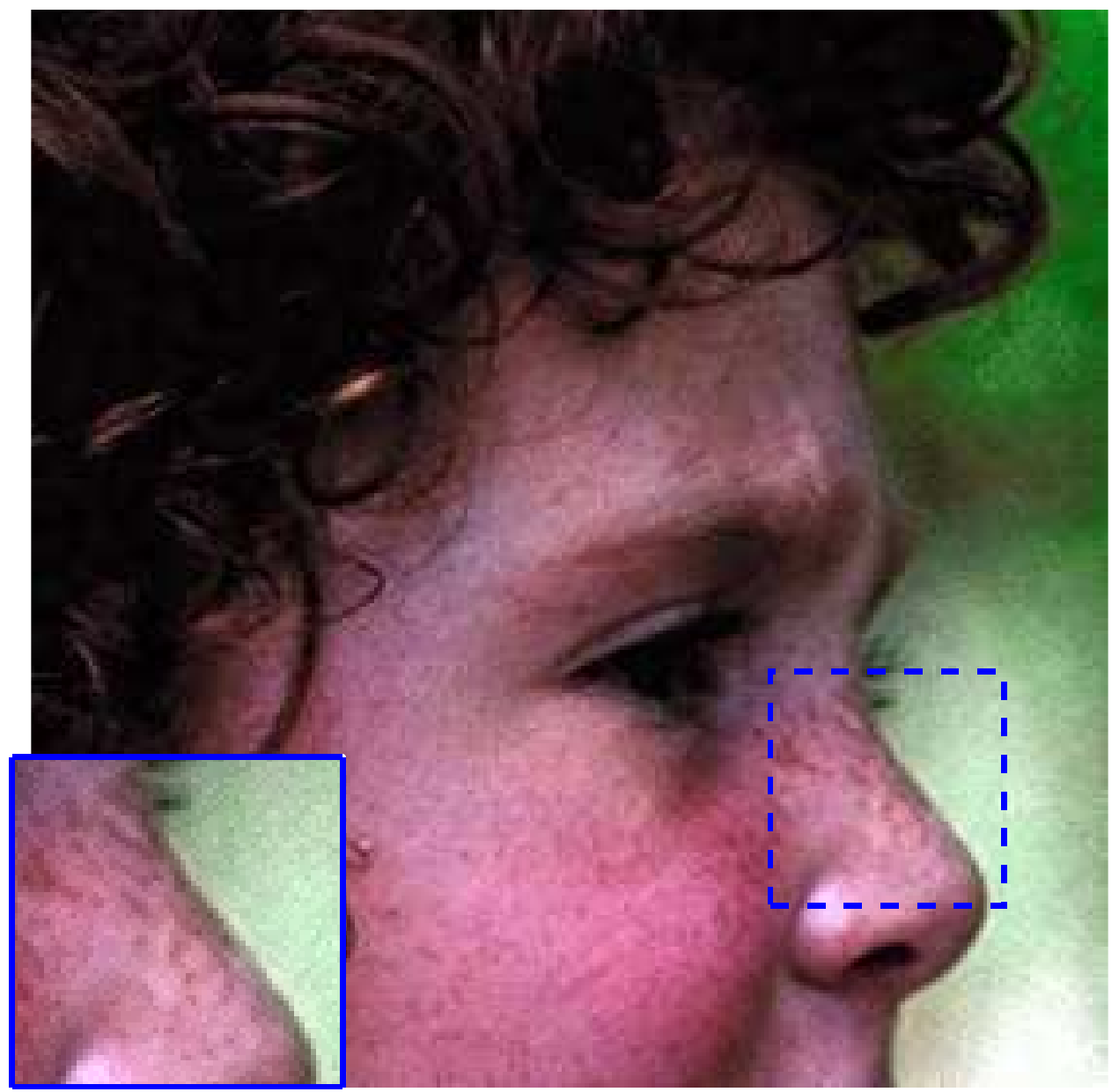} \label{fig_gp_gd}}
\subfigure[Algo. \ref{Algo_FSR_Gl2}]{\includegraphics[width=0.2\linewidth]{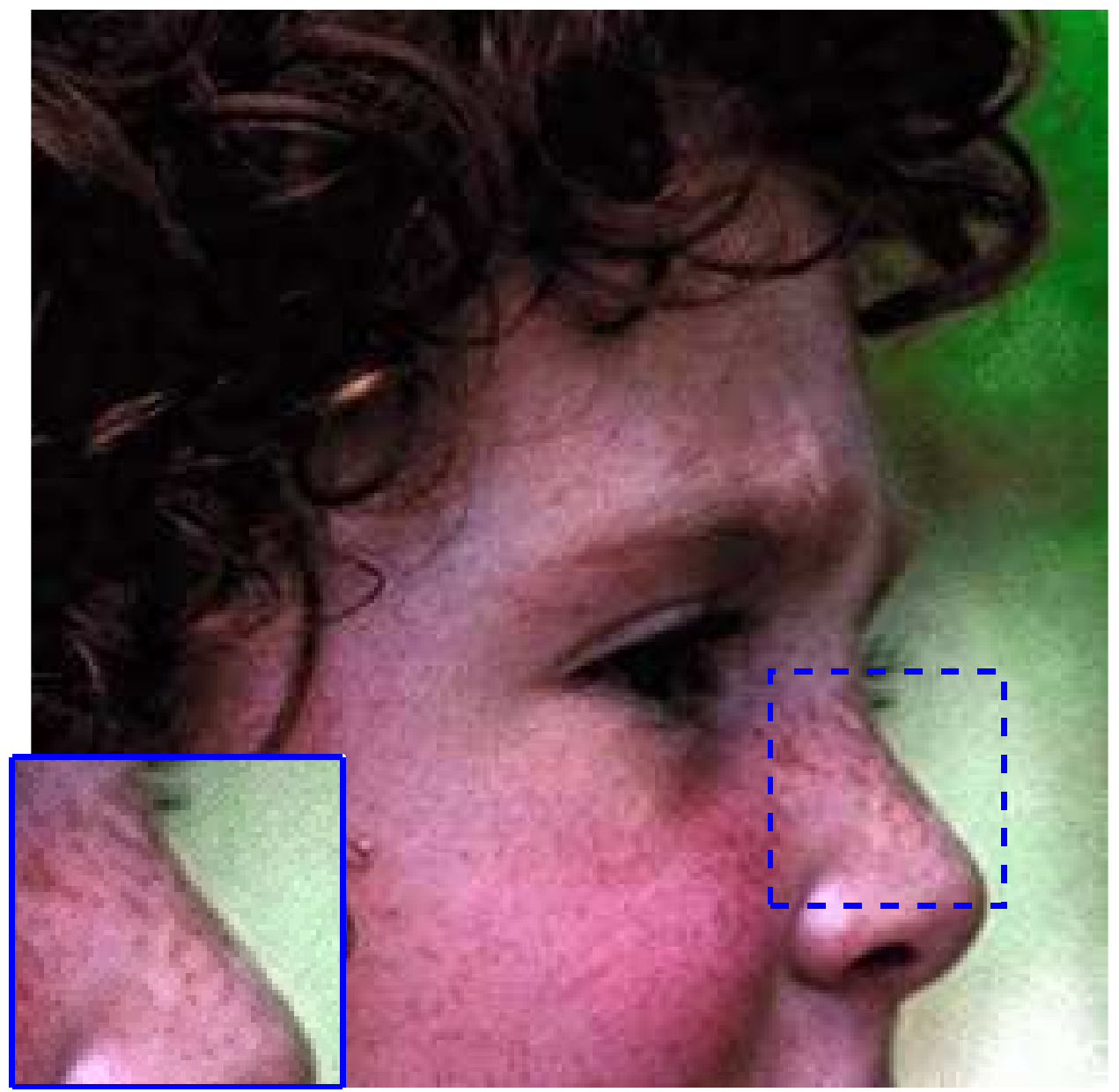} \label{fig_gp_proposed}}
\caption[The LOF caption]{SR of the face image when considering an $\ell_2-\ell_2$-model in the gradient domain: visual results.}
\label{ex1_face}
\end{center}
\end{figure}

\begin{table}[h!]
\begin{center}
\caption{SR of the face image when considering an $\ell_2-\ell_2$-model in the gradient domain: quantitative results.}
\label{tab_gp}
\begin{tabular}{|c|c|c|c|c|}
\hline
 Method  & PSNR (dB) &ISNR (dB) & MSSIM  & Time (s.) \\
 \hline
Bicubic  & 26.84 & -  & 0.49 & 0.001\\
ADMM     & 42.82 & 15.98  & 0.98 & 0.71\\
CG & 42.82 & 15.98 & 0.98 &0.35  \\
Algo. \ref{Algo_FSR_Gl2} & 42.82 & 15.98  & 0.98 & \textbf{0.009}\\
\hline
\end{tabular}
\end{center}
\end{table}

\begin{figure}[h!]
\begin{center}
\includegraphics[width=0.5\linewidth]{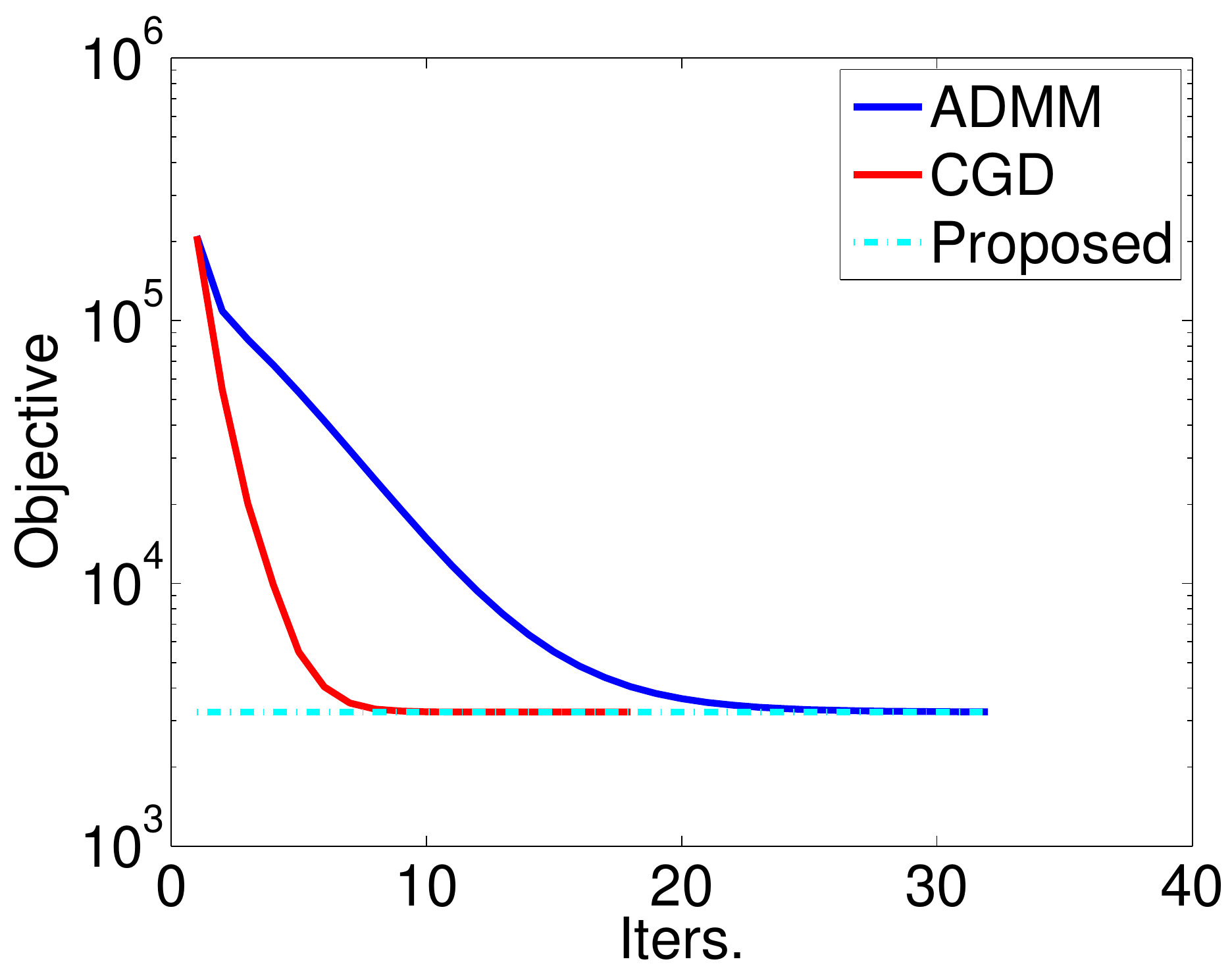} \label{fig_obj_gd}
\caption{SR of the face image when considering an $\ell_2-\ell_2$-model in the gradient domain: objective functions.}
\label{ex1_face_obj}
\end{center}
\end{figure}

\begin{figure}[h!]
\begin{center}
\subfigure[Observation]{\includegraphics[width=0.2\linewidth]{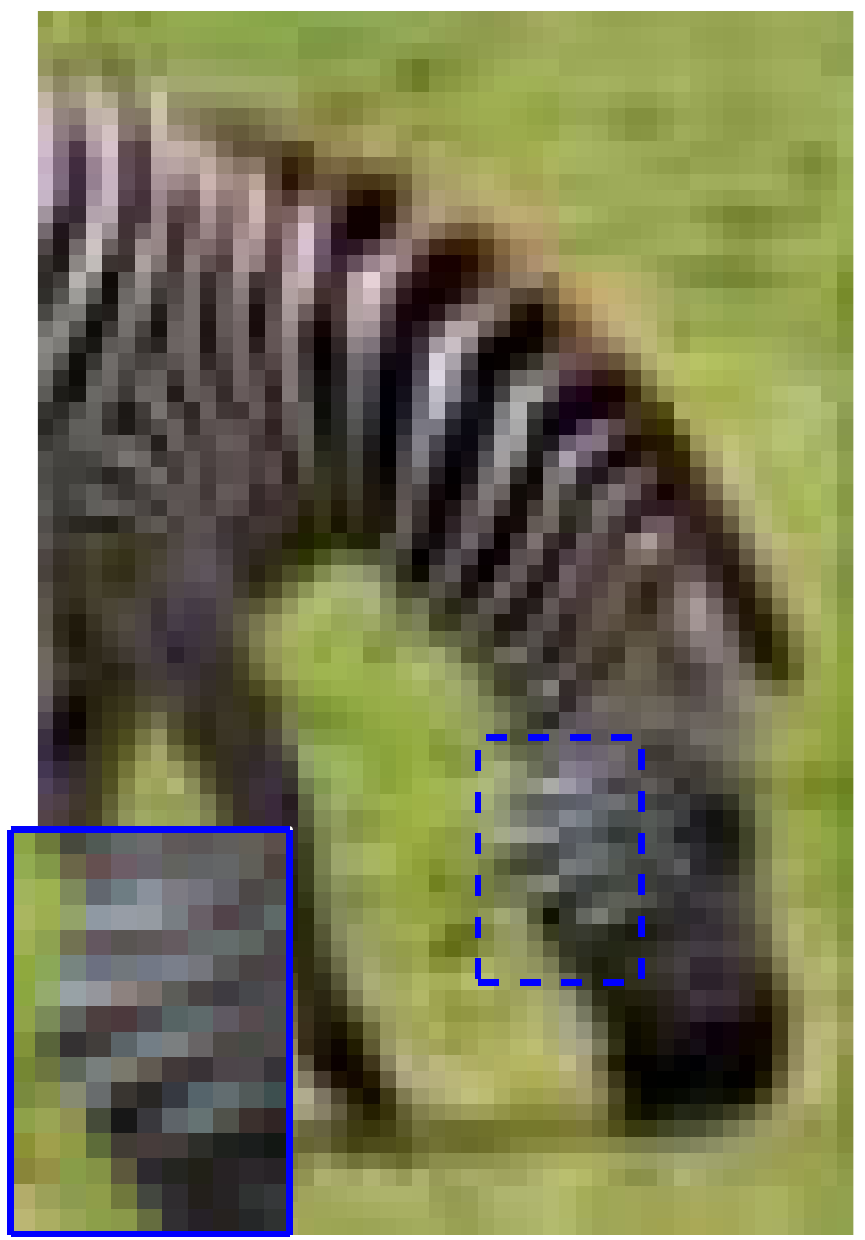} \label{fig_learn_obs}}
\subfigure[Ground truth]{\includegraphics[width=0.2\linewidth]{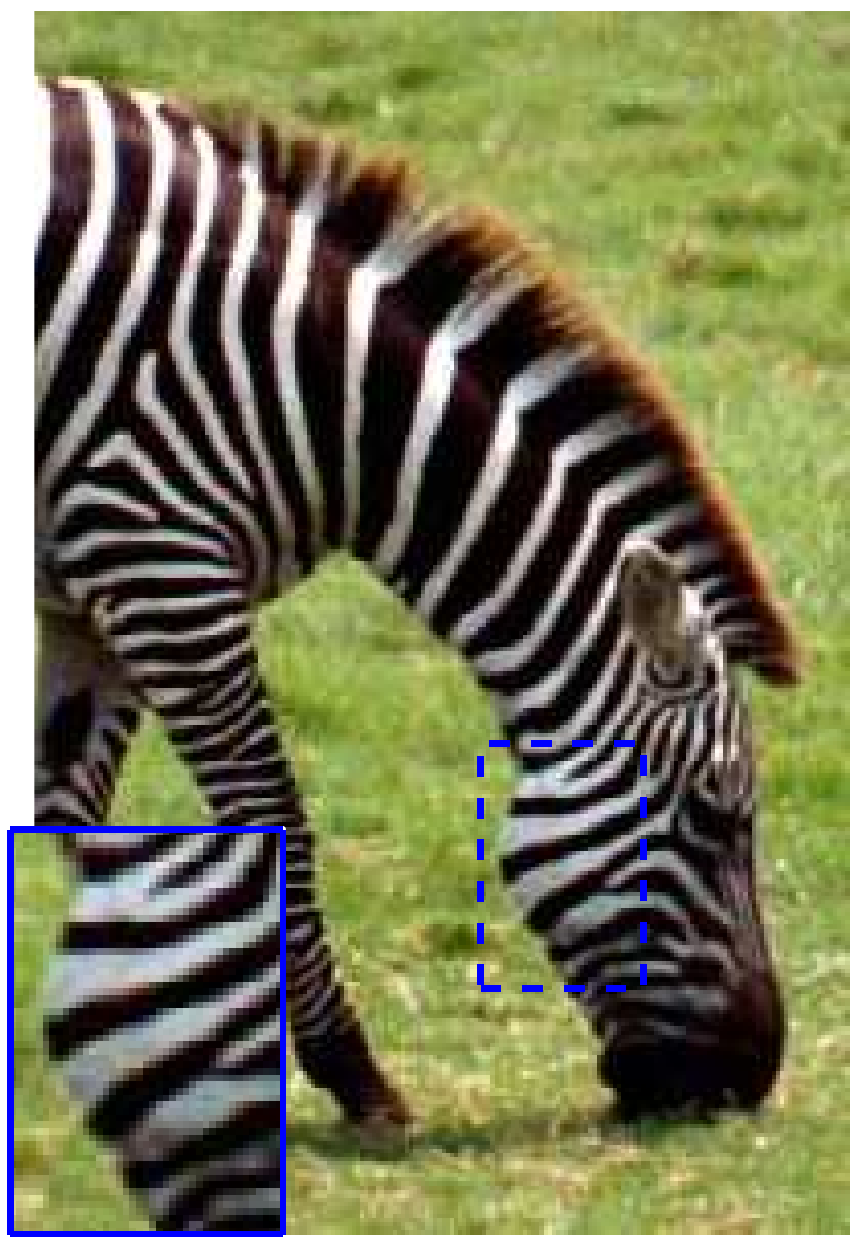} \label{fig_learn_gnd}}
\subfigure[Bicubic]{\includegraphics[width=0.2\linewidth]{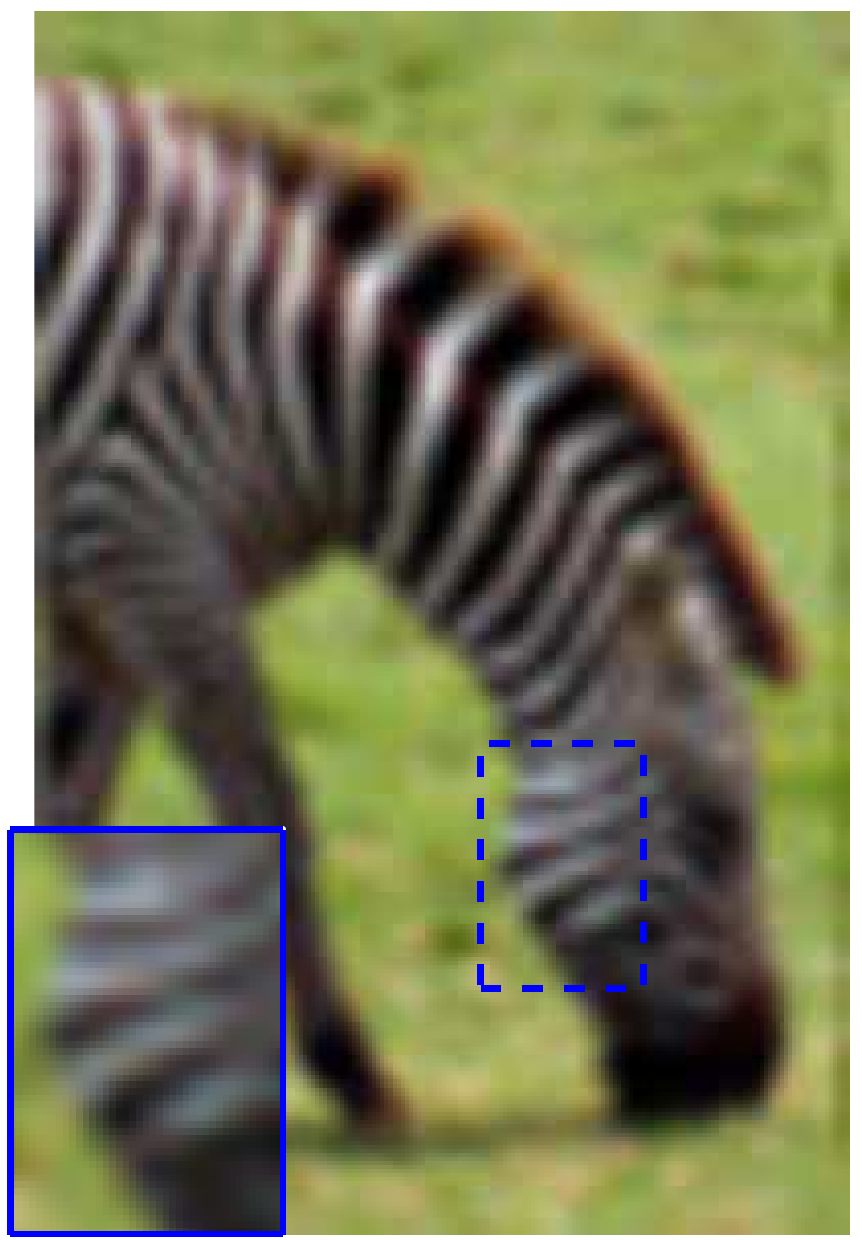} \label{fig_learn_bicubic}}  \\
\subfigure[SC \cite{Yang2010TIP}]{\includegraphics[width=0.2\linewidth]{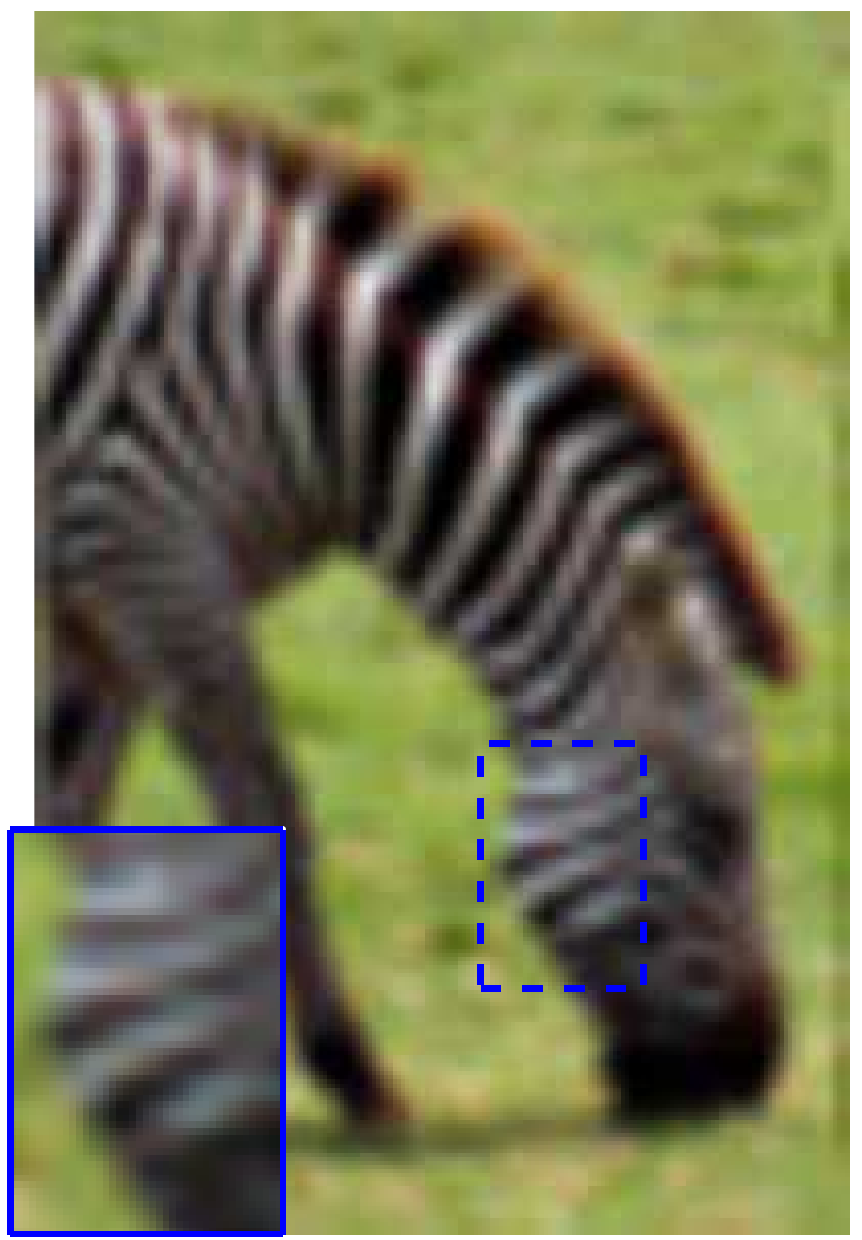} \label{fig_learn_proposed}}
\subfigure[SC+GD \cite{Yang2010TIP}]{\includegraphics[width=0.2\linewidth]{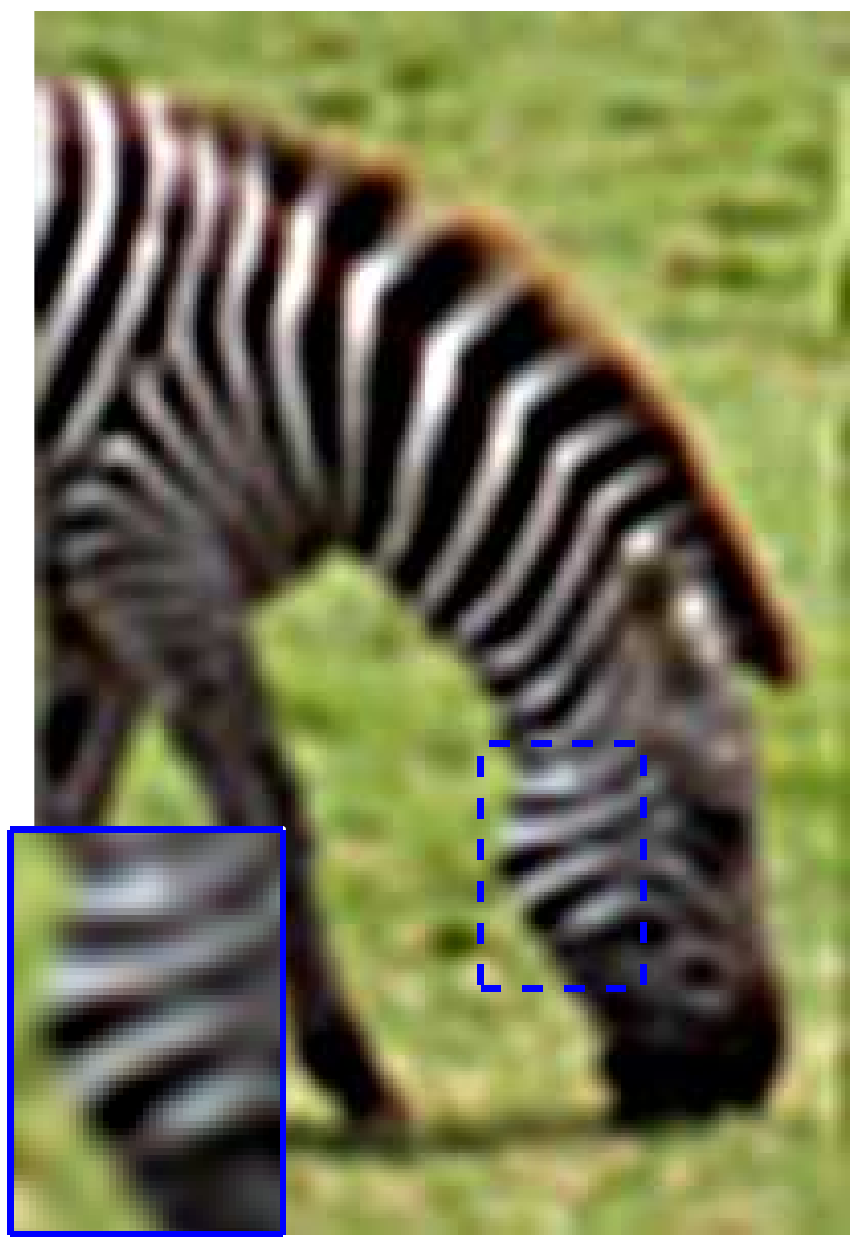} \label{fig_learn_yang}}
\subfigure[SC+Algo. \ref{Algo:FastSR}]{\includegraphics[width=0.2\linewidth]{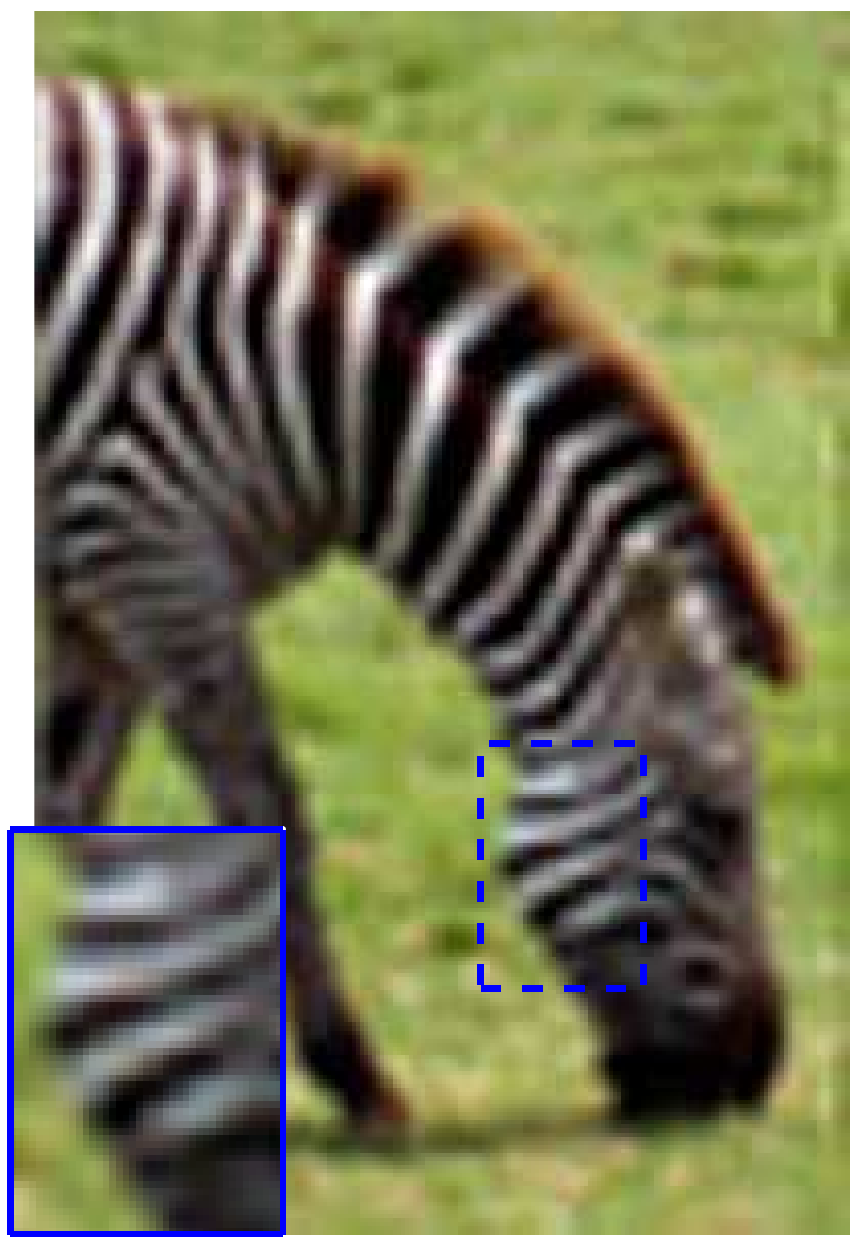} \label{fig_learn_proposed}}
\caption{SR of the zebra image when considering a learning-based $\ell_2$-norm regularization: visual results.}
\label{ex1_zebra}
\end{center}
\end{figure}

\subsubsection{Learning-based $\ell_2$-norm regularization}
\AB{This section studies the performance of the algorithm obtained when the analytical solution of Theorem 1 is embedded in the learning-based method of \cite{Yang2010TIP}. The method investigated in \cite{Yang2010TIP} computed an initial estimation of the HR image via sparse coding (SC) and used a back-projection (BP) procedure to improve the SR performance. The BP operation was performed by a GD method in \cite{Yang2010TIP}. Here, this GD step has been replaced by the analytical solution provided by Theorem 1 and Algo. \ref{Algo:FastSR}. The image ``zebra'' was used in this experiment to compare the performance of both algorithms\footnote{For comparison purpose, the authors used the MATLAB code corresponding to \cite{Yang2010TIP} available at \url{http://www.ifp.illinois.edu/~jyang29.}}. The LR and HR images (of size $300 \times 200$) are shown in Fig. \ref{fig_learn_obs} and \ref{fig_learn_gnd}. The regularization parameter was set to $\tau = 0.1$. The restored images shown in Figs. \ref{fig_learn_bicubic}-\ref{fig_learn_yang} were obtained using the initial SC estimation proposed in \cite{Yang2010TIP}, the back-projected SC image combined with the gradient descent (GD) algorithm of \cite{Yang2010TIP} (referred to as ``SC + GD'') and the proposed closed-form solution (referred to as ``SC + Algo. \ref{Algo:FastSR}"). The corresponding numerical results are reported in Table \ref{tab_learn}. The restored images obtained with the two back-projection approaches are clearly better than the restoration obtained with the SC method. While the quality of the images obtained with these projection approaches is similar, the use of the analytical solution of Theorem 1 allows the computational cost of the GD step to be reduced significantly.}

\begin{table}[h!]
\setlength{\tabcolsep}{1.5pt}
\begin{center}
\caption{SR of the zebra image when considering a learning-based $\ell_2$-norm regularization: quantitative results.}
\label{tab_learn}
\begin{tabular}{|c|c|c|c|c|}
\hline
 Method  & PSNR (dB)  &ISNR (dB) & MSSIM  & Time (s.) \\
 \hline
Bicubic                  & 18.98 &-     & 0.37  & 0.001\\
SC \cite{Yang2010TIP}    & 19.15 &0.16  & 0.38  &  170.9\\
SC+GD \cite{Yang2010TIP} & 20.76 &1.78  & 0.47  &  170.9+1.23\\
SC+Algo. \ref{Algo:FastSR}              & 29.99 &1.88  & 0.48  &  170.9+\textbf{0.01}\\
\hline
\end{tabular}
\end{center}
\end{table}

\subsection{Embedding the $\ell_2-\ell_2$ analytical solution into the ADMM framework}
In this second group of experiments, we consider two non-Gaussian priors that have been widely used for image reconstruction problems: the TV regularization in the spatial domain and the $\ell_1$-norm regularization in the wavelet domain. In both cases, the analytical solution of Theorem 1 is embedded into a standard ADMM algorithm inspired from \cite{MNg2010SR_TV} (the resulting algorithms referred to as Algo. 4 and 5 are detailed in Appendix C). The stopping criterion for both implementations is chosen as the relative cost function error defined as
\begin{equation}
\frac{|f(\bfx^{k+1}) - f(\bfx^{k})|}{f(\bfx^{k})}
\end{equation}
where $f(\bfx) = \frac{1}{2} \|\bfy - \bfS\bfH\bfx\|_2^2 + \tau \phi(\bfA\bfx)$.
Note that other stopping criteria such as those studied in \cite{Boyd2011ADMM} could also be investigated. The $512 \times 512$ images ``Lena", ``monarch" and ``Barbara" were considered in these experiments. The observed LR images and the HR images (ground truth) are displayed in Fig. \ref{images_TV} (first two columns).

\begin{figure*}
\centering\begin{tabular}{@{}c@{}c@{}c@{ }c@{ }c@{ }c@{}}
Observation&Ground truth &Bicubic & ADMM \cite{MNg2010SR_TV} & Algo. 4 \\
\includegraphics[width=0.18\linewidth]{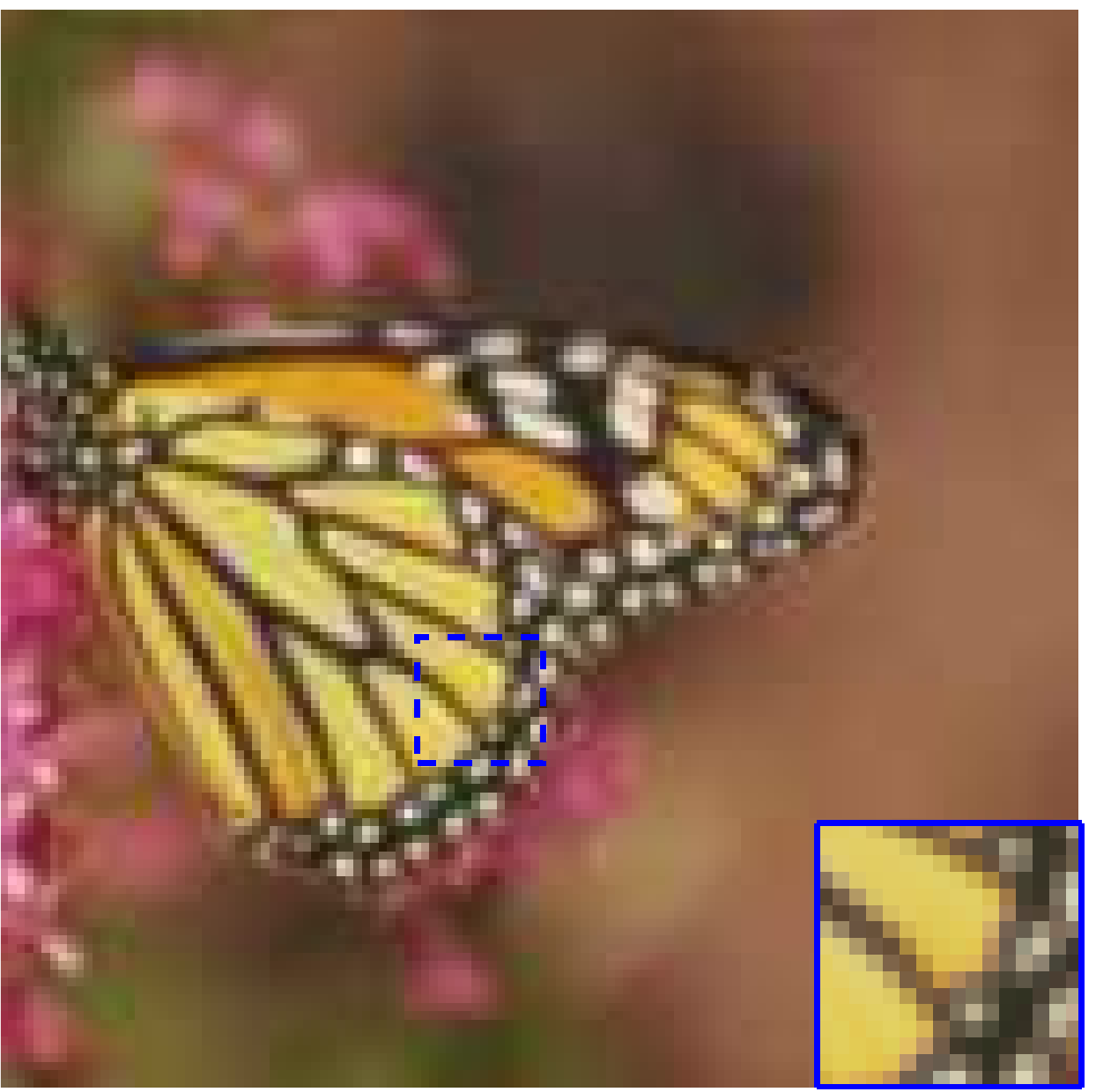} & \label{fig_obs_monarch_TV}
\includegraphics[width=0.18\linewidth]{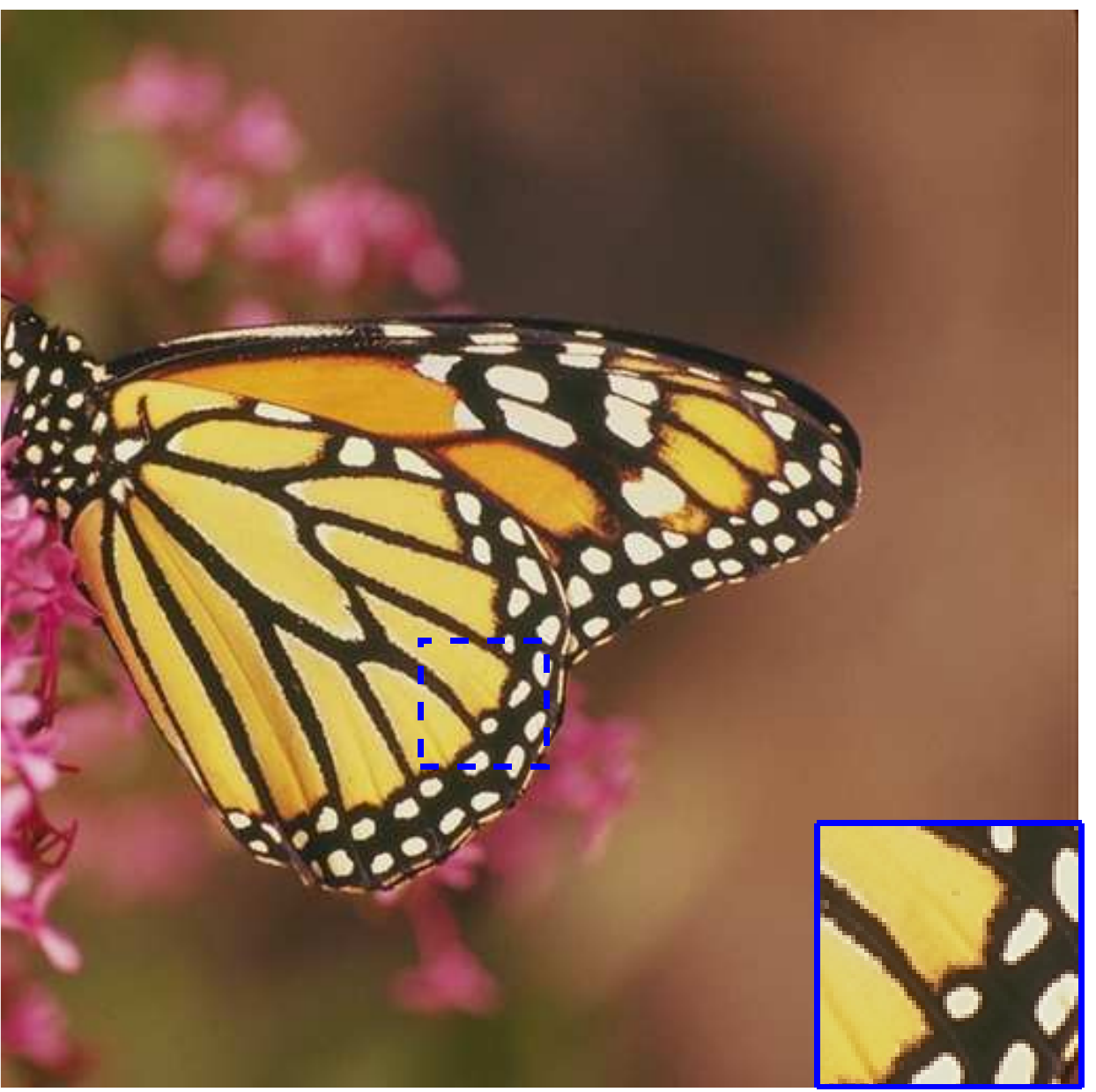}& \label{fig_bicubic_monarch_TV}
\includegraphics[width=0.18\linewidth]{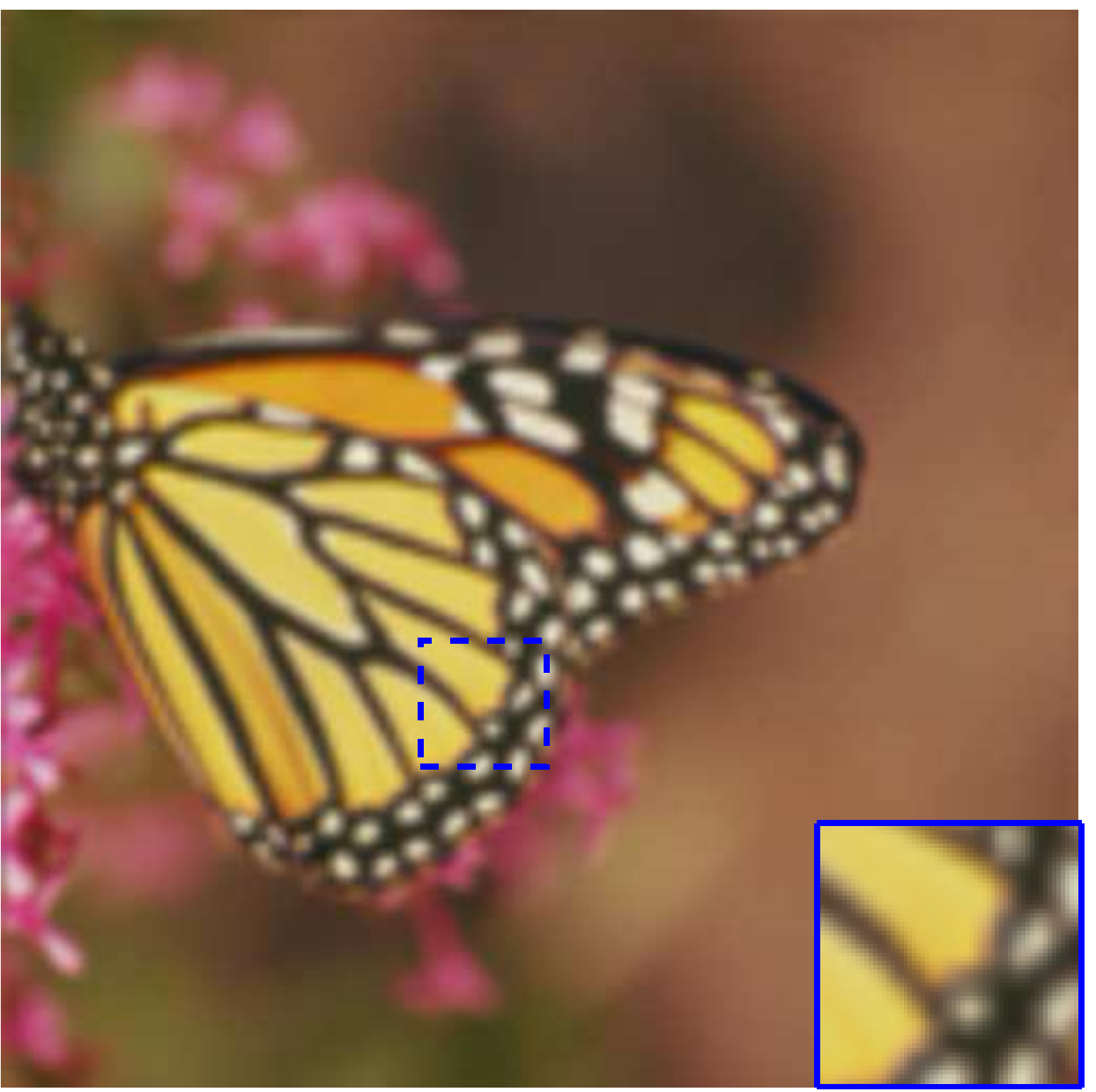} &\label{fig_bicubic_monarch_TV}
\includegraphics[width=0.18\linewidth]{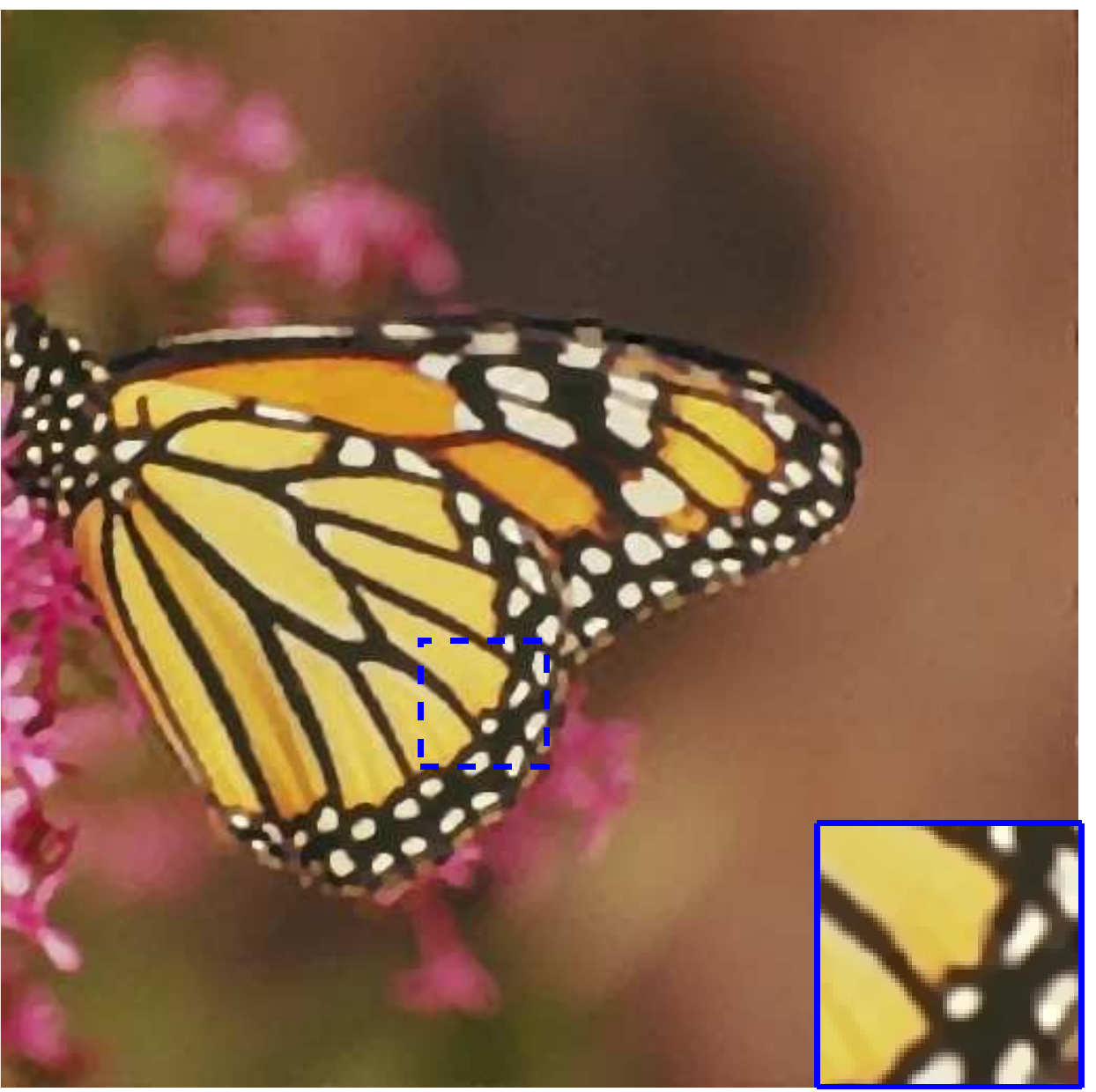} &\label{fig_direct_monarch_TV}
\includegraphics[width=0.18\linewidth]{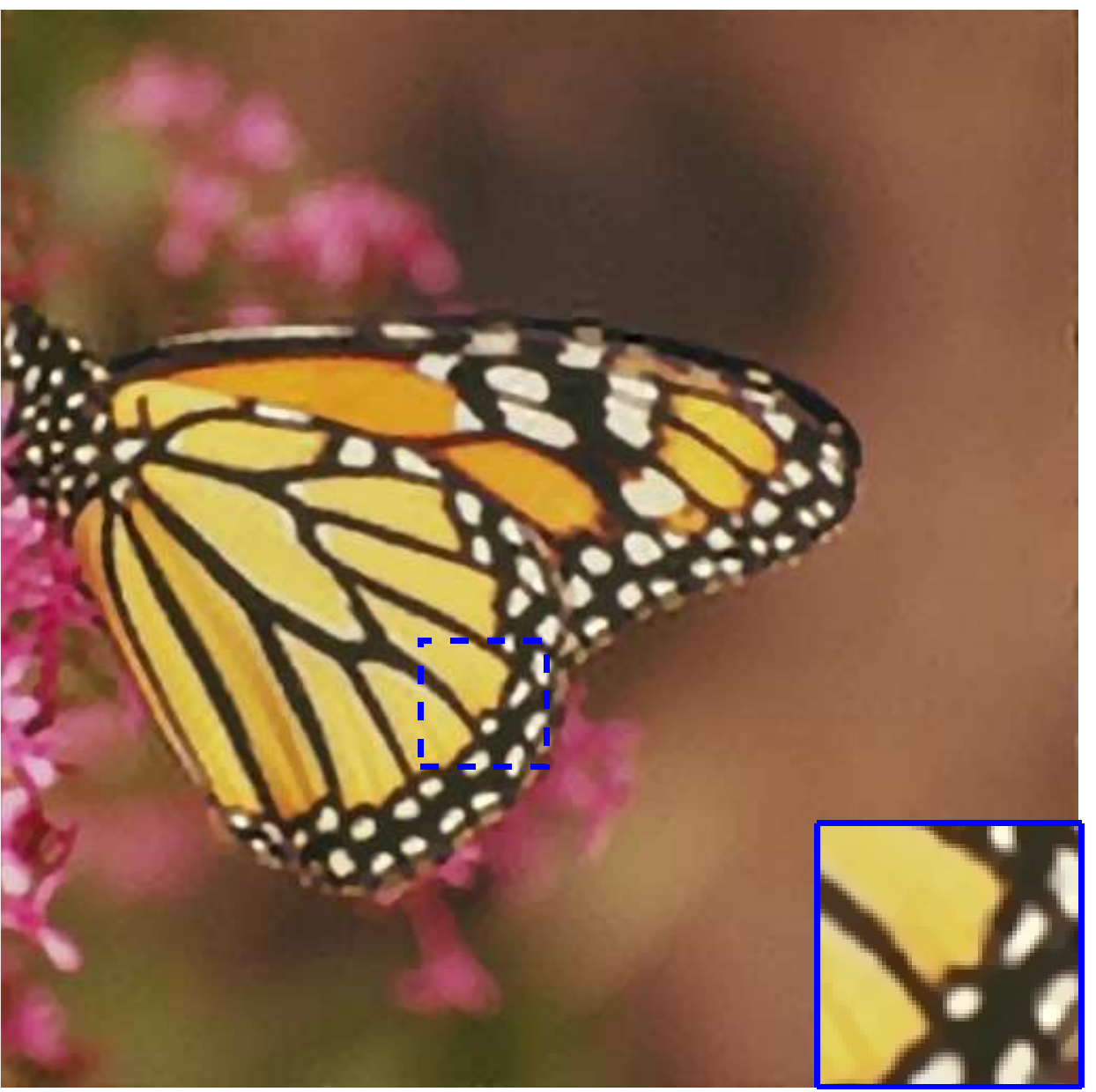} &\label{fig_FSR_monarch_TV} \\

\includegraphics[width=0.18\linewidth]{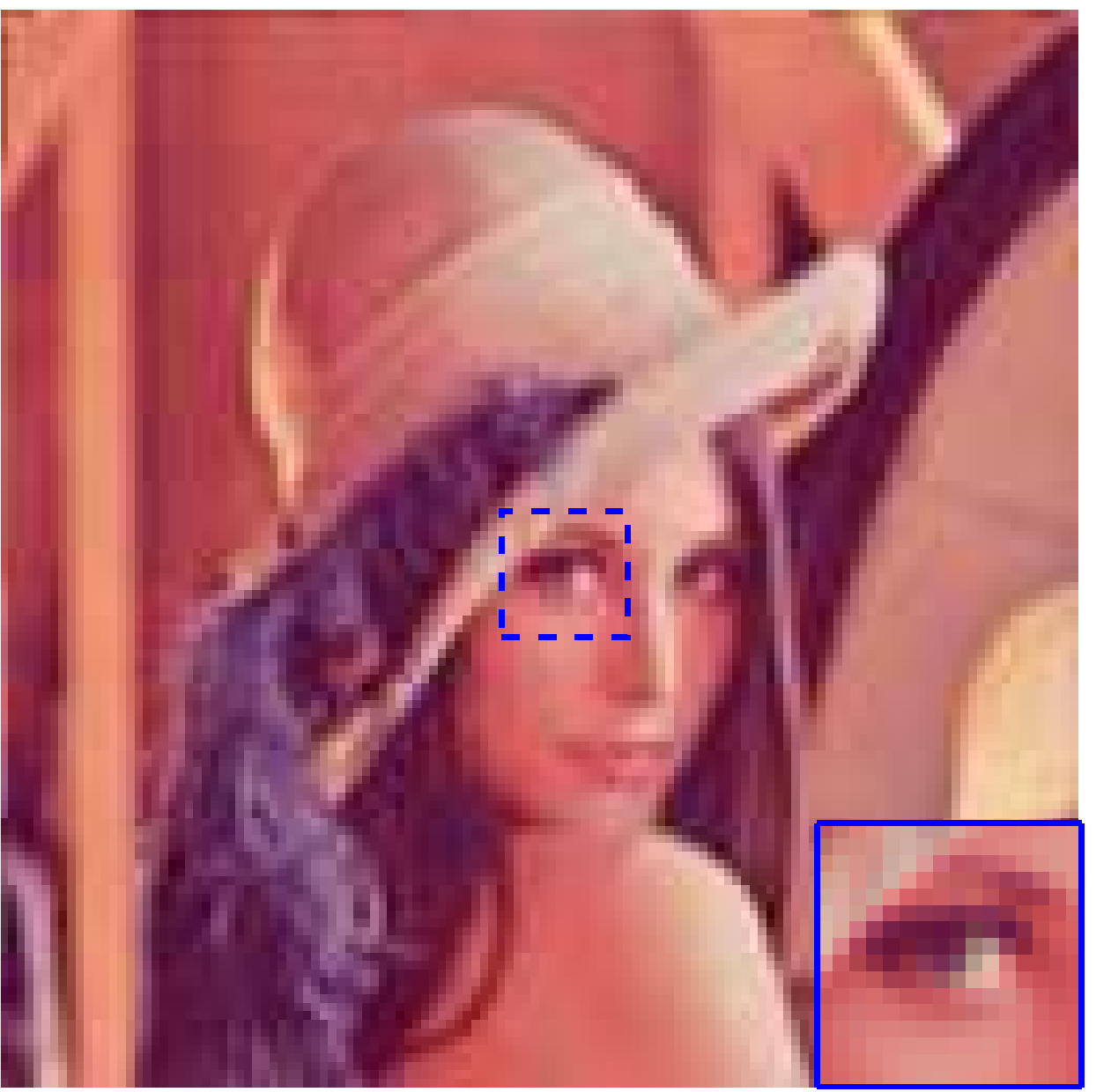} & \label{fig_obs_lena_TV}
\includegraphics[width=0.18\linewidth]{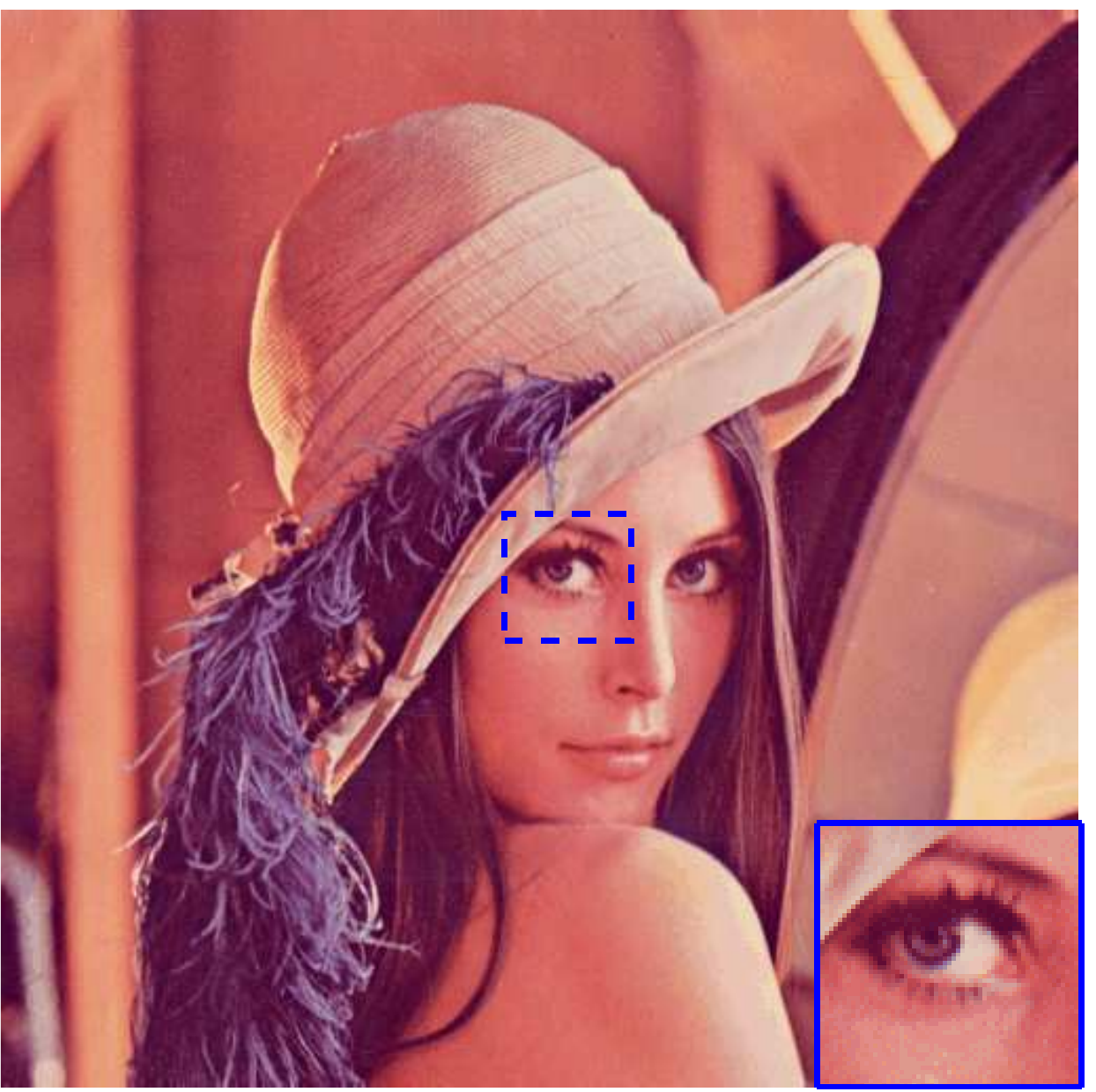}& \label{fig_bicubic_monarch_TV}
\includegraphics[width=0.18\linewidth]{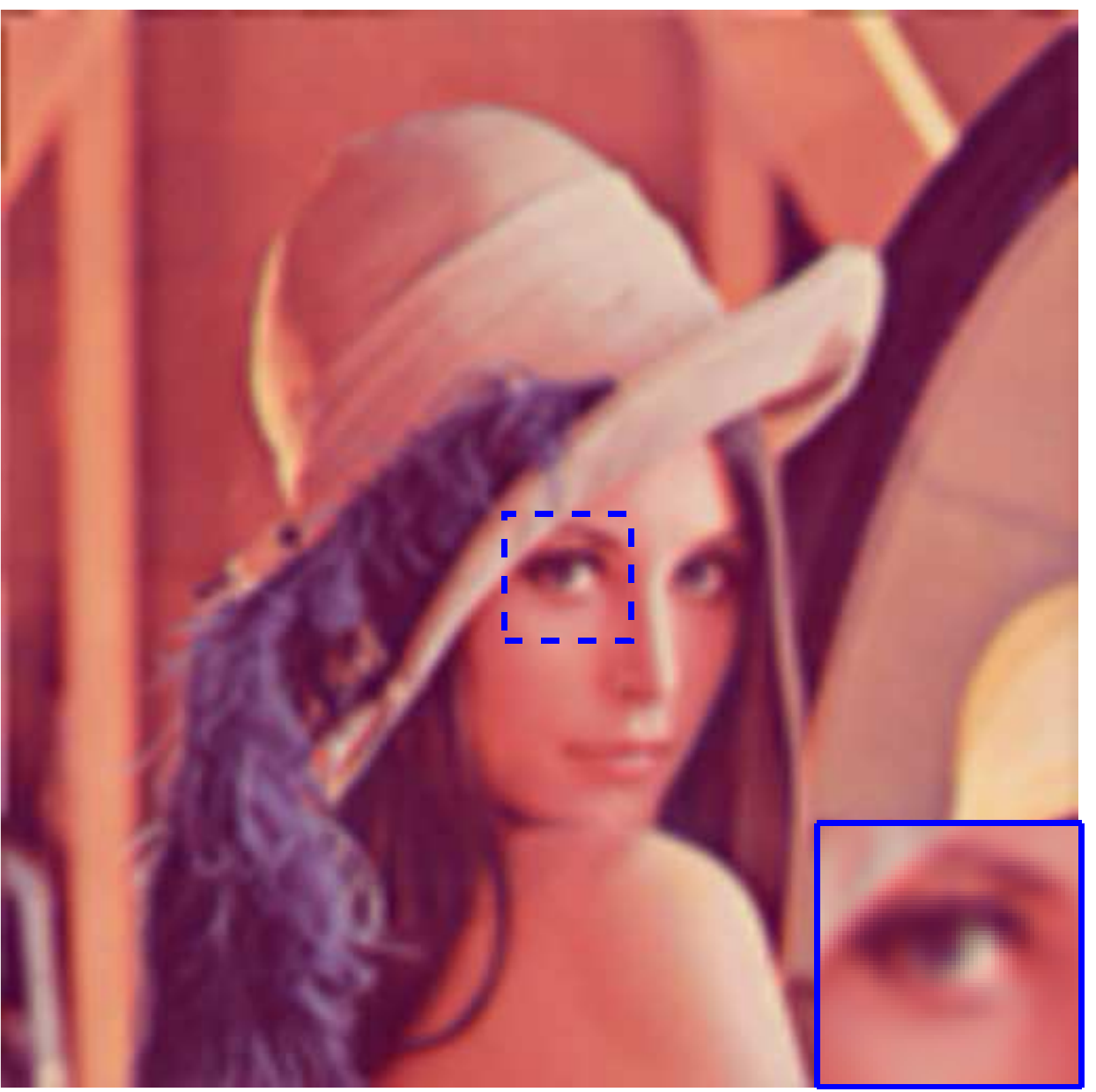} &\label{fig_bicubic_lena_TV}
\includegraphics[width=0.18\linewidth]{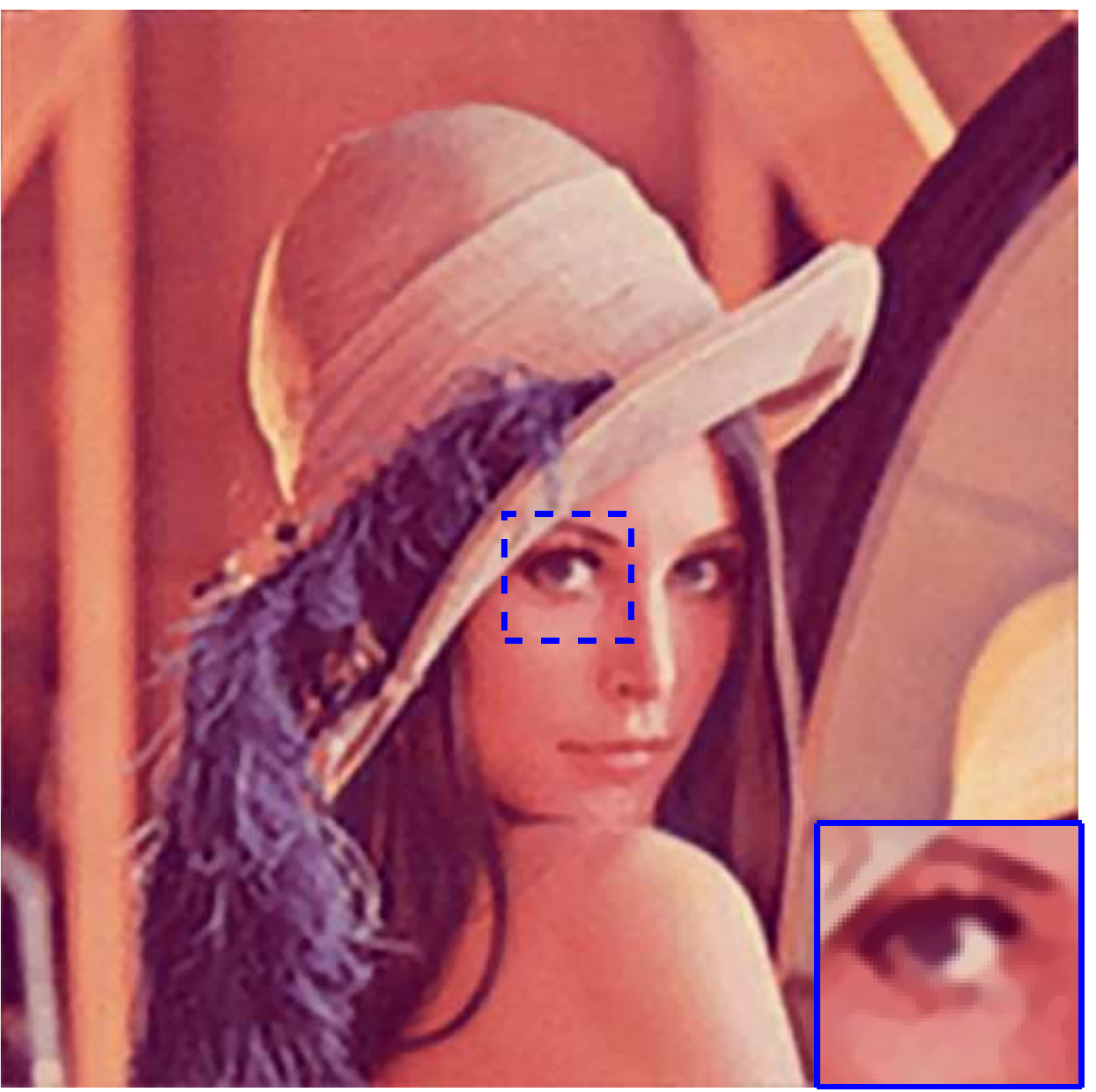} &\label{fig_direct_lena_TV}
\includegraphics[width=0.18\linewidth]{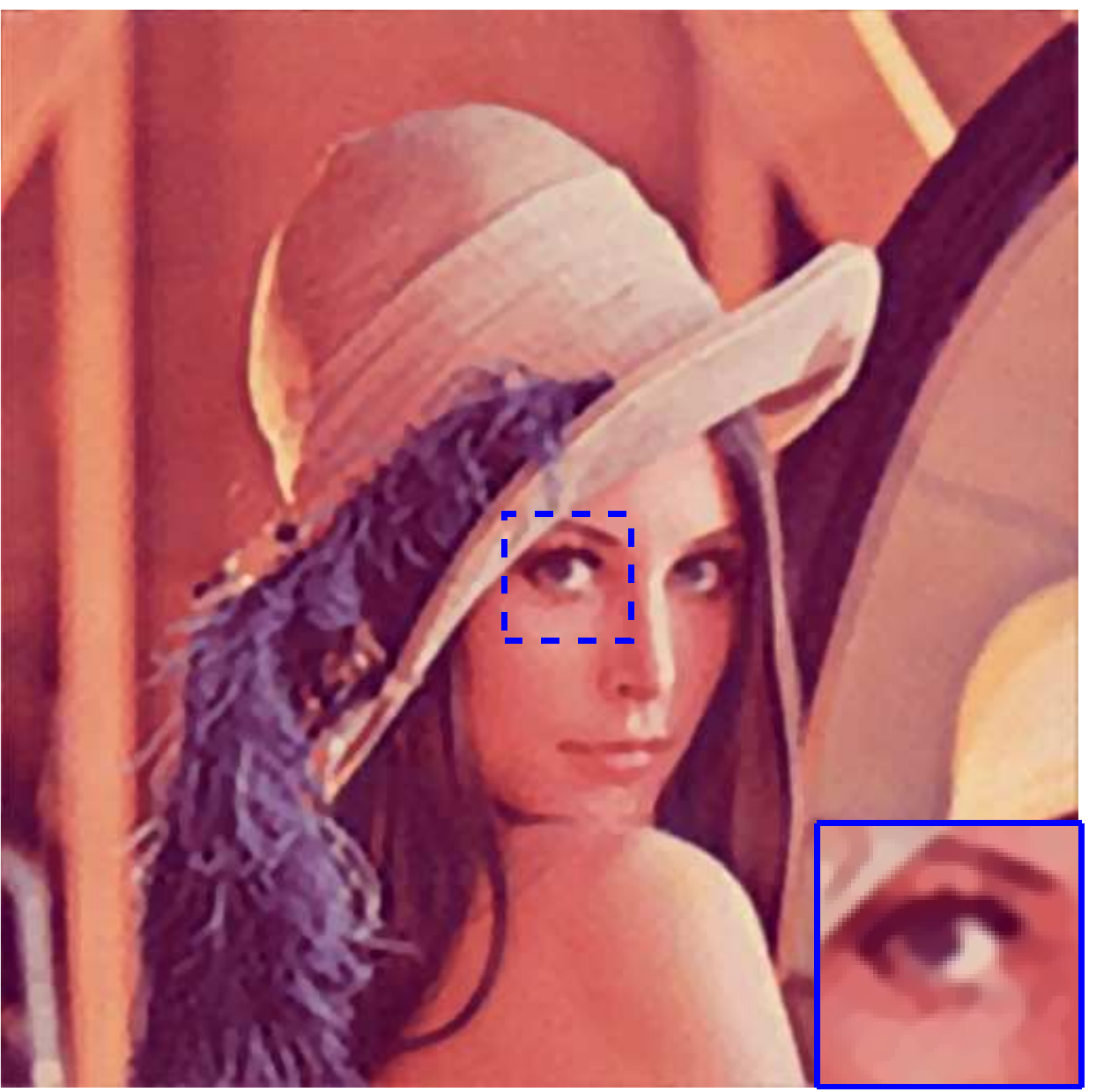} &\label{fig_FSR_lena_TV}\\

\includegraphics[width=0.18\linewidth]{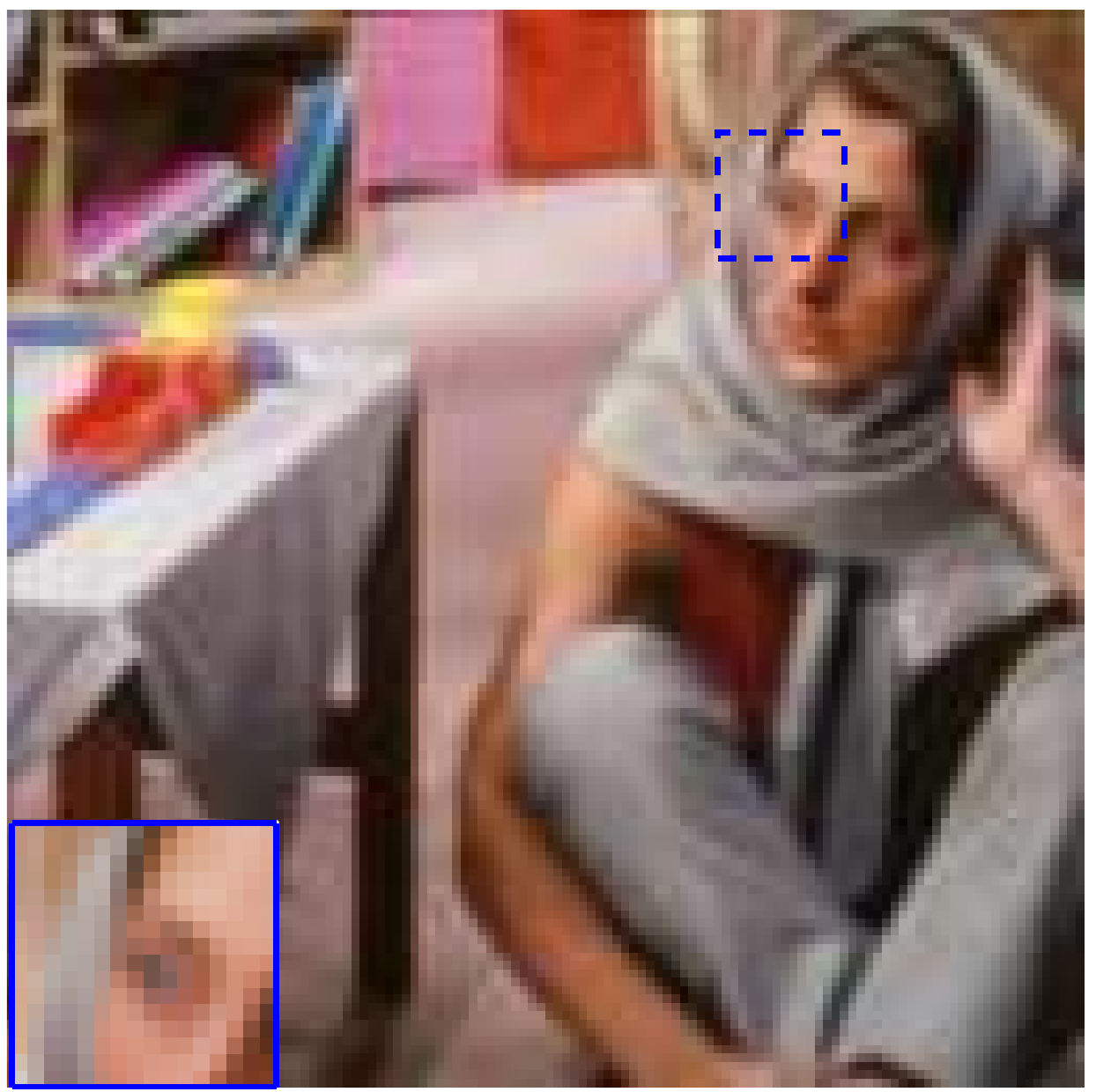} & \label{fig_obs_barbara_TV}
\includegraphics[width=0.18\linewidth]{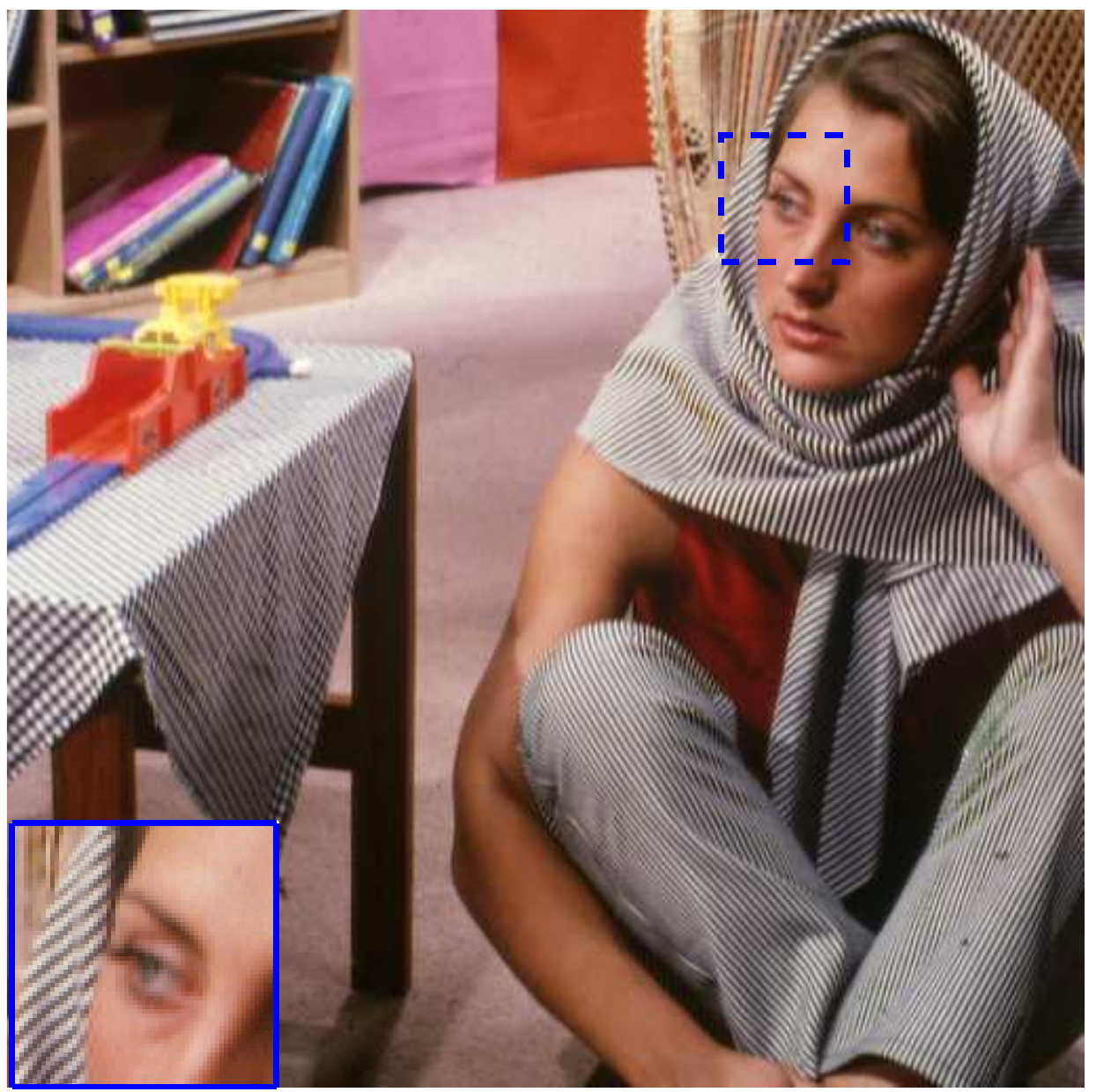}& \label{fig_bicubic_monarch_TV}
\includegraphics[width=0.18\linewidth]{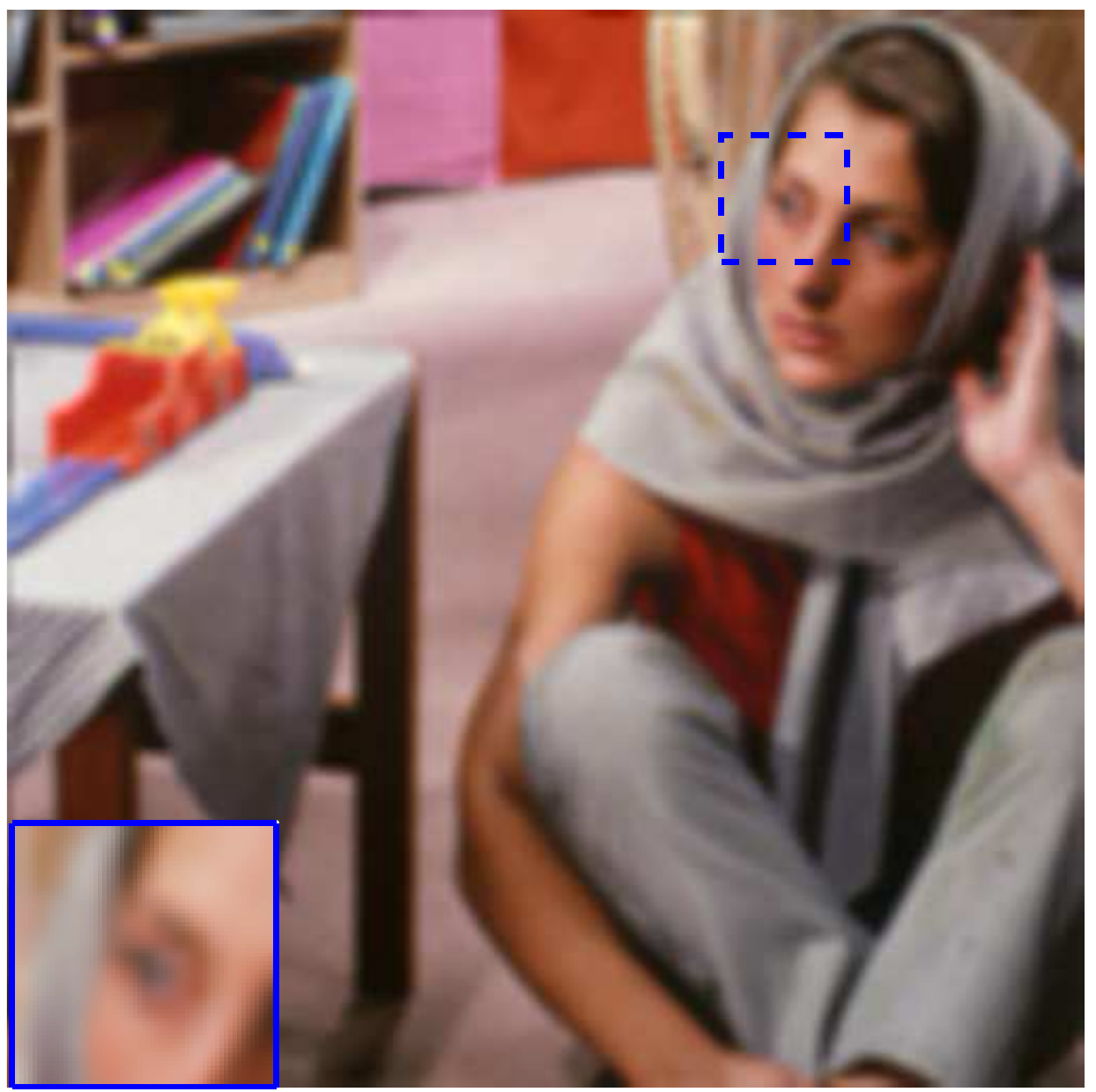} &\label{fig_bicubic_barbara_TV}
\includegraphics[width=0.18\linewidth]{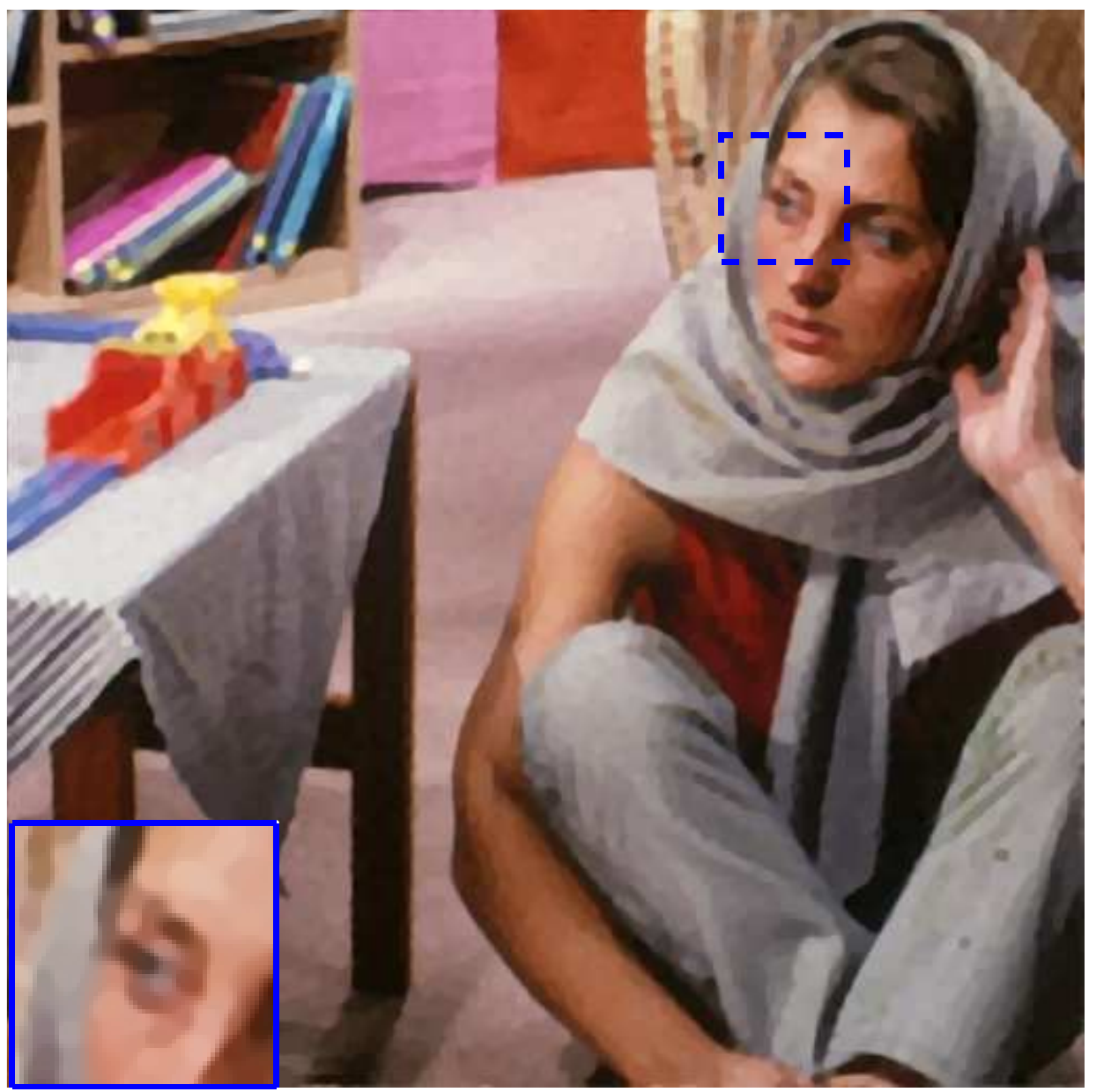} &\label{fig_direct_barbara_TV}
\includegraphics[width=0.18\linewidth]{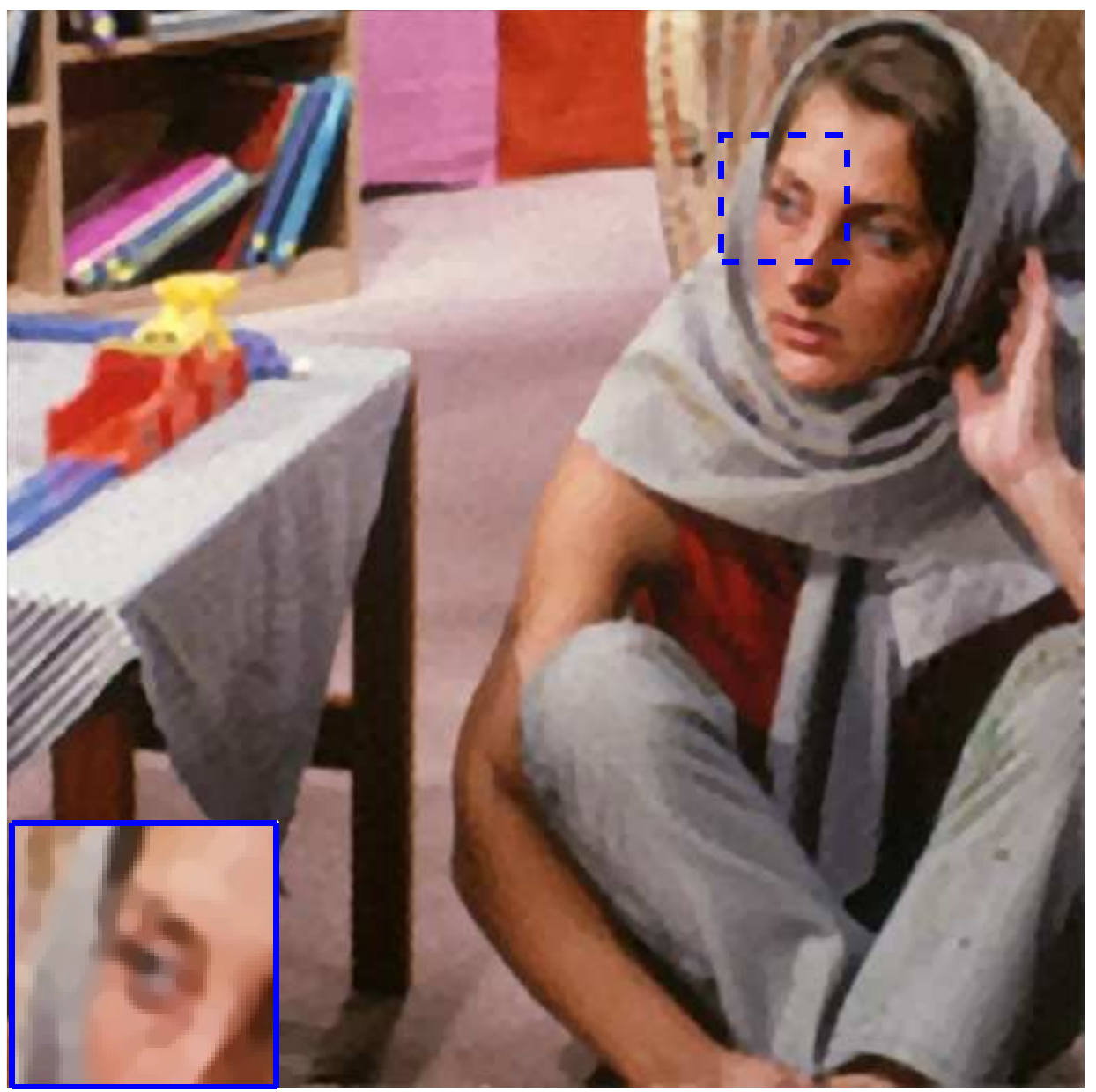} &\label{fig_FSR_barbara_TV}
\end{tabular}
\caption{SR of the Monarch, Lena and Barbara images when considering a TV-regularization: visual results.}
\label{images_TV}
\end{figure*}

\subsubsection{TV-regularization}
The regularization parameter was manually fixed  (by cross validation) to $\tau=2\times 10^{-3}$ for the image ``Lena", to $\tau = 1.8\times 10^{-3}$ for the image ``monarch" and to $\tau = 2.5\times 10^{-3}$ for the image ``Barbara". Fig. \ref{images_TV} shows the SR results obtained using the bicubic interpolation  (third column), ADMM based algorithm of \cite{MNg2010SR_TV}  (fourth column) and the proposed Algo. 4 (last column). As expected, the ADMM reconstructions perform much better than a simple interpolation of the LR image that is not able to solve the upsampling and deblurring problem. \AB{The results obtained with the proposed algorithm and with the method of \cite{MNg2010SR_TV} are visually very similar}. This visual inspection is  confirmed by the quantitative results provided in Table \ref{tab_TV_metrics}. However, the proposed algorithm has the advantage of being much faster than the algorithm of \cite{MNg2010SR_TV} \AB{(with computational times reduced by a factor larger than $2$)}. Moreover, Fig. \ref{curves_TV} illustrates the convergence of the two algorithms. The proposed single image SR algorithm (Algo. 4) converges faster and with less fluctuations than the algorithm of \cite{MNg2010SR_TV}. This result can be explained by the fact that the algorithme in \cite{MNg2010SR_TV} requires to handle more variables in the ADMM scheme than the proposed algorithm.

\begin{table*}
\begin{center}
\caption{SR of the Monarch, Lena and Barbara images when considering a TV-regularization: quantitative results.}
\label{tab_TV_metrics}
\begin{tabular}{|c|c|c|c|c|c|c|}
\hline
Image & Method & PSNR (dB)& ISNR (dB)  & MSSIM  & Time (s) & Iter.\\
\hline
\multirow{3}{*}{Monarch}
& Bicubic             & 23.11 & -     & 0.75 & 0.002 & -\\
& ADMM \cite{MNg2010SR_TV} & 29.49 & 6.38   & 0.84 & 78.95 & 812 \\
&  Algo. 4            & 29.38 & 6.28  & 0.83 & \textbf{19.81} & \textbf{170}\\
\hline
\multirow{3}{*}{Lena}
& Bicubic             & 25.80 & -     & 0.57 & 0.002 & -\\
& ADMM \cite{MNg2010SR_TV} & 30.81 & 5.00   & 0.66 & 35.67 & 372 \\
&  Algo. 4             & 30.91 & 5.11   & 0.66 & \textbf{20.63} & \textbf{164}\\
\hline
\multirow{3}{*}{Barbara}
& Bicubic             & 22.71 & -     & 0.48 & 0.002 & -\\
& ADMM \cite{MNg2010SR_TV} & 24.80 & 2.09  & 0.56 & 13.85 & 148 \\
&  Algo. 4             & 24.84 & 2.13  & 0.56 & \textbf{8.36}& \textbf{73} \\
\hline
\end{tabular}
\end{center}
\end{table*}

\begin{figure}
\settoheight{\tempdima}{\includegraphics[width=.3\linewidth]{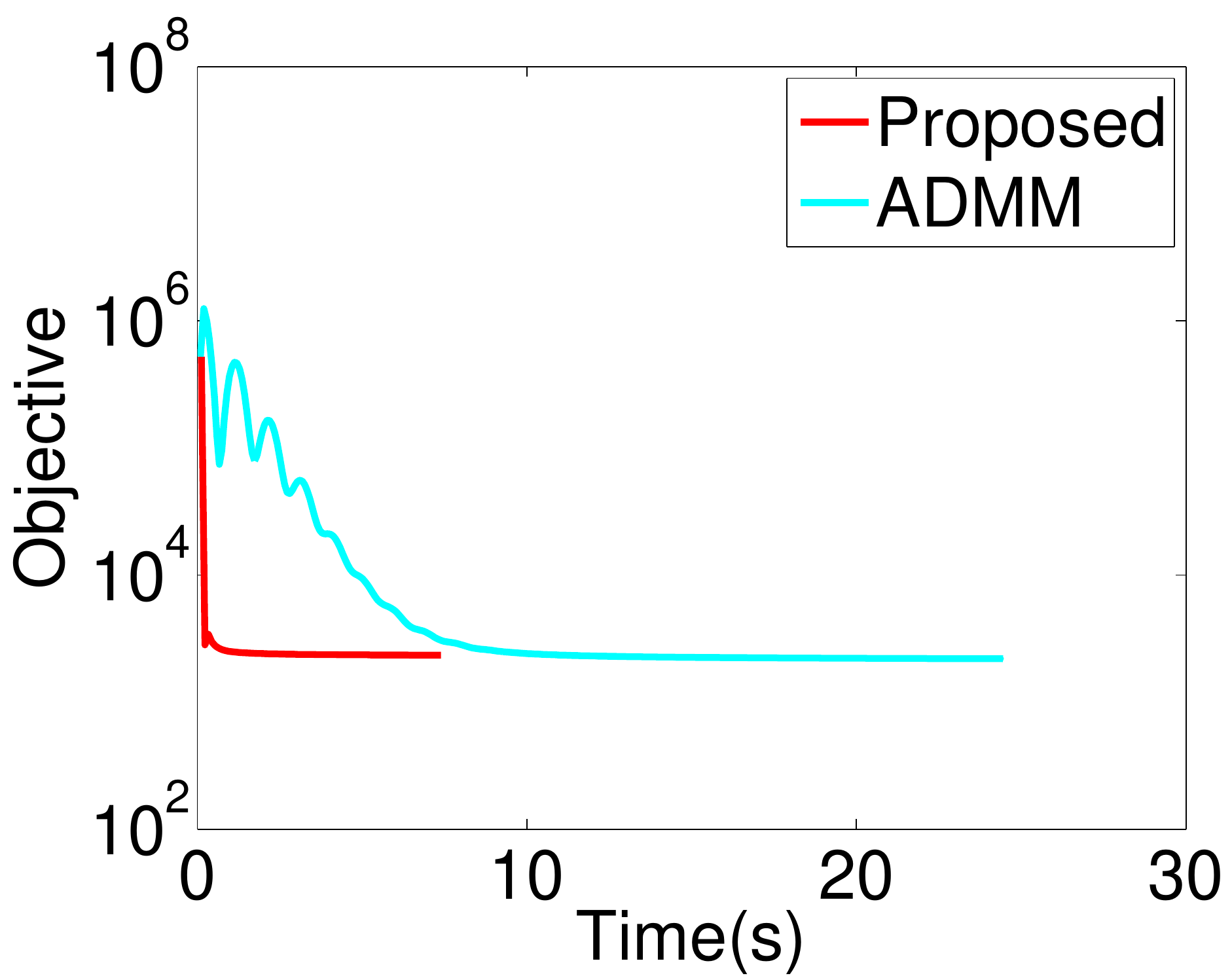}}%
\centering\begin{tabular}{@{}c@{}c@{}c@{}}
 &\textbf{Objective}& \textbf{ISNR} \\
 \rowname{Monarch}&
\includegraphics[width=0.3\linewidth]{figures/ex2_tv/obj_time_monarch_tv} &\label{fig_CostCPU_monarch_TV}
\includegraphics[width=0.3\linewidth]{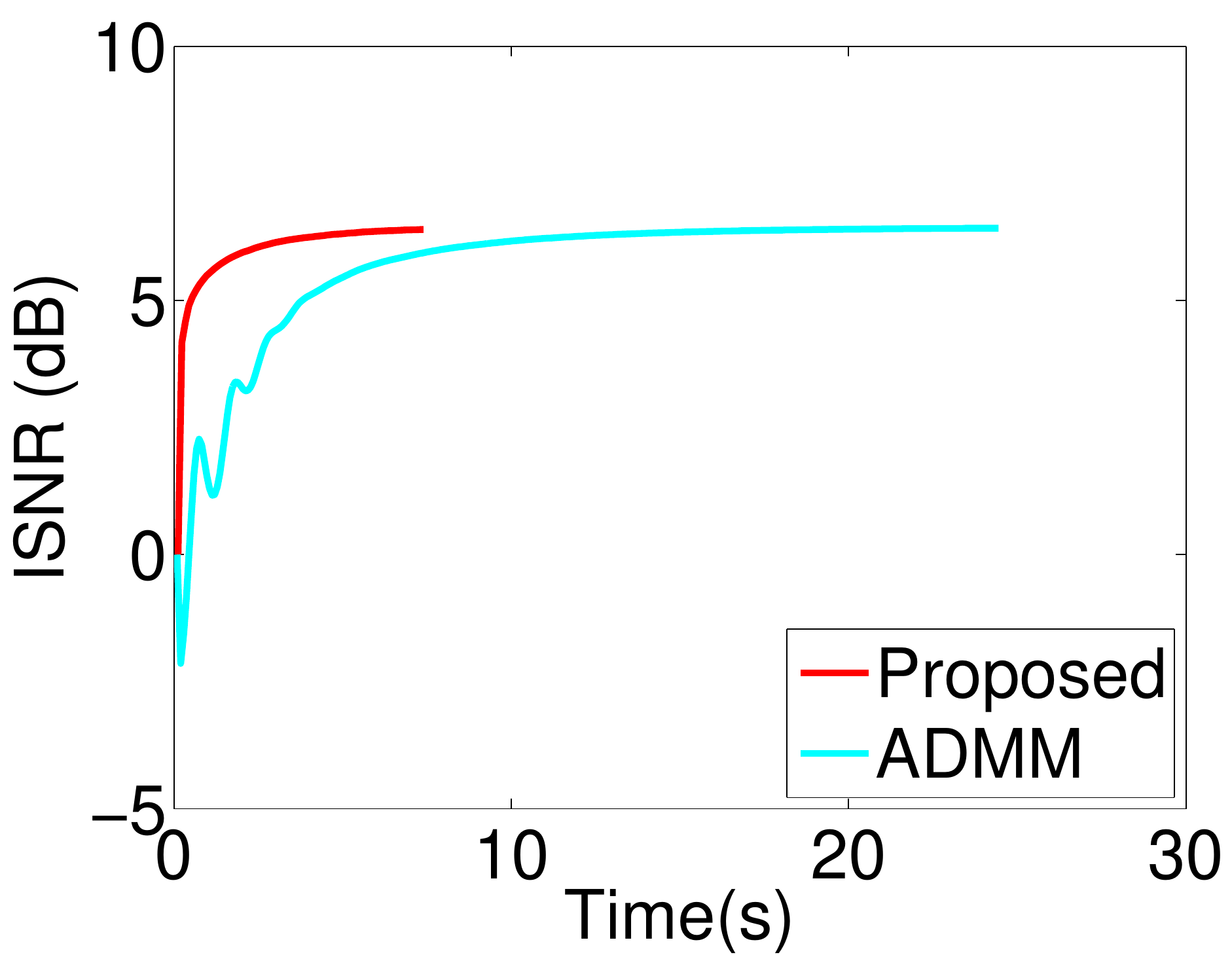} \label{fig_ISNRCPU_monarch_TV}  \\
\rowname{Lena}&
\includegraphics[width=0.3\linewidth]{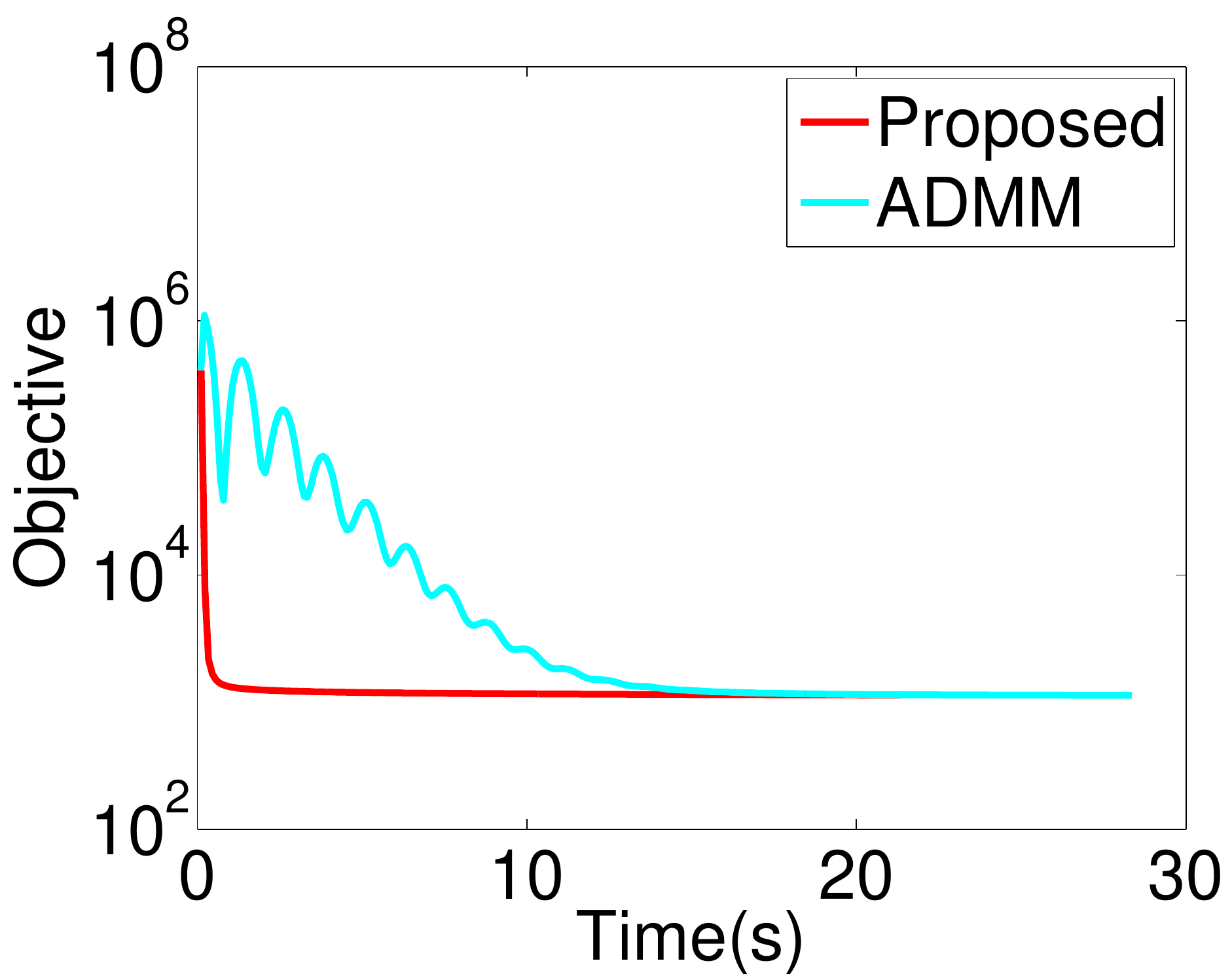} &\label{fig_CostCPU_Lena_TV}
\includegraphics[width=0.3\linewidth]{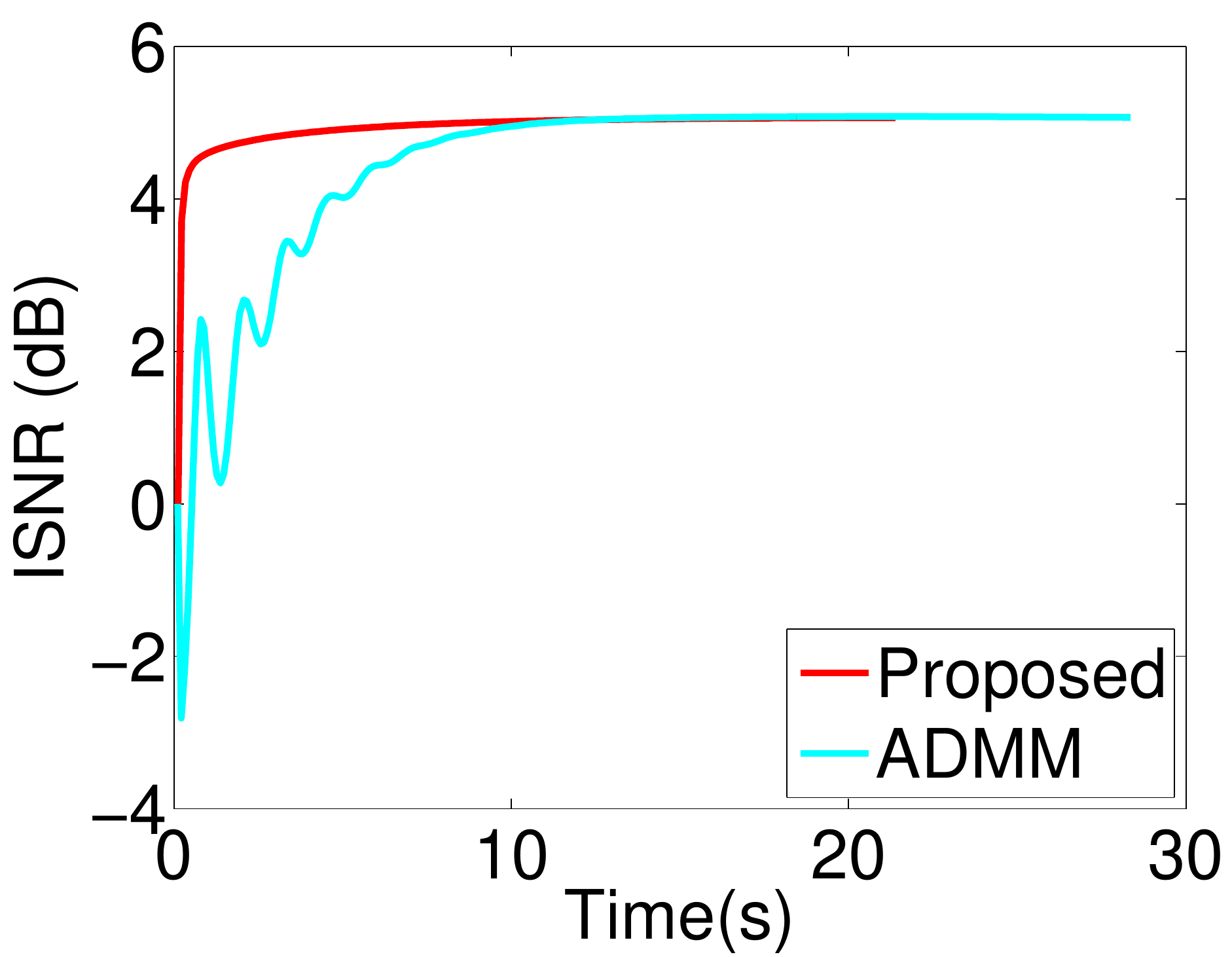} \label{fig_ISNRCPU_Lena_TV}  \\
\rowname{Barbara}&
\includegraphics[width=0.3\linewidth]{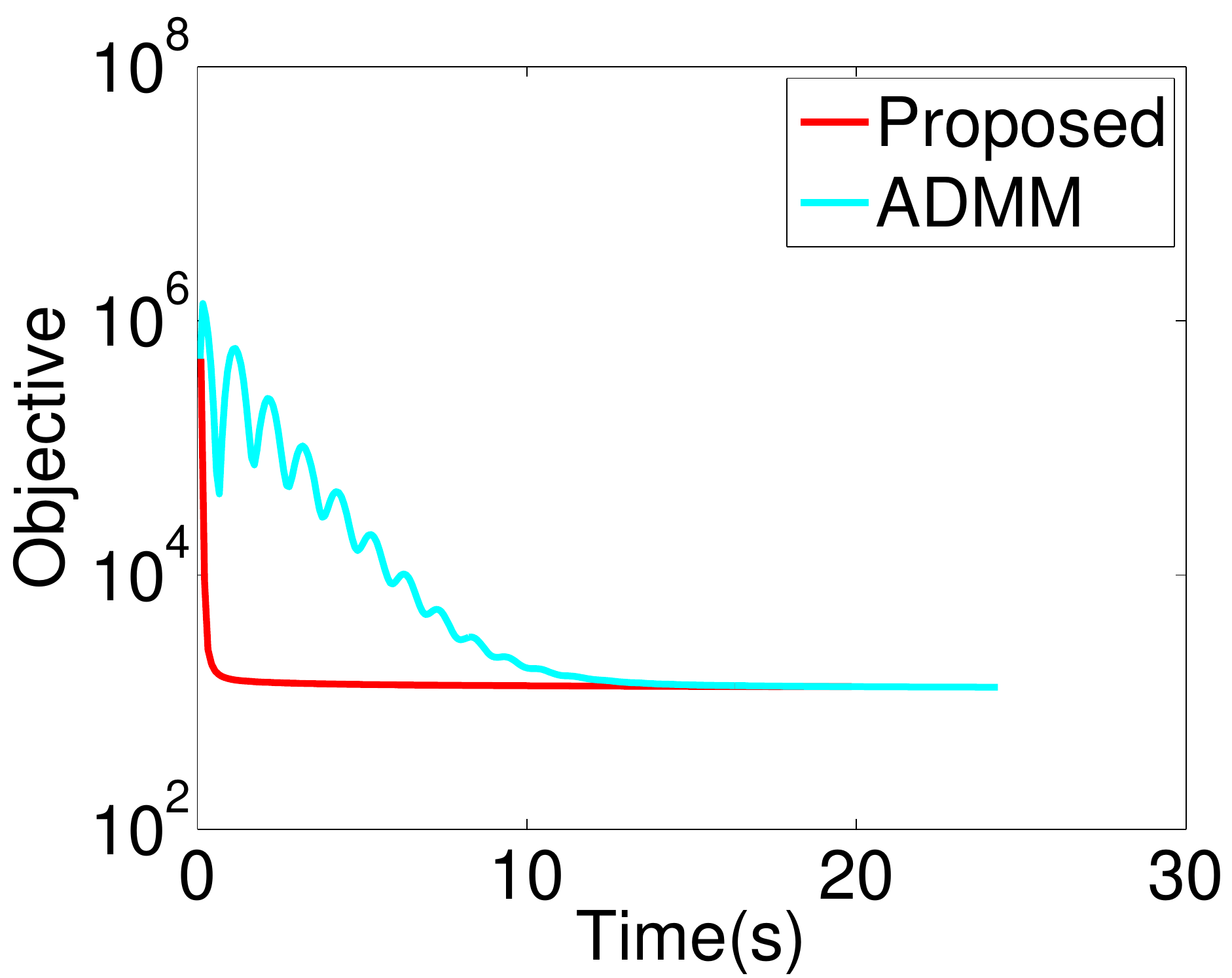} &\label{fig_CostCPU_Barbara_TV}
\includegraphics[width=0.3\linewidth]{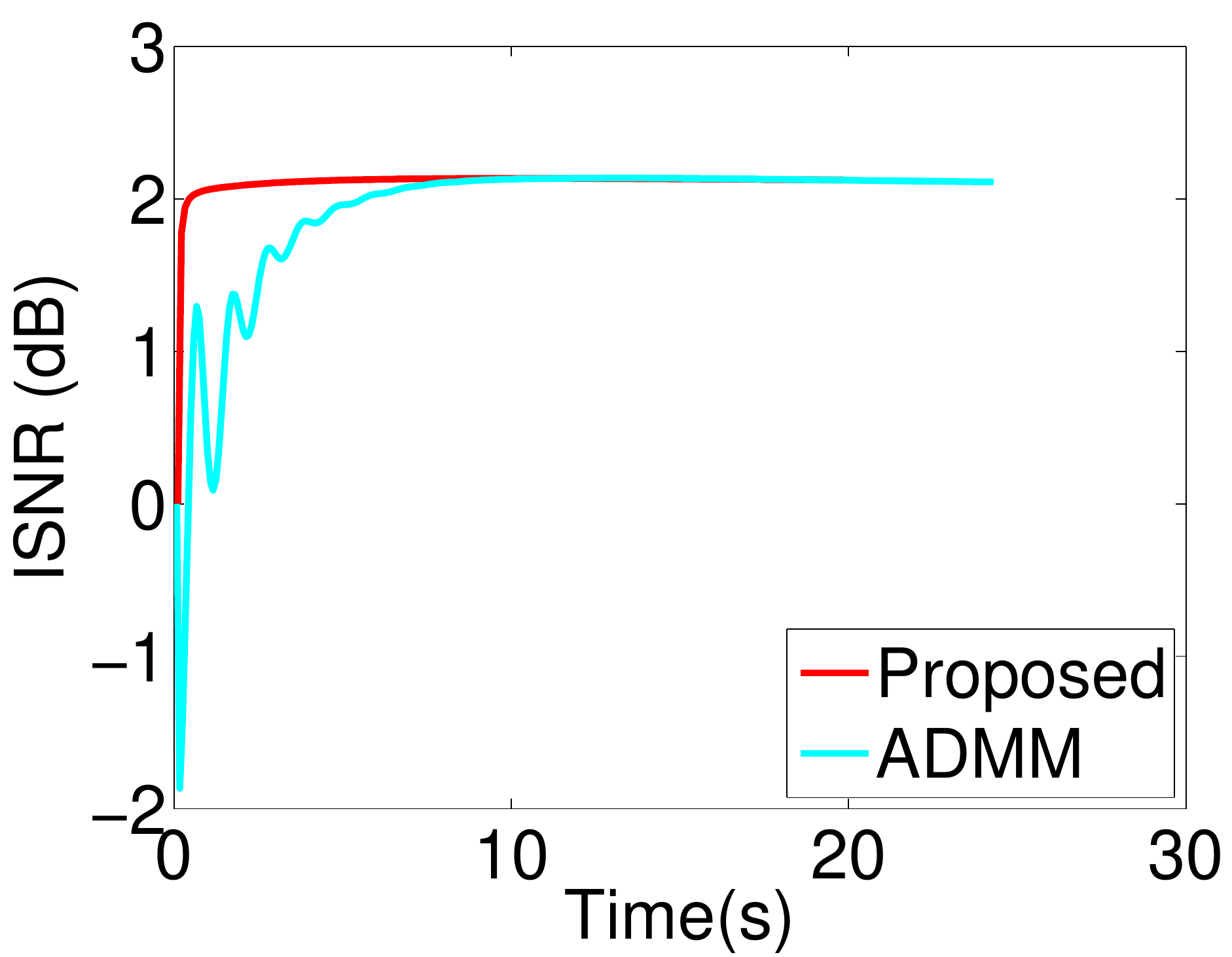} \label{fig_ISNRCPU_Barbara_TV}
\end{tabular}
\caption{SR of the Monarch, Lena and Barbara images when considering a TV-regularization: objective function (left) and ISNR (right) vs time.}
\label{curves_TV}
\end{figure}

\subsubsection{$\ell_1$-norm regularization in the wavelet domain}
This section evaluates the performance of Algo. 5, which is compared with a generalization of the method proposed in \cite{MNg2010SR_TV} to an $\ell_1$-norm regularization in the wavelet domain. The motivations for working in the wavelet domain are essentially to take advantage of the sparsity of the wavelet coefficients. All experiments were conducted using the discrete Haar wavelet transform and the Rice wavelet toolbox \cite{rwt_toolbox}. For both implementations, the regularization parameter was adjusted by cross validation, leading to $\tau = 2\times 10^{-4}$ for the image ``Lena", $\tau =  1.8\times 10^{-4}$ for the image ``Monarch" and $\tau = 2.5\times 10^{-4}$ for the image ``Barbara".

\begin{figure}[h!]
\centering\begin{tabular}{@{}c@{}c@{}}
 ADMM \cite{MNg2010SR_TV} & Algo. 5 \\
\includegraphics[width=0.25\linewidth]{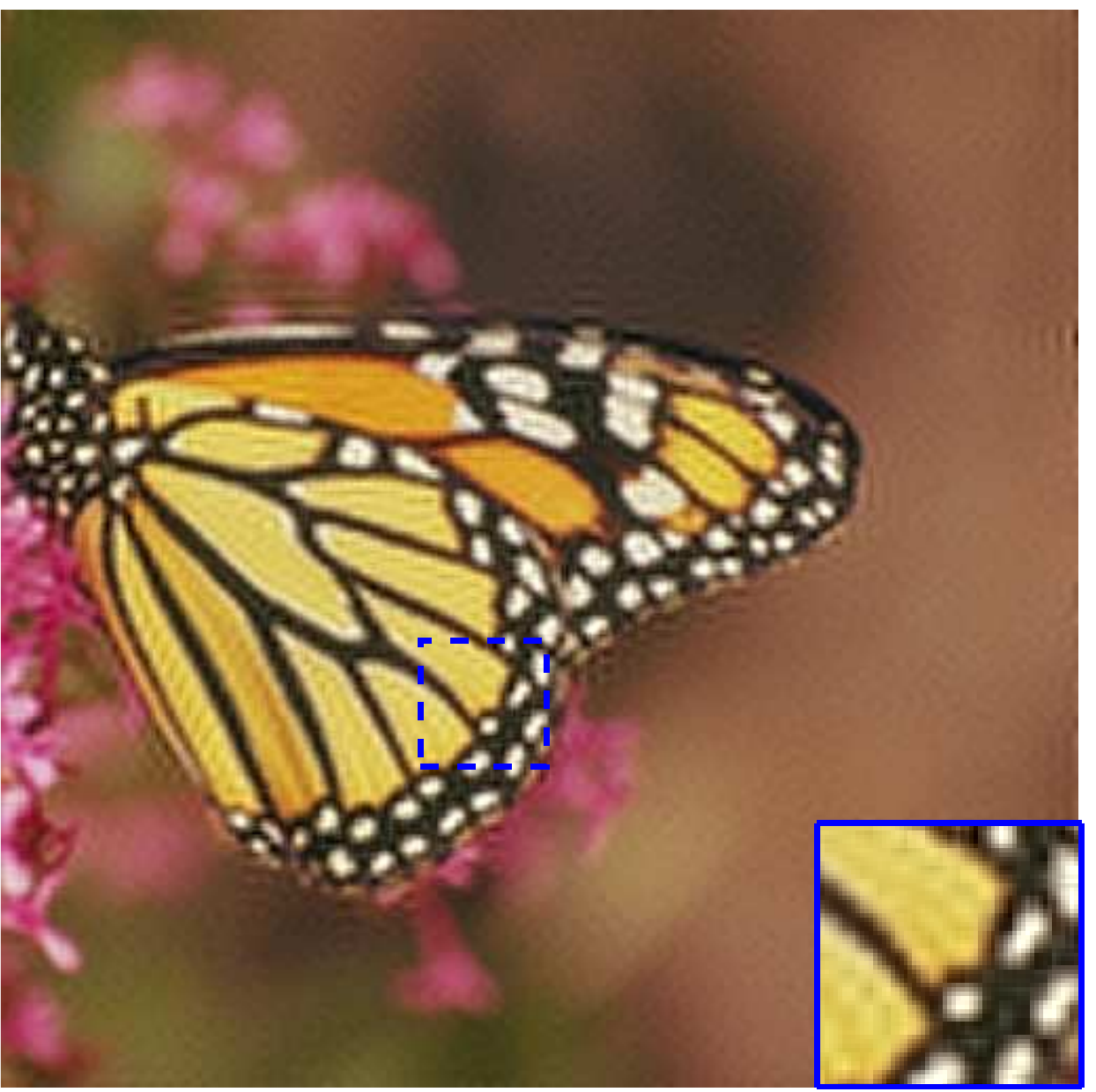} \label{fig_direct_monarch_l1} &
\includegraphics[width=0.25\linewidth]{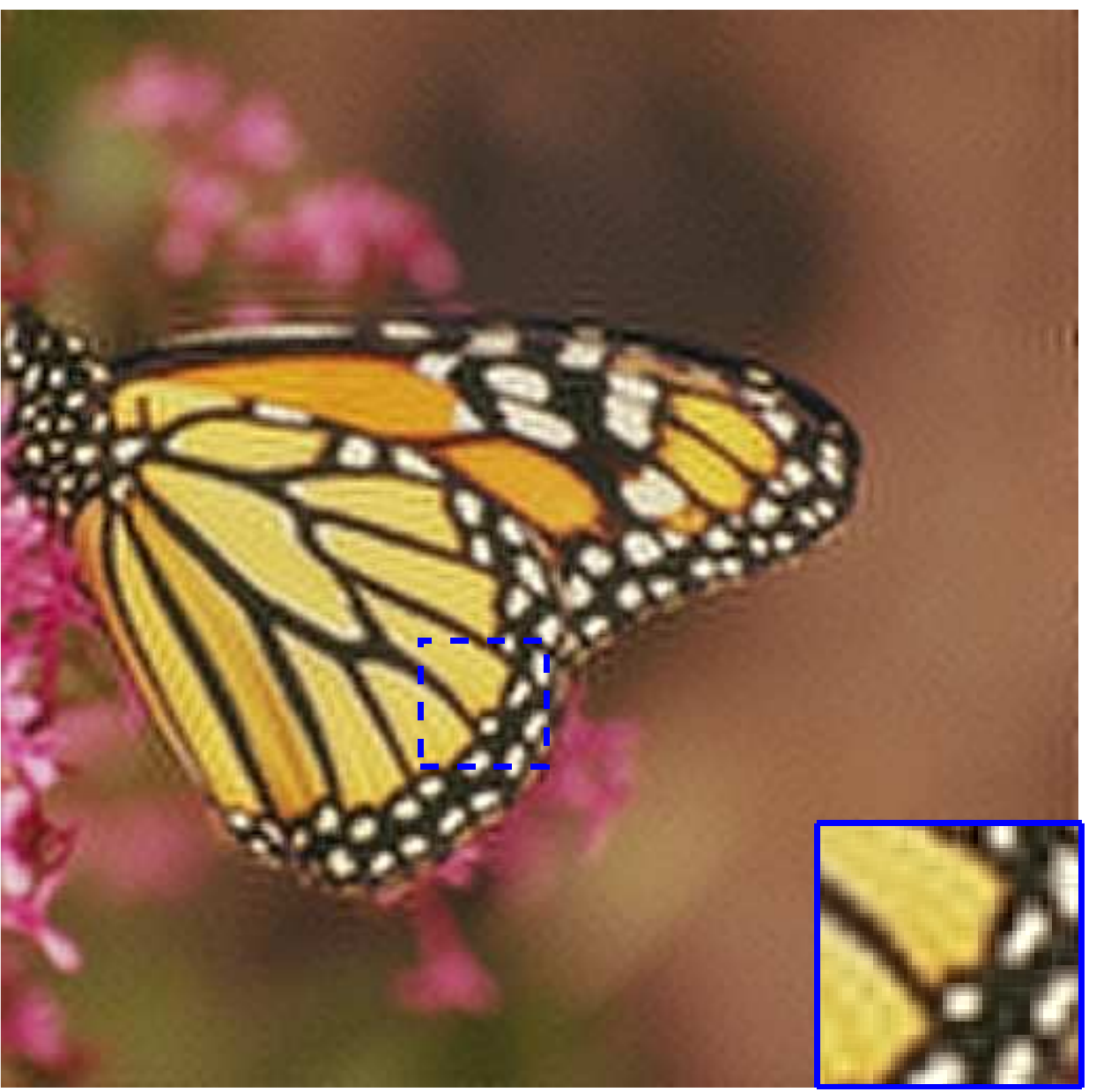} \label{fig_FSR_monarch_l1}\\
\includegraphics[width=0.25\linewidth]{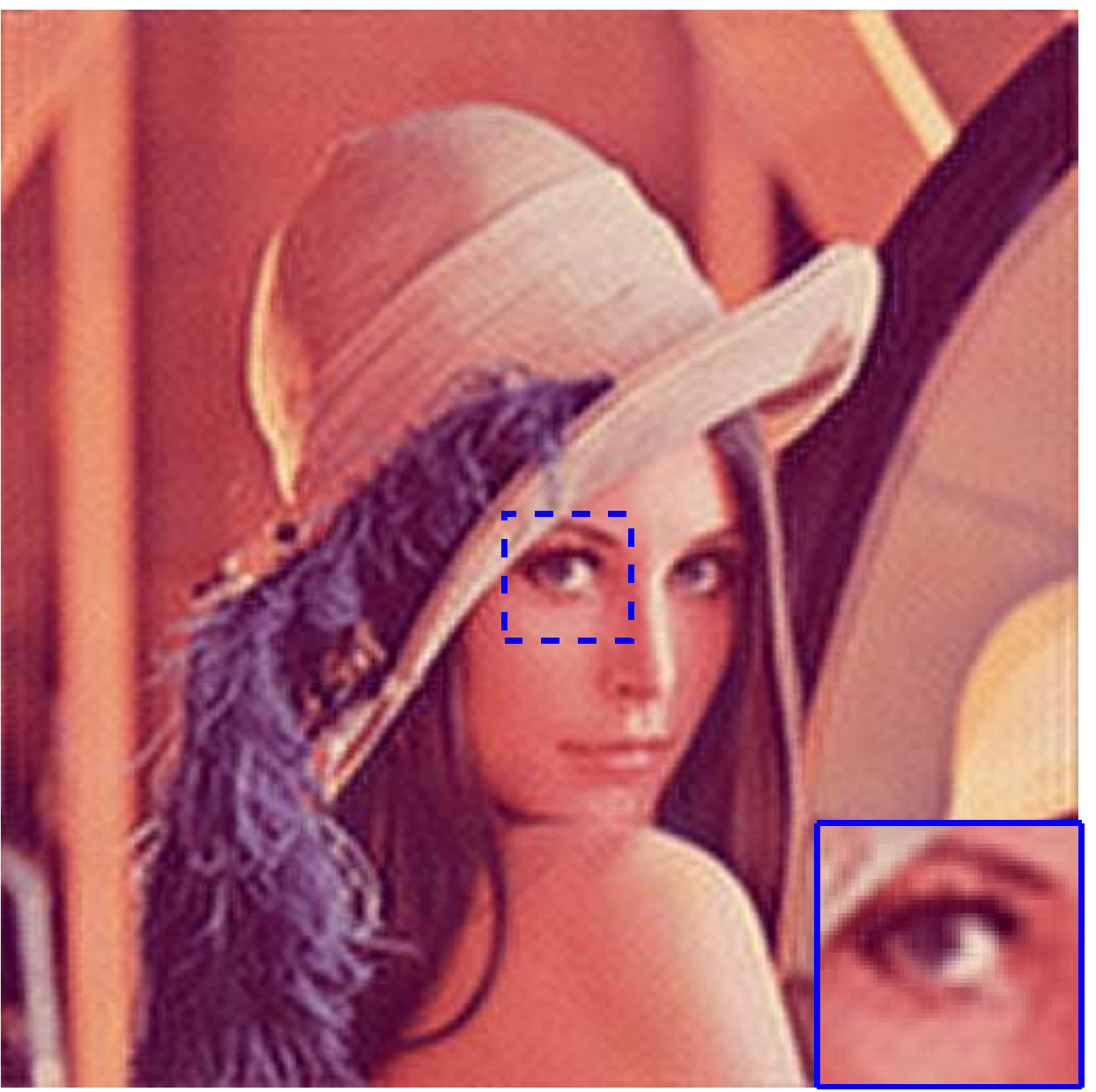} \label{fig_direct_lena_l1} &
\includegraphics[width=0.25\linewidth]{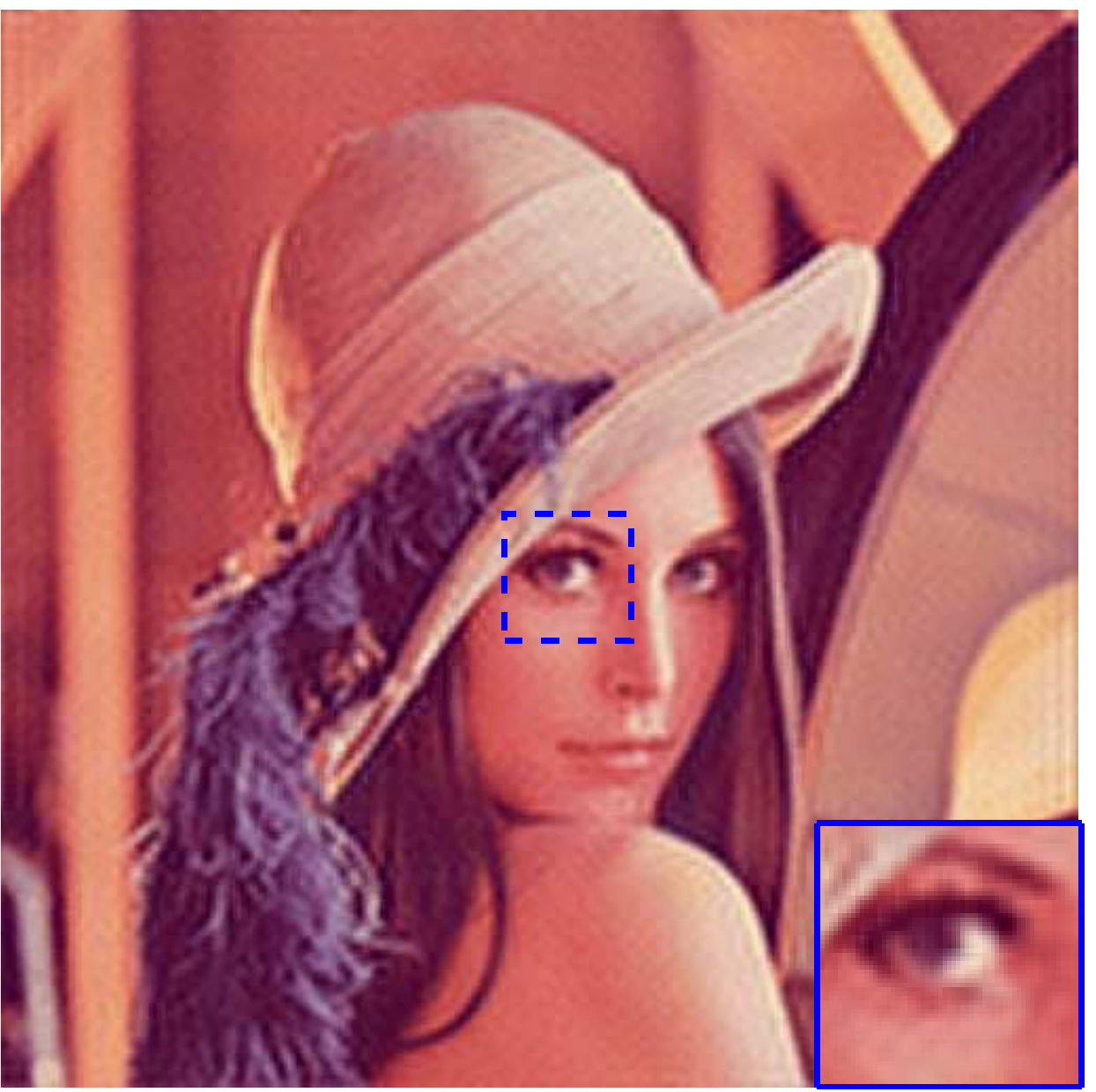} \label{fig_FSR_lena_l1}\\
\includegraphics[width=0.25\linewidth]{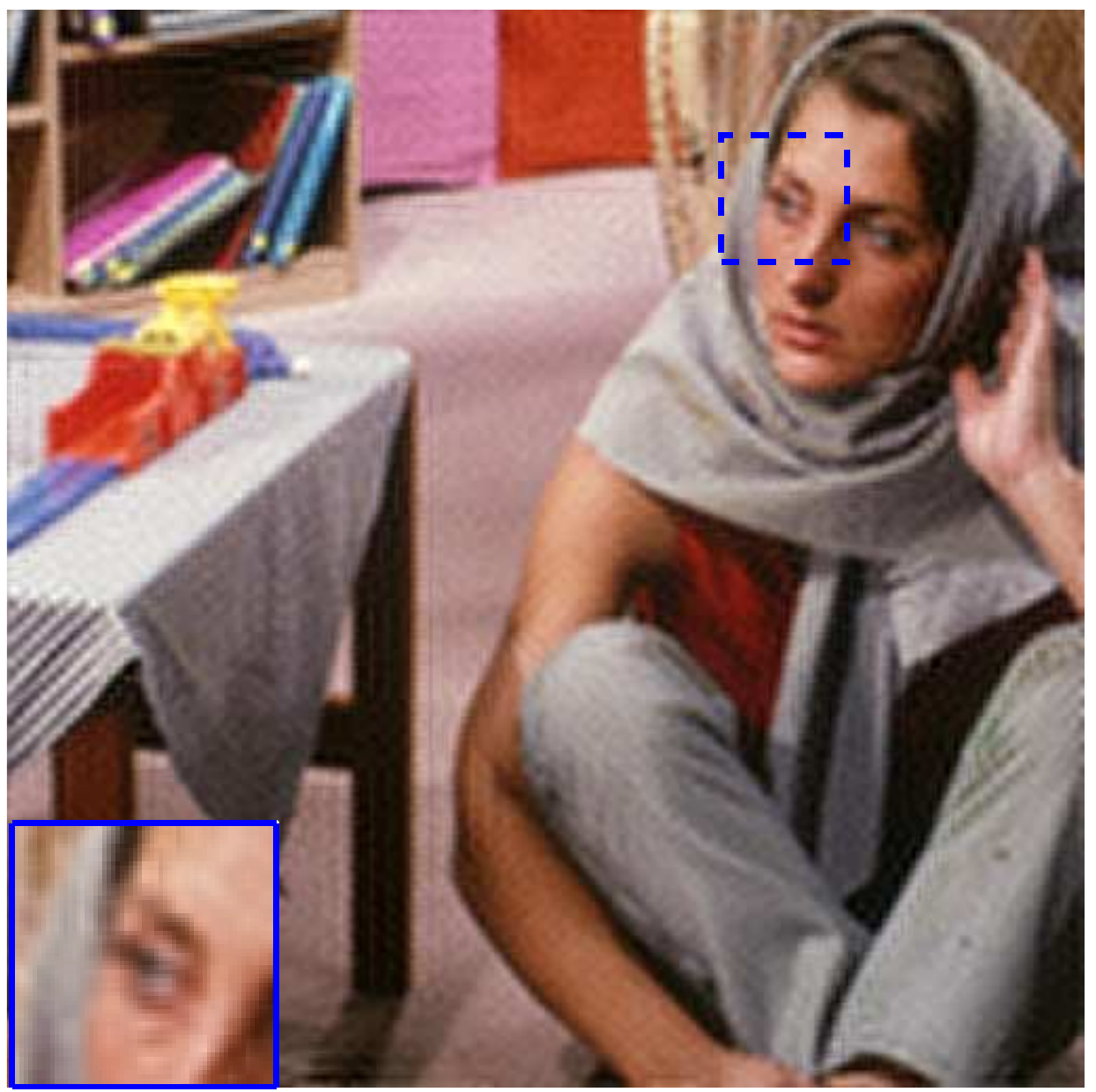} \label{fig_direct_barbara_l1} &
\includegraphics[width=0.25\linewidth]{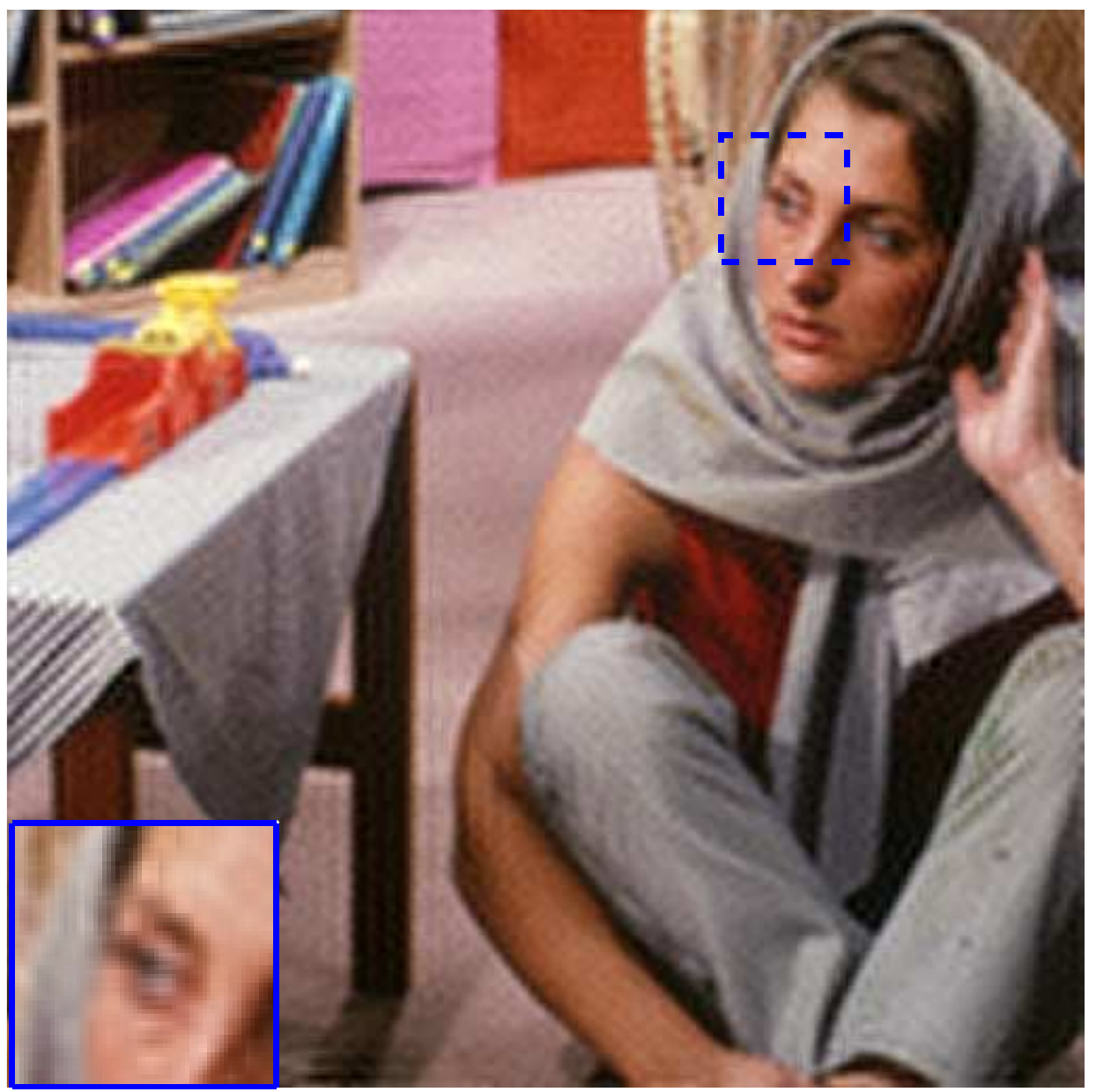} \label{fig_FSR_barbara_l1}
\end{tabular}
\caption{SR of the Monarch, Lena and Barbara images when considering a $\ell_1 $-norm regularization in the wavelet domain: visual results.}
\label{images_l1}
\end{figure}

Fig. \ref{images_l1} shows the SR reconstruction results with an $\ell_1$-norm minimization in the wavelet domain.
The HR images obtained with Algo. 5 and with the algorithm of \cite{MNg2010SR_TV} adapted to the $\ell_1$-norm prior are visually similar and better than a simple interpolation. The numerical results shown in Table \ref{tab_l1_metrics} confirm that the two algorithms provide similar reconstruction performance. However, as in the previous case (TV regularization), the proposed algorithm is characterized by much smaller computational times than the standard ADMM implementation. The faster and smoother convergence obtained with the proposed method (Algo. 5) can be observed in Fig. \ref{curves_l1}. Note that the fluctuations of the objective function and PSNR values (versus the number of iterations) obtained with the method of \cite{MNg2010SR_TV} are due to the variable splitting, which requires more variables and constraints \nd{to be handled} than for the proposed method.

\begin{table*}
\begin{center}
\caption{SR of the Monarch, Lena and Barbara images when considering a $\ell_1 $-norm regularization in the wavelet domain: quantitative results.}
\label{tab_l1_metrics}
\begin{tabular}{|c|c|c|c|c|c|c|}
\hline
Image & Method & PSNR (dB) & ISNR (dB) & MSSIM  & Time (sec.) &Iter.\\
\hline
\multirow{3}{*}{Monarch}
& Bicubic           & 23.11 & -     & 0.75 & 0.002 & -\\
&ADMM \cite{MNg2010SR_TV}$^\ast$ &27.08 & 3.97 & 0.74 & 34.08 & 400 \\
& Algorithm 5          & 27.13 & 4.03 & 0.74 & \textbf{15.02} & \textbf{177}\\
\hline
\multirow{3}{*}{Lena}
& Bicubic            & 25.80 & -    & 0.57 & 0.002 & -\\
&ADMM \cite{MNg2010SR_TV} & 30.09 & 4.29   & 0.62 & 38.48 & 450\\
& Algorithm 5            & 30.21 & 4.41  & 0.63 & \textbf{14.25}& \textbf{164} \\
\hline
\multirow{3}{*}{Barbara}
& Bicubic            & 22.71 & -     & 0.48 & 0.002 & -\\
&ADMM \cite{MNg2010SR_TV} & 24.66 & 1.95  & 0.52 & 34.13 & 400 \\
& Algorithm 5            & 24.70 & 2.00  & 0.53 & \textbf{14.83}& \textbf{171} \\
\hline
\end{tabular}
\end{center}
{\footnotesize \rule{0in}{1.2em}$^\ast$\scriptsize The algorithm of \cite{MNg2010SR_TV} was originally proposed for SR using a TV regularization. This algorithm has been modified by the authors to solve the $\ell_1$-norm penalized optimization problem.}
\end{table*}

\begin{figure}
\settoheight{\tempdima}{\includegraphics[width=.3\linewidth]{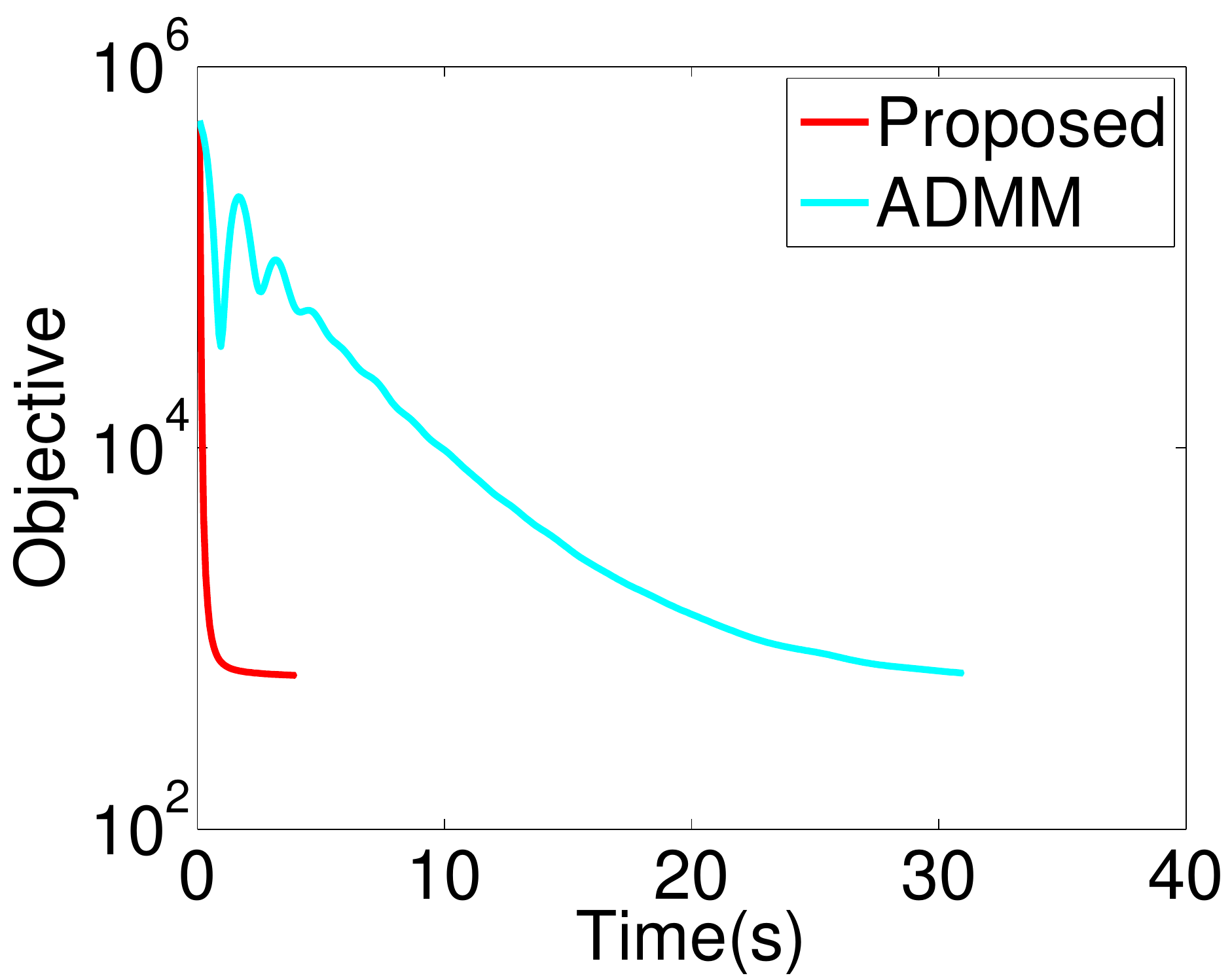}}%
\centering\begin{tabular}{@{}c@{}c@{}c@{}}
 &\textbf{Objective}& \textbf{\AB{PSNR}} \\
  \rowname{Monarch}&
\includegraphics[width=0.3\linewidth]{figures/ex2_l1/obj_time_monarch_l1} &\label{fig_CostCPU_monarch_l1}
\includegraphics[width=0.3\linewidth]{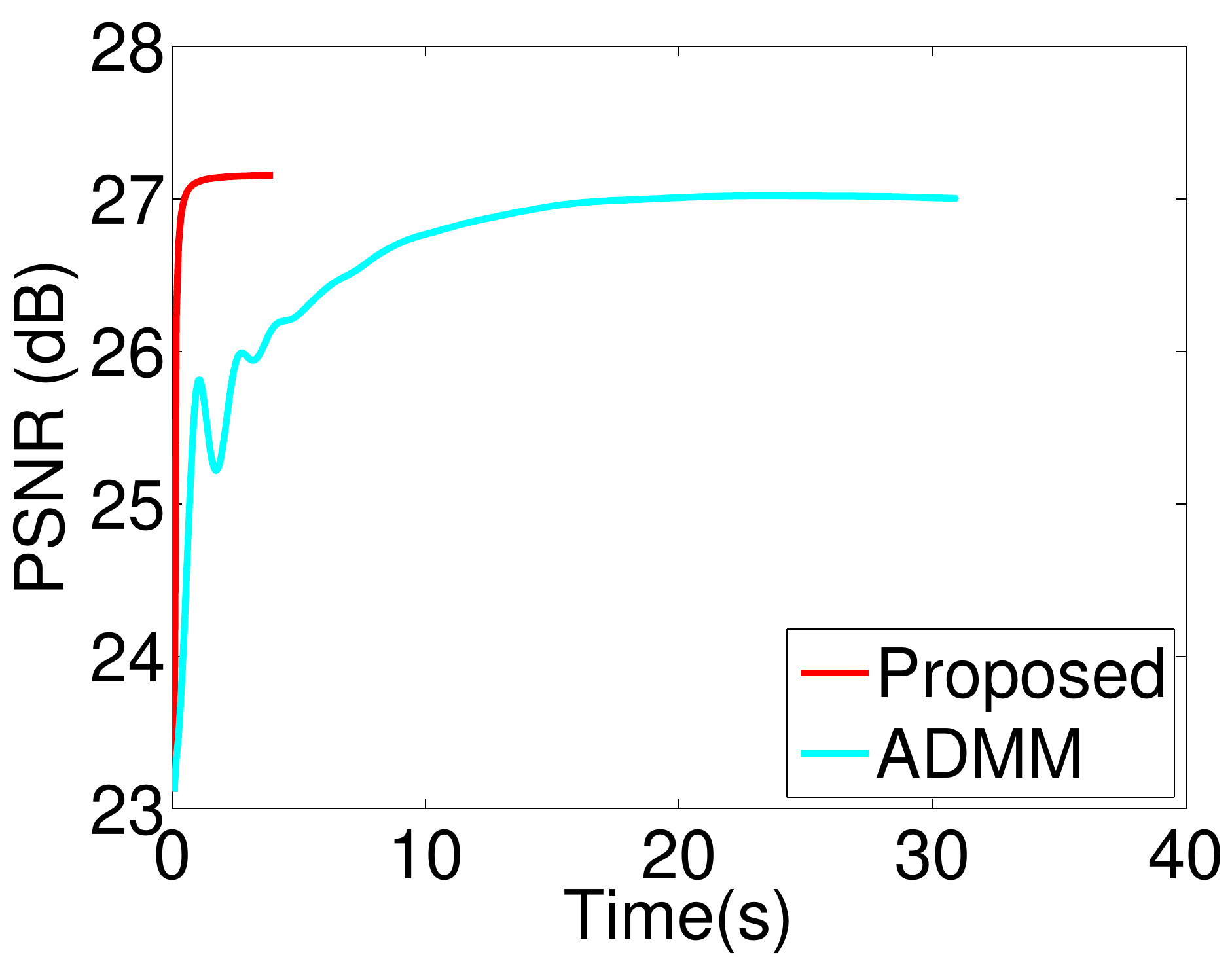} \label{fig_ISNRCPU_monarch_l1}  \\
 \rowname{Lena}&
\includegraphics[width=0.3\linewidth]{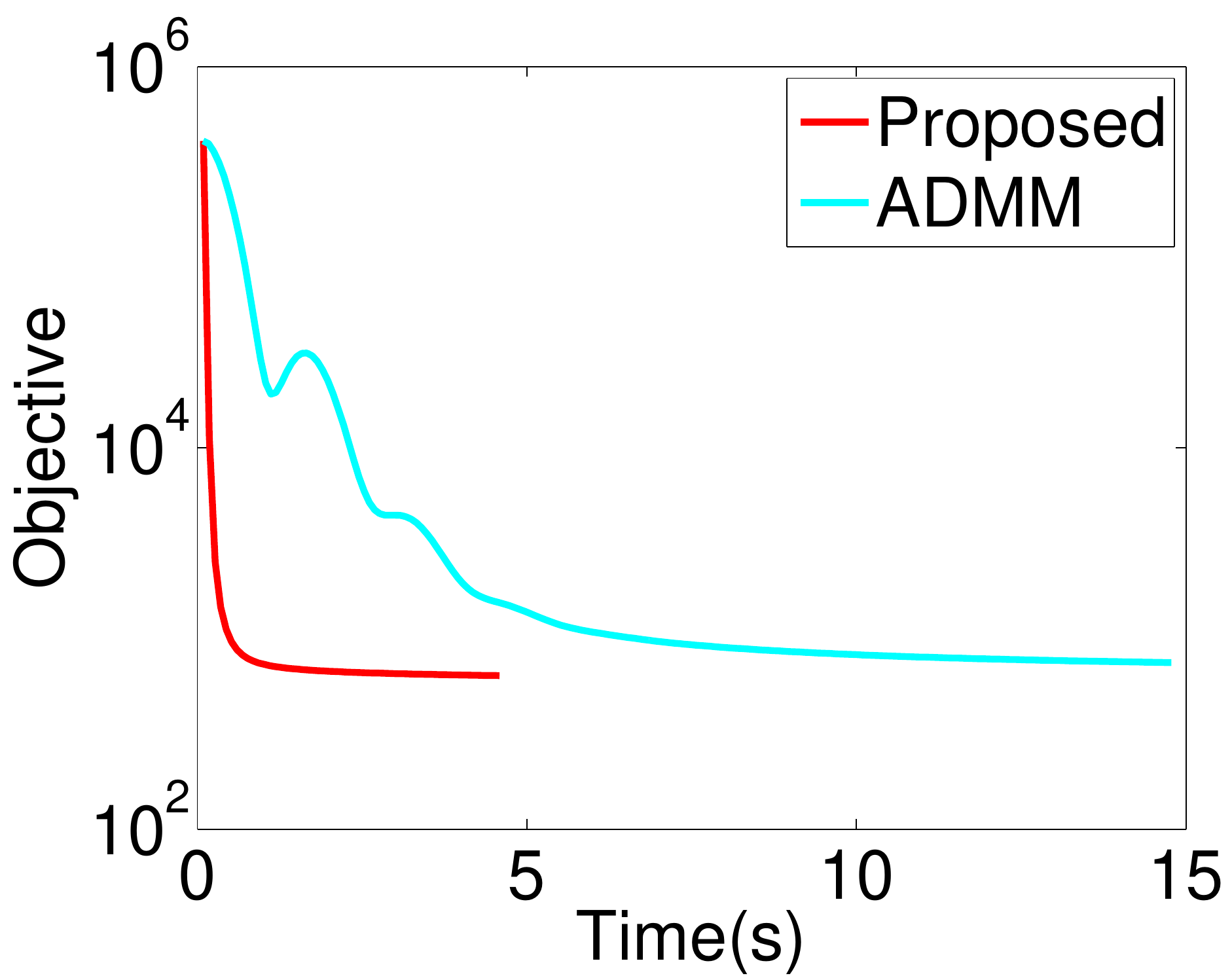} &\label{fig_CostCPU_Lena_l1}
\includegraphics[width=0.3\linewidth]{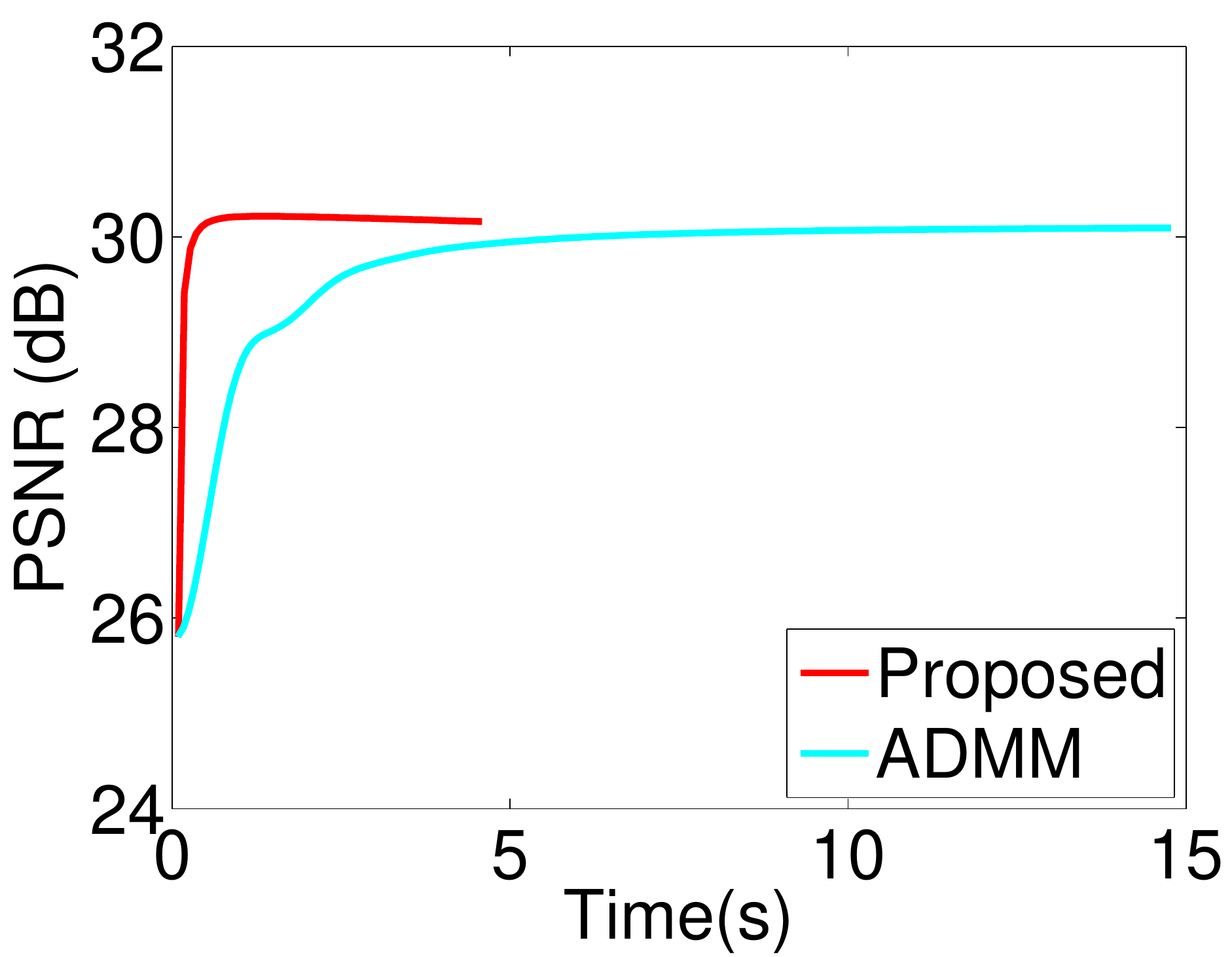} \label{fig_ISNRCPU_Lena_l1}  \\
 \rowname{Barbara}&
\includegraphics[width=0.3\linewidth]{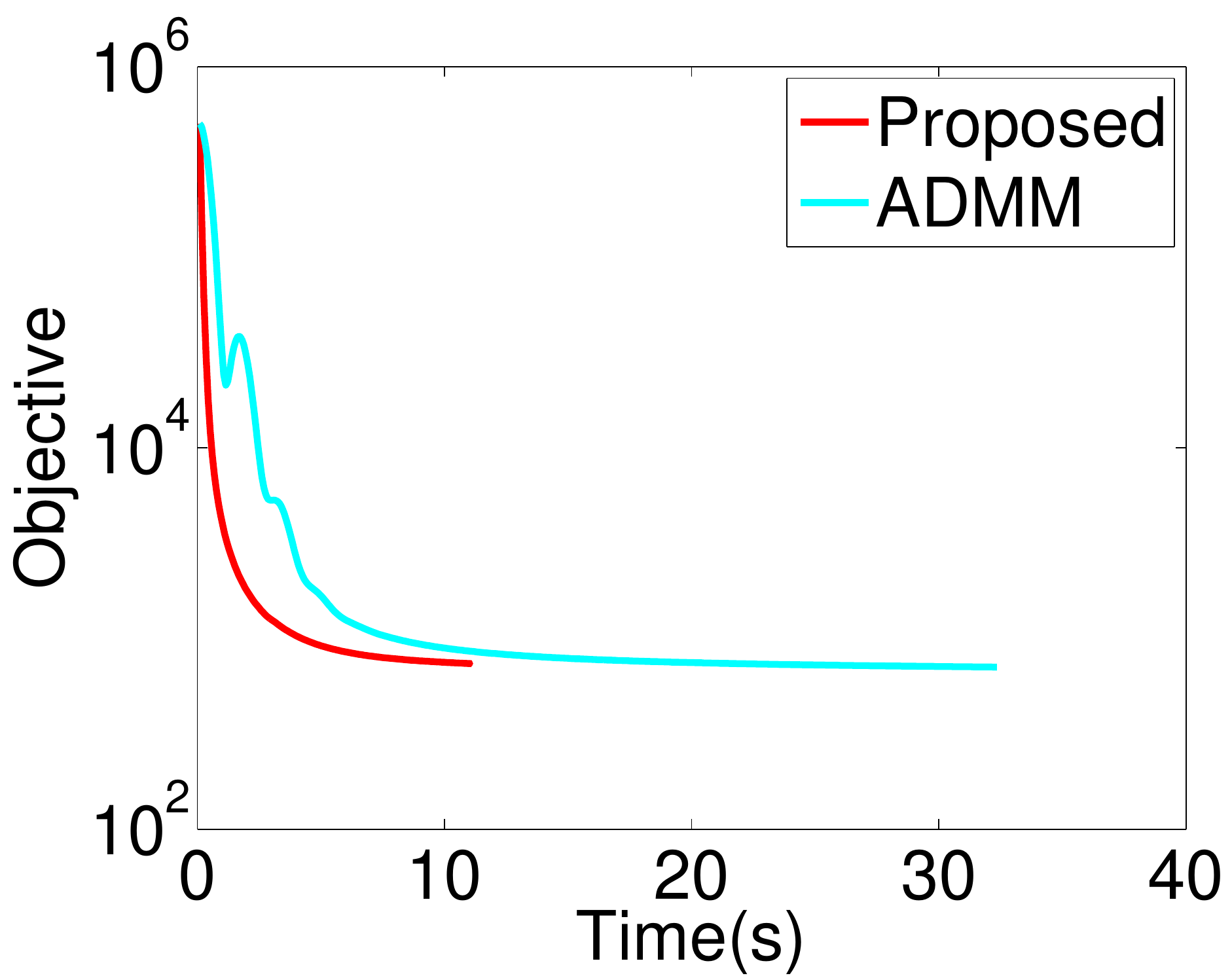} &\label{fig_CostCPU_Barbara_l1}
\includegraphics[width=0.3\linewidth]{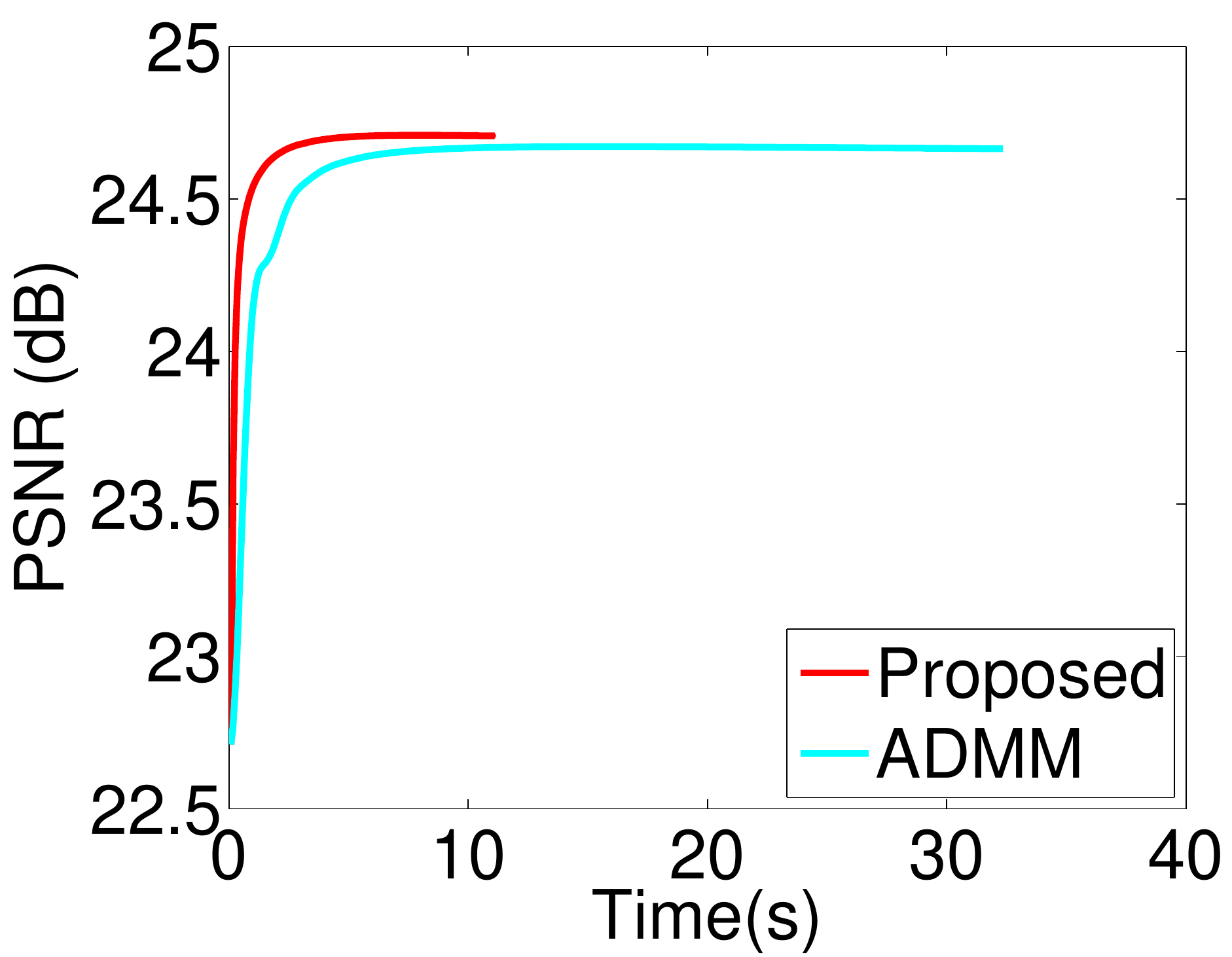} \label{fig_ISNRCPU_Barbara_l1}
\end{tabular}
\caption{SR of the Monarch, Lena and Barbara images when considering an $\ell_1 $-norm regularization in the wavelet domain: objective function (left) and PSNR (right) vs time.}
\label{curves_l1}
\end{figure}


\section{Conclusion and perspectives}
This paper \nd{studied} a new fast single image super-resolution framework based on the widely used image formation model. The proposed super-resolution approach \nd{computed the super-resolved image efficiently by exploiting intrinsic properties of the decimation and the blurring operators in the frequency domain}. A large variety of priors was shown to be able to be handled in the proposed super-resolution scheme. Specifically, \nd{when considering an $\ell_2$-regularization, the target image was computed analytically,} getting rid of any iterative steps. \nd{For more complex priors (i.e., non $\ell_2$-regularization),} variable splitting allowed this analytical solution to be embedded into the \AB{augmented Lagrangian} framework, thus accelerating \nd{various existing} schemes for single image super-resolution. Results on several natural images \nd{confirmed} the computational efficiency of the proposed approach and shown its fast and smooth convergence. \AB{As a perspective of this work, an interesting research track consists of extending the proposed method to some online applications such as video super-resolution and medical imaging, to evaluate its robustness to non-Gaussian noise and to extend it to semi-blind or blind deconvolution or multi-frame super-resolution. Considering a more practical case for super-resolving real images compressed by JPEG or JPEG-2000 algorithms will also be interesting and deserves further exploration.}

\appendices
\section{\AB{Derivation of the analytical solution \eqref{eq_anas_gl2}}}
\label{app:theorem1}
\AB{The computational details for obtaining the result in \eqref{eq_anas_gl2} from \eqref{l2_anas_generic} are summarized hereinafter. First, denoting $\bfr = \bfH^H\bfS^H\bfy +2\tau\bfA^H\bfv$, the solution \eqref{l2_anas_generic} is
\begin{align}
\hat{\bfx} &= (\bfH^H\underline{\bfS}\bfH +2\tau \bfA^H\bfA)^{-1}\bfr \nonumber\\
& = \bfF^H \left(\bsLambda^H\bfF\underline{\bfS}\bfF^H\bsLambda + 2\tau \Qi{\bfF\bfA^H\bfA\bfF^H} \right)^{-1}\bfF \bfr. \label{eq_b1}
\end{align}
Based on Lemma \ref{lemma:1}, $\bsLambda^H\bfF\underline{\bfS}\bfF^H\bsLambda$ is computed as
\begin{eqnarray}
&&\bsLambda^H\bfF\underline{\bfS}\bfF^H\bsLambda  \notag \\
&=&\frac{1}{d}\bsLambda^H \left(\bfJ_d \otimes \bfI_{N_l}\right) \bsLambda  \label{eq_lemma1_a} \\
&=&  \frac{1}{d}\bsLambda^H \left(\left(\textbf{1}_d \textbf{1}_d^T\right) \otimes \left(\bfI_{N_l}\bfI_{N_l}\right)\right) \bsLambda \label{eq_lemma1_b}  \\
&=&\frac{1}{d}\bsLambda^H \left(\textbf{1}_d\otimes \bfI_{N_l}\right)\left(\textbf{1}_d^T \otimes \bfI_{N_l}\right) \bsLambda \label{eq_lemma1_c}  \\
&=& \frac{1}{d}\left(\bsLambda^H [\underbrace{\bfI_{N_l},\cdots ,\bfI_{N_l}}_d]^T\right) \left( [\underbrace{\bfI_{N_l},\cdots ,\bfI_{N_l}}_d]\bsLambda \right)  \label{eq_lemma1_d} \\
&=&\frac{1}{d} \EigSq^H \EigSq.  \label{eq_lemma1}
\end{eqnarray}
Note that \eqref{eq_lemma1_b} was obtained from \eqref{eq_lemma1_a} by replacing $\bfJ_d$ by $\textbf{1}_d \textbf{1}_d^T$,
where $\textbf{1}_d \in \mathbb{R}^{d \times 1}$ is a vector of ones. Obtaining \eqref{eq_lemma1_c} from \eqref{eq_lemma1_b} is straightforward using the following property of the Kronecker product $\otimes$
\begin{equation*}
\mathcal{A}\mathcal{B} \otimes \mathcal{C}\mathcal{D} = (\mathcal{A}\otimes \mathcal{C})(\mathcal{B} \otimes \mathcal{D}).
\label{eq^{k}ronecker}
\end{equation*}
In \eqref{eq_lemma1_d}, $\bsLambda \in \mathbb{R}^{N_h\times N_h}$ whereas $\left[\bfI_{N_l},\cdots ,\bfI_{N_l}\right] \in \mathbb{R}^{N_l\times N_h}$ and $\left[\bfI_{N_l},\cdots ,\bfI_{N_l}\right]^T \in \mathbb{R}^{N_h\times N_l}$ are block matrices whose blocks are equal to the identity matrix $\bfI_{N_l}$.
The matrix $\EigSq \in \mathbb{R}^{N_l\times N_h}$ in \eqref{eq_lemma1} is given by
\begin{align}
\EigSq
& = [\bfI_{N_l},\cdots ,\bfI_{N_l}] \bsLambda  \notag \\
& = [\bfI_{N_l},\cdots ,\bfI_{N_l}] \diag{\bsLambda_1,\cdots,\bsLambda_d} \notag \\
& = [\bfI_{N_l},\cdots ,\bfI_{N_l}]
	\begin{bmatrix}
       \bsLambda_1 & \cdots &0           \\[0.3em]
        \vdots&  \ddots &\vdots     \\[0.3em]
        0  & \cdots        & \bsLambda_d
     \end{bmatrix} \\
&=[\bsLambda_1,\bsLambda_2, \cdots,\bsLambda_d].
\label{eq:def_Lambda}
\end{align}}
As a consequence, \eqref{eq_b1} can be written as in \eqref{l2_anas_generic_v2}, i.e.,
\begin{align}
\hat{\bfx}
&= \bfF^H \left( \frac{1}{d}\EigSq^H\EigSq + 2\tau  \Qi{\bfF\bfA^H\bfA\bfF^H}  \right)^{-1}\bfF \bfr \label{eq_b2}\\
& = \bfF^H \left[ \frac{1}{2\tau }\bsPsi - \frac{1}{2\tau }\bsPsi\EigSq^H \left( d\bfI_{N_l} +\frac{1}{2\tau }\EigSq \bsPsi\EigSq^H \right)^{-1} \EigSq \bsPsi\frac{1}{2\tau}\right]\bfF\bfr  \label{eq_b3}\\
& = \frac{1}{2\tau }\bfF^H\bsPsi\bfF\bfr- \frac{1}{2\tau } \bfF^H\bsPsi\EigSq^H\left(2\tau d\bfI_{N_l} + \EigSq\bsPsi\EigSq^H  \right)^{-1}\EigSq\bsPsi\bfF\bfr
\end{align}
where $\bsPsi = \Qi{\bfF\left(\bfA^H\bfA\right)^{-1}\bfF^H}$. The Lemma \ref{lemma:2} is adopted from \eqref{eq_b2} to \eqref{eq_b3} \revAQ{with $\bfA_1 = \Qi{2\tau\bfF\bfA^H\bfA\bfF^H}$, $\bfA_2 = \EigSq^H$, $\bfA_3 = \frac{1}{d}\bfI$ and $\bfA_4 = \EigSq$.} \nd{Note that the matrices $\bfA_1$ and $\bfA_3$ are always invertible, implying that the Woodbury formula can be applied.}

\section{\AB{Pseudo codes of the proposed fast ADMM super-resolution methods for TV and $\ell_1$-norm regularizations}}
\label{app:psd_code}
\IncMargin{1em}
\begin{algorithm}
\label{alg3_admm_TV_fast}
\KwIn{$\bfy$, $\bfH$, $\bfS$, $\tau$, $d$, $\bfD_{\rm{h}}$ and $\bfD_{\rm{v}}$}
Set $k = 0$, choose $\mu>0$, $\bfd^0$, $\bfu^0$\;
\tcpp{Factorization of matrix $\bfH$}
$\bfH = {\bfF^H \bsLambda \bfF}$\;
$\EigSq = [\bsLambda_1,\bsLambda_2, \cdots,\bsLambda_d]$\;
\tcpp{Factorization of matrices $\bfD_{\rm{h}}$ and $\bfD_{\rm{v}}$}
$\bfD_{\rm{h}} = {\bfF^H \bsSigma_{\rm{h}} \bfF}$\;
$\bfD_{\rm{v}} = {\bfF^H \bsSigma_{\rm{v}} \bfF}$\;
$\bsPsi = (\bsSigma_{\rm{h}}^H\bsSigma_{\rm{h}} + \bsSigma_{\rm{v}}^H\bsSigma_{\rm{v}})^{-1}$ \;
\textbf{Repeat}\\
\tcpp{Update $\bfx$ using Theorem \ref{the:Ubar}}
$\bsrho_{\rm{h}} = \bfu^{k}_{\rm{h}} - \bfd^{k}_{\rm{h}}$\; $\bsrho_{\rm{v}} = \bfu^{k}_{\rm{v}} - \bfd^{k}_{\rm{v}}$ \;
$\bfF\bfr = \bfF(\bfH^H\bfS^H\bfy + \mu\bfD_{\rm{h}}\bsrho_{\rm{h}} + \mu\bfD_{\rm{v}}\bsrho_{\rm{v}})$\;
$\bfx_f = \left(\bsPsi\EigSq^H \left( \mu d\bfI_{N_l} + \EigSq\bsPsi\EigSq^H \right)^{-1}\EigSq\bsPsi\right)\bfF\bfr$ \;
$\bfx^{k+1} = \frac{1}{\mu}\bfF^H\bsPsi\bfF\bfr-\frac{1}{\mu}\bfF^H\bfx_f$ \;
\tcpp{Update $\bfu $ using the vector-soft-thresholding operator}
$\bsnu = [\bfD_{\rm{h}}\bfx^{k+1} + \bfd^{k}_{\rm{h}},\bfD_{\rm{v}}\bfx^{k+1} + \bfd^{k}_{\rm{v}}]$ \;
$\bfu^{k+1}[i] = \max \lbrace \textbf{0},\|\bsnu[i]\|_2-\tau / \mu \rbrace \frac{\bsnu[i]}{\|\bsnu[i]\|_2}$ \;
\tcpp{Update the dual variables $\bfd$}
$\bfd^{k+1} = \bfd^{k} +(\bfA\bfx^{k+1} - \bfu^{k+1})$\;
$k = k+1$\;
until stopping criterion is satisfied\;
\KwOut{$\hat{\bfx}= \bfx^{k}$.}
\caption{FSR with TV regularization}
\end{algorithm}

\newpage
\begin{algorithm}
\label{alg3_admm_l1_fast}
\KwIn{$\bfy$, $\bfH$, $\bfS$, $\tau$, $d$}
Set $k = 0$, choose $\mu>0$, $\bfd^0$, $\bfu^0$\;
\tcpp{Factorization of matrix $\bfH$}
$\bfH = {\bfF^H \bsLambda \bfF}$\;
$\EigSq = [\bsLambda_1,\bsLambda_2, \cdots,\bsLambda_d]$\;
\textbf{Repeat}\\
\tcpp{Update $\bstheta$ using Theorem \ref{the:Ubar}}
$\bfF \bfr = \bfF (\bfH^H\bfS^H\bfy + \mu\bfW(\bfu^{k} - \bfd^{k})$\;
$\bfx_f = \left(\EigSq^H \left( \mu d\bfI_{N_l} + \EigSq\EigSq^H \right)^{-1}\EigSq\right)\bfF\bfr$ \;
$\bfx^{k+1} = \frac{1}{\mu}\bfF\bfr-\frac{1}{\mu}\bfx_f$ \;
\tcpp{Update $\bfu $ using the soft-thresholding operator}
$\bsnu = \bfW^H \bfx^{k+1} + \bfd^k$ \;
$\bfu^{k+1} = \max\lbrace0,|\bsnu| - \tau/\mu\rbrace$; \tcpp{$|\bsnu| \triangleq [|\bsnu_1|,\cdots,|\bsnu_M|]^T \in \mathbb{R}^{M\times1}$}
\tcpp{Update the dual variables $\bfd$}
$\bfd^{k+1}= \bfd^{k} +(\bfW^H \bfx^{k+1} - \bfu^{k+1})$\;
$k = k+1$\;
until stopping criterion is satisfied\;
\KwOut{$\hat{\bfx} = \bfx^{k}$.}
\caption[The LOF caption]{FSR with $\ell_1$-norm regularization in the wavelet domain}
\DecMargin{1em}
\end{algorithm}

\bibliographystyle{IEEEtran}
\bibliography{strings_all_ref,bibFSR}
\end{document}